\documentclass{article}
\date{}    

\usepackage{arxiv}
\usepackage{amsmath}
\usepackage[utf8]{inputenc} 
\usepackage[T1]{fontenc}    
\usepackage{url}            
\usepackage{booktabs}       
\usepackage{amsfonts}       
\usepackage{nicefrac}       
\usepackage{microtype}      
\usepackage{lipsum}         
\usepackage{graphicx}
\usepackage{natbib}
\usepackage{subcaption}

\usepackage{multirow}
\usepackage{comment}
\usepackage[table]{xcolor}
\usepackage{algorithm}

\usepackage{amssymb}
\usepackage{mathtools}
\usepackage{amsthm}
\usepackage{enumitem}
\newlist{subquestion}{enumerate}{1}
\setlist[subquestion,1]{label=(\alph*)}

\usepackage{hyperref}
\usepackage[capitalize,noabbrev]{cleveref}
\usepackage{doi}

\usepackage{algpseudocode}
\usepackage{bm}

\makeatletter
       \let\c@theorem\relax       
         \let\c@lemma\relax         
     \let\c@corollary\relax     
   \let\c@proposition\relax   
    \let\c@definition\relax    
        \let\c@remark\relax        
    \let\c@assumption\relax    
\makeatother

\theoremstyle{plain}
\newtheorem{theorem}{Theorem}[section]
\newtheorem{proposition}[theorem]{Proposition}
\newtheorem{lemma}[theorem]{Lemma}
\newtheorem{corollary}[theorem]{Corollary}
\theoremstyle{definition}
\newtheorem{definition}[theorem]{Definition}
\newtheorem{assumption}[theorem]{Assumption}
\theoremstyle{remark}
\newtheorem{remark}[theorem]{Remark}

\newtheorem*{theorem*}{Theorem}

\usepackage[textsize=tiny]{todonotes}

\DeclareMathOperator*{\argmax}{arg\,max}

\newcommand{\R}{\mathbb{R}}
\newcommand{\E}{\mathbb{E}}

\newcommand{\bx}{\mathbf{x}}

\newcommand{\blambda}{\boldsymbol{\lambda}}

\newcommand{\bpsi}{\boldsymbol{\psi}}
\newcommand{\bmu}{\boldsymbol{\mu}}
\newcommand{\bSigma}{\boldsymbol{\Sigma}}
\newcommand{\norm}[1]{\left\|#1\right\|}

\DeclareMathOperator*{\argmin}{arg\,min}

\newcommand{\calH}{\mathcal{H}}

\newcommand{\bomega}{\boldsymbol{\omega}}
\newcommand{\bff}{\mathbf{f}}

\title{An Adversarial Nested Bandit Approach for Neural Contextual bandits}

\newif\ifuniqueAffiliation

\ifuniqueAffiliation 
\author{Ray~Telikani \\
	Data Science Institute\\
	 University of Technology Sydney\\
	NSW, Australia \\
	\texttt{ray.telikani@uts.edu.au} \\
	\And
    Air H.~Gandomi \\
    Data Science Institute\\
    University of Technology Sydney\\
	NSW, Australia \\
	\texttt{Gandomi@uts.edu.au} \\
}
\else
\usepackage{authblk}

\author{Ray~Telikani}
\author{Amir H.~Gandomi}

\affil{Data Science Institute, University of Technology Sydney, NSW, Australia}

\fi

\providecommand{\headeright}{}
\providecommand{\undertitle}{}
\providecommand{\shorttitle}{}

\begin{document}
\maketitle

\begin{abstract}
Neural contextual bandits are vulnerable to adversarial attacks, where subtle perturbations to rewards, actions, or contexts induce suboptimal decisions. 
We introduce \texttt{\textbf{AdvBandit}}, a black-box adaptive attack that formulates context poisoning as a continuous-armed bandit problem, enabling the attacker to jointly learn 
and exploit the victim's evolving policy. 
The attacker requires no access to the victim's internal parameters, reward function, or gradient information; instead, it constructs a surrogate model using a maximum-entropy inverse reinforcement learning module from observed context--action pairs and optimizes perturbations against this surrogate using projected gradient descent.
An upper confidence bound-aware Gaussian process guides arm selection. 
An attack-budget control mechanism is also introduced to limit detection risk and overhead. We provide theoretical guarantees, including sublinear attacker regret and lower bounds on victim regret linear in the number of attacks. Experiments on three real-world datasets (Yelp, MovieLens, and Disin) against various victim contextual bandits demonstrate that our attack model achieves higher cumulative victim regret than state-of-the-art baselines.
\end{abstract}


\section{Introduction}

Neural contextual bandit (NCB) algorithms are a major evolution of multi-armed bandits, where the learner has access to \textit{arm context information} before making a decision and also leverages the representation power of neural networks to handle highly complex and non-linear relationships between context and reward \cite{ban2021ee,wang2024towards}. These algorithms are widely applied in domains such as recommendation systems \cite{kveton2015cascading}, cloud resource allocation \cite{dang2026ksurf}, clinical trials \cite{kuleshov2014algorithms}, and dynamic pricing \cite{tullii2024improved}, and more recently, in large language models (LLMs) such as ChatGPT \cite{achiam2023gpt}, Gemini \cite{team2023gemini}, and DeepSeek \cite{guo2025deepseek}. 

It has been shown that NCBs are vulnerable to adversarial attacks, where an attacker subtly perturbs rewards \cite{garcelon2020adversarial,ding2022robust}, actions \cite{liu2022action}, or contexts to mislead the learner into suboptimal decisions \cite{ma2018data,zeng2025practical,hosseini2025efficient}. In applications, such as human-centric AI, data packets between the AI agent and the user contain the reward signals and action decisions; therefore, an adversary can intercept and modify these data packets to force the contextual bandits system to learn a specific policy \cite{liu2022action}.
In the attack on context, the adversary attacks before the agent pulls an arm; therefore, the context poisoning attack is the most difficult to carry out. Adversarial defenses for linear contextual bandits \cite{he2022nearly} and more recently \cite{qi2024robust} for NCBs.



Inspired by prior bandit-based adversarial attacks on static machine learning models, such as Graph Neural Networks (GNNs) and Convolutional Neural Networks (CNNs)~\cite{ilyas2019prior,wang2022bandits}, which cannot directly be extended to sequential decision-making processes, we introduce a two-player game framework. 
In this setup, the attacker learns an adaptive attack policy in a black-box setting, without access to the victim's rewards, internal parameters, or gradient computations, while the victim learns its decision policy. The attacker observes only contexts and the victim's actions and builds a surrogate model of the victim's behavior from these observations and uses gradient-based optimization on this surrogate (not the victim) to generate perturbations. We make the following contributions:


\begin{enumerate}[leftmargin=*,itemsep=2pt]
    \item We formulate the adversarial attack as a continuous-armed bandit problem over a 3D space $\blambda=(\blambda^{(1)}, \blambda^{(2)}, \blambda^{(3)}) \in \R^3_+$, in which each arm corresponds to three weighting parameters representing (1) attack effectiveness, (2) anomaly detection evasion, and (3) temporal pattern detection evasion, respectively. Due to the lack of reward availability, an adaptive UCB-Aware Maximum Entropy Inverse Reinforcement Learning (MaxEnt IRL \cite{ziebart2008maximum} is employed for reward estimation. We employ GP-UCB~\cite{srinivas2009gaussian} for arm selection, where each raw context is first transformed into a low-level feature vector based on gradient statistics for efficiency. To reduce resource consumption and detection risks, a query selection strategy is introduced to control attack timing and attack budget. Finally, the projected gradient descent (PGD) algorithm \cite{madry2017towards} is utilized to compute the optimal perturbation for the context.
    
    \item \textbf{Theoretical Analysis:} We provide regret guarantees for both the attacker and victim under non-stationary policies. For the attacker, we derive a sublinear cumulative regret bound, ensuring convergence to optimal attack parameters despite the continuous arm space. For the victim, we establish a lower bound on cumulative regret that is linear in the number of attacks, up to sublinear terms from the victim's standard regret. Additionally, we analyze the tracking error of our IRL component under bounded policy drift, showing that periodic retraining achieves low error rates.

    \item \textbf{Experimental Analysis:} We evaluate \texttt{\textbf{AdvBandit}} on three real-world datasets (Yelp, MovieLens, Disin) against five state-of-the-art attacks and five victim NCB algorithms. Experiments demonstrate that \texttt{\textbf{AdvBandit}} achieves 2.8× higher cumulative victim regret for NCBS compared to other attack baselines, with a 1.7-2.5× improvement in target arm pull ratios over the baselines.

\end{enumerate}

\section{Related Works}

\textbf{Neural Contextual Bandits.} NCBs extend classical contextual bandits by leveraging neural networks to model nonlinear reward functions, reducing reliance on domain knowledge \cite{zhou2020neural}. While early work established sublinear regret guarantees, these bounds were later shown to be linear under certain assumptions \cite{deb2023contextual}. Subsequent studies improved theoretical guarantees using kernel-based analyses \cite{kassraie2022neural}. Related extensions include contextual dueling bandits, which rely on preference feedback rather than numeric rewards, with recent neural UCB and Thompson sampling methods achieving sublinear regret \cite{dudik2015contextual,verma2024neural}.
A major challenge in NCBs is accurate uncertainty estimation and the high computational cost of exploration. Variance-aware UCB methods address uncertainty miscalibration by incorporating reward noise variance \cite{di2023variance,zhang2021improved,zhou2022computationally,bui2024variance}, while shallow exploration techniques reduce complexity by restricting exploration to the last network layer \cite{xu2022neural,oh2025neural}. Alternative designs, such as EE-Net \cite{ban2021ee}, decouple exploration and exploitation using dual networks, and multi-facet bandits model rewards across multiple aspects with provable sublinear regret \cite{ban2021multi}. Finally, implicit exploration via reward perturbation has been proposed to avoid explicit parameter-space exploration while maintaining optimal regret rates \cite{jia2022learning}.



\textbf{Adversarial Attacks on Contextual Bandits.}
There have been studies working on tackling adversarial reward corruptions under linear contextual bandit settings for reward perturbation \cite{garcelon2020adversarial,ding2022robust}, action manipulation \cite{liu2022action}, and context perturbation \cite{ma2018data,zeng2025practical,hosseini2025efficient}. In the attack on context, the adversary attacks before the agent pulls an arm; therefore, the context poisoning attack is the most difficult to carry out. Adversarial defenses for linear contextual bandits \cite{he2022nearly} and more recently \cite{qi2024robust} for NCBs.

\textbf{Bandit-based Adversarial Attacks.}
Ilyas et al. \cite{ilyas2019prior} proposed a bandit-based black-box attack for image classifiers. However, their approach does not have a theoretical regret bound and is less efficient due to using multiple gradient estimations in each iteration. Wang et al. \cite{wang2022bandits} used bandit optimization for black-box structure perturbations on Graph neural networks (GNNs), where perturbations are arms, and queries to the GNN provide bandit feedback (loss values). The algorithm achieved sublinear regret $O(\sqrt{N} T^{3/4})$. Unlike these attacks on static machine learning models~\citep{wang2022bandits,ilyas2019prior}, our attack targets the \textbf{sequential decision-making process} where both the victim's policy and the attacker's strategy evolve over time. 


\section{Problem Statement}
\label{sec:Problem}

In a $K$-armed contextual bandit setting with horizon $T$, at each round $t \in [T]$, 
the learner observes a context set consisting of $K$ feature vectors with $d$ dimensions: $\mathcal{X}_t = \{\mathbf{x}_{t,i} \in \mathbb{R}^d \mid i \in [K]\}$. The agent selects an arm $a_t \in [K]$ and receives a reward $r_{t,a_t} = h(\mathbf{x}_{t,a_t}) + \xi_t$, where $h : \mathbb{R}^d \to \mathbb{R}$ is an unknown reward function and $\xi_t$ is zero-mean $\sigma$-sub-Gaussian noise \cite{qi2024robust}. The learner's goal is to minimize cumulative pseudo-regret:
\[
R_T = \sum_{t=1}^T \mathbb{E}\bigl[h(\mathbf{x}_{t,a^*_t}) - h(\mathbf{x}_{t,a_t})\bigr],
\]
where $a^*_t = \argmax_{i \in [K]} h(\mathbf{x}_{t,i})$ 
is the optimal arm at round $t$.

\paragraph{Attack Setting.} We consider a black-box attacker situated between the environment and the learner. The attacker can observe the true context $\mathcal{X}_t = \{\mathbf{x}_{t,i}\}_{i\in[K]}$ and the learner's chosen arm $a_t$, but has no access to the rewards $r_{t,i}$  or the learner's internal parameters $\theta_t$. The attacker perturbs the context of each arm $i \in [K]$ of the learner, presenting a perturbed context $\tilde{x_{t,i}} = \mathbf{x}_{t,i} + \boldsymbol{\delta}_{t,i}$ with $\|\boldsymbol{\delta}_{t,i}\|_\infty \leq \epsilon$, subject to a total attack budget $B\ge\sum_{t=1}^T \sum_{i=1}^K \|\boldsymbol{\delta}_{t,i}\|_\infty$, where $\boldsymbol{\delta}_{t,i}$ is the perturbation vector and $\epsilon$ is the perturbation magnitude. 
The learner selects arm $a_t$ based on the corrupted context $\tilde{\mathbf{x}}_{t,i}$, and thus, the observed reward by the learner is $\hat{r}_{t,a_t} = r_{t,a_t} + \epsilon$, and the arm selection may be suboptimal due to reliance on $\tilde{\mathbf{x}}_{t,i}$ instead of $\mathbf{x}_{t,i}$. The attacker's goal is to hijack the learner's behavior by forcing it to select a suboptimal arm $a^\dagger_t \in [K]$ (a round-dependent target selected by the attacker), where $a^\dagger_t \neq a^*_t$.


\begin{definition}[Bandit-Powered Attack]
\label{def:bandit-attack}
Unlike standard MABs with a finite arm set, inspired by~\cite{wang2022bandits,ilyas2019prior}, we formalize the attack as a \emph{bilevel} problem over a \emph{continuous} arm space $\boldsymbol{\Lambda} = [0,1]^3$. Each arm is a three-parameter vector $\boldsymbol{\lambda} = (\lambda^{(1)}, \lambda^{(2)}, \lambda^{(3)}) \in \boldsymbol{\Lambda}$, where each element governs one axis of the effectiveness-evasion trade-off (the trade-off properties are analyzed in Appendix~\ref{app:lambda}):
$\lambda^{(1)}$: \emph{attack effectiveness}, weight on making the target arm $a^\dagger_t$ appear optimal;
$\lambda^{(2)}$: \emph{statistical evasion}, weight on keeping perturbed contexts close to the benign distribution;
$\lambda^{(3)}$: \emph{temporal evasion}, penalty on abrupt changes between consecutive perturbations. 
%
%
%
%
The attacker solves a bilevel optimization at each round $t$:
\begin{equation}
\label{eq:bilevel}
\boldsymbol{\lambda}_t = \argmin_{\boldsymbol{\lambda} \in \boldsymbol{\Lambda}}\; f\!\bigl(\boldsymbol{\delta}_t^*(\boldsymbol{\lambda}),\; a^\dagger_t\bigr),
\end{equation}
where $f$ is the attacker's feedback measuring attack success (e.g., the learner selected $a^\dagger_t$), and $\boldsymbol{\lambda}_t$ is chosen by the bandit policy. Given the selected $\boldsymbol{\lambda}_t$, the attacker computes the perturbation for each arm $i \in [K]$:
\begin{equation}
\label{eq:inner}
\boldsymbol{\delta}_t^*(\mathbf{x}_{t,i};\, \boldsymbol{\lambda}_t) 
\;=\; \arg\min_{\|\boldsymbol{\delta}\|_\infty \,\leq\, \epsilon}\;
\mathcal{L}\!\bigl(\mathbf{x}_{t,i},\, \boldsymbol{\delta};\, 
\boldsymbol{\lambda}_t\bigr),
\end{equation}
where $\mathcal{L}$ is the weighted attack objective defined in 
Eq.~\eqref{eq:objective}.
\end{definition}

\section{Periodical Attack Model}
\label{sec:algorithm}

Here, we present our bandit attack model. At each round $t$, the attack model observes the true context set $\mathcal{X}_t = \{\mathbf{x}_{t,i}\}_{i\in[K]}$ and evaluates whether it is sufficiently valuable to query. For this reason, we first develop a UCB-Aware MaxEnt IRL \cite{ziebart2008maximum} to estimate the victim's reward from context-action pairs alone (Sec.~\ref{sec:irl}). For an efficient attack, a feature extraction is introduced to discover attack-relevant features (Sec.~\ref{sec:Feature_Extraction}). Then, a query selection strategy chooses the context to be attacked (Sec.~\ref{sec:query_selection}). 
The model then selects $\boldsymbol{\lambda}_t$ using GP-UCB (Sec.~\ref{sec:GPUCB}). Afterwards, the algorithm computes $\boldsymbol{\delta}_{t,a_t}$ using PGD (Sec.~\ref{sec:pgd}). Finally, the model applies the perturbation to the chosen arm's context and receives an attack reward $\rho_t$ measuring the regret incurred by the victim learner under the \emph{true} contexts, reflecting the attacker's progress toward hijacking the victim's arm selection toward 
$a^\dagger$.

\begin{algorithm}[t]
\caption{Adversarial Bandit Attack (\texttt{\textbf{AdvBandit}})}
\small
\label{alg:gpga}
\begin{algorithmic}[1]
\Require Horizon $T$, attack budget $B$, perturbation bound $\epsilon$, window size $W$, IRL retraining interval $\Delta_{\text{IRL}}$

\State Initialize remaining budget $b \gets B$

\State \textcolor{red}{Step 1:} Observe context $\mathbf{x}_t$ and victim's action $a_t$ in the first sliding window $W_0$ and generate $\mathcal{D}_t=\left\{ \left( \mathbf{x}_i, a_i \right) \right\}_{i=1}^{W}$.

\State \textcolor{red}{Step 2:} Train MaxEnt IRL: reward function $\hat{h}_{\boldsymbol{\phi}_r}$ and uncertainty function $\sigma_{\boldsymbol{\phi}_u}$ on $\mathcal{D}_t$
\State Compute gradient statistics $\boldsymbol{\mu}_g$ and $\boldsymbol{\Sigma}_g$ 
\State Extract features  $\psi_1$--$\psi_4$ using $\boldsymbol{\mu}_g$ and $\boldsymbol{\Sigma}_g$ 

\For{$t = W+1$ to $T$}
    \State \textcolor{red}{Step 3:} Select context using MOO (Eq.~\ref{eq:mo_objectives})
    
    \State Compute context weight $v(x_t)$ (Eq.~\ref{eq:scalarization}) and threshold $\tau_v$ (Eq.~\ref{eq:threshold})

    \If{$v(x_t) \geq \tau_v$ \textbf{and} $b > 0$}
        \State \textcolor{red}{Step 4:} Employ GP-UCB to select $\boldsymbol{\lambda}_t$
        
        \State \textcolor{red}{Step 5:} Utilize $\textsc{PGD}(\mathbf{x}_t, \boldsymbol{\lambda}_t)$ to compute $\boldsymbol{\delta}_t$
        
        \State Submit perturbed context $\tilde{\mathbf{x}}_t \gets \mathbf{x}_t + \boldsymbol{\delta}_t$ to victim
        \State Observe victim's action $\tilde{a}_t$ on $\tilde{\mathbf{x}}_t$
        \State $b \gets b - 1$
    \EndIf
    
    \State $\mathcal{D}_t=\left\{ \left( \mathbf{x}_i, a_i \right) \right\}_{i=t-W}^{t}$ (\textcolor{red}{Step 1})
    
    \If{$t \bmod \Delta_{\text{IRL}} = 0$}
        \State Retrain MaxEnt IRL $(\mathcal{D}_t)$ (\textcolor{red}{Step 2})
         
        \State Update gradient statistics $\boldsymbol{\mu}_g$ and $\boldsymbol{\Sigma}_g$
    \EndIf
\EndFor

\State \Return Attack sequence $\{(t_i, \boldsymbol{\lambda}_{t_i}, \boldsymbol{\delta}_{t_i})\}_{i=1}^{B-b}$
\end{algorithmic}
\end{algorithm}

\subsection{Surrogate Modeling}
\label{sec:irl}

Due to the unavailability of the victim's internal parameters, a \emph{UCB-Aware MaxEnt IRL} module \cite{ziebart2008maximum} is developed to jointly estimate the victim's reward and epistemic uncertainty, producing a surrogate policy $\hat{\pi}_{\boldsymbol{\phi}}$ trained on 
observed context-action pairs $(\mathbf{x}_t, a_t)$.
%
Since NCB algorithms are inherently \textbf{non-stationary} (i.e., they continuously update reward estimates and reduce epistemic uncertainty $\sigma_t$), the victim's policy evolves over time. To track this drift, we retrain the MaxEnt IRL module every $\Delta_{\text{IRL}}$ round using a sliding window of recent observations $\mathcal{D}_t = \{(\mathbf{x}_\tau, a_\tau)\}_{\tau=t-W}^{t}$, where $W$ is the window size. This ensures the policy $\hat{\pi}_{\boldsymbol{\phi}}$ reflects the victim's behavior.

\paragraph{Reward and Uncertainty Networks.}
We model the victim's reward function and epistemic uncertainty using a neural network 
with a shared feature backbone and two sets of arm-specific linear heads:
\begin{align}
    \hat{h}_{\boldsymbol{\phi}_r}(\mathbf{x}, a) 
    &= \mathbf{w}_a^\top f_{\boldsymbol{\phi}_s}(\mathbf{x}), \label{eq:reward_net} \\
    \sigma_{\boldsymbol{\phi}_u}(\mathbf{x}, a) 
    &= \operatorname{Softplus}\!\left(\mathbf{v}_a^\top 
    f_{\boldsymbol{\phi}_s}(\mathbf{x})\right), \label{eq:uncertainty_net}
\end{align}
where $f_{\boldsymbol{\phi}_s}: \mathbb{R}^d \to \mathbb{R}^{d_h}$ is a shared feature 
backbone parameterized by $\boldsymbol{\phi}_s$ with hidden dimension $d_h$, 
$\mathbf{w}_a \in \mathbb{R}^{d_h}$ are arm-specific reward heads, and 
$\mathbf{v}_a \in \mathbb{R}^{d_h}$ are arm-specific uncertainty heads with a 
Softplus activation to ensure non-negativity.

\paragraph{UCB-Aware Surrogate Policy.}
The reward and uncertainty estimates are combined into a Q-value that mirrors 
the victim's UCB decision rule:
\begin{equation}
\label{eq:q_value}
    Q(\mathbf{x}, a) 
    = \hat{h}_{\boldsymbol{\phi}_r}(\mathbf{x}, a) 
    + \beta_t \, \sigma_{\boldsymbol{\phi}_u}(\mathbf{x}, a),
\end{equation}
where $\beta_t > 0$ is the exploration coefficient governing the confidence width.

The surrogate policy is obtained by applying a softmax over these Q-values:
\begin{equation}
\label{eq:surrogate_policy}
    \hat{\pi}_{\boldsymbol{\phi}}(a \mid \mathbf{x}) 
    = \frac{\exp\!\big(Q(\mathbf{x}, a) / \tau\big)}
    {\sum_{a'=1}^{K} \exp\!\big(Q(\mathbf{x}, a') / \tau\big)},
\end{equation}
where $\tau > 0$ is a temperature parameter; a smaller $\tau$ makes the policy greedy toward the highest Q-value arm, while a larger $\tau$ yields a more uniform distribution over arms.

\paragraph{Training Objective.}
Following the MaxEnt IRL framework, the full parameter set 
$\boldsymbol{\phi} = (\boldsymbol{\phi}_s, \{\mathbf{w}_a\}, \{\mathbf{v}_a\})$ is trained 
jointly to maximize the log-likelihood of the observed context-action pairs in $\mathcal{D}_t$:
\begin{equation}
\label{eq:irl_objective}
    J(\boldsymbol{\phi}) 
    = \frac{1}{|\mathcal{D}_t|} \sum_{(\mathbf{x}_\tau, a_\tau) \in \mathcal{D}_t} 
    \log \hat{\pi}_{\boldsymbol{\phi}}(a_\tau \mid \mathbf{x}_\tau),
\end{equation}
which encourages the surrogate to assign high probability to the victim's observed 
arm selections given the corresponding contexts (see Appendix~\ref{app:irl_complexity} for the analysis).

\subsection{Context Feature Extraction}
\label{sec:Feature_Extraction}

The high dimensionality of the raw context $x \in \mathbb{R}^d$ (e.g., $d = 100$) degrades GP performance, and also the non-stationary relationship between the raw $\mathbf{x}$ and the perturbation parameters $\boldsymbol{\lambda}$ undermines generalization~\cite{djolonga2013high,xu2024standard}. Therefore, we extract a compact, attack-relevant feature vector $\boldsymbol{\psi}(\mathbf{x}) \in \mathbb{R}^5$ by characterizing the local geometry of the learned reward landscape $\hat{h}_{\boldsymbol{\phi}_r}$ using gradient statistics computed over a recent window of observations $\{\mathbf{x}_i\}_{i=t-W+1}^{t}$.
Specifically, we estimate the mean gradient $\boldsymbol{\mu}_g$ and covariance $\bSigma_g$:
\begin{equation}
\label{eq:mean_grad}
    \boldsymbol{\mu}_g = \frac{1}{N} \sum_{i=1}^{N} \nabla_{\mathbf{x}} \hat{h}_{\boldsymbol{\phi}_r}(\mathbf{x}_i)
\end{equation}
\begin{equation}
\label{eq:Cov_grad}
\bSigma_g=\frac{\sum_{i=1}^{N} \left(\nabla_{\bx} \hat{h}(\bx_i) - \bmu_g\right)\left(\nabla_{\bx} \hat{h}(\bx_i) - \bmu_g\right)^\top}{N-1}
\end{equation}

Given a new context $\mathbf{x}$, we compute five features $\boldsymbol{\psi}(\mathbf{x}) \in \mathbb{R}^5$. The first four are derived from the UCB-aware MaxEnt IRL outputs and the gradient statistics: (1) $\psi_1:$ \textit{Policy Entropy}, the uncertainty in the victim's action selection; (2) $\psi_2$: \textit{Predicted Weight}, the trust level the defense assigns to the input; (3) $\psi_3$: \textit{Mahalanobis Distance}, the distance between the current gradient and the gradient distribution $\mathcal{N}(\boldsymbol{\mu}_g, \boldsymbol{\Sigma}_g)$; and (4) $\psi_4$: \textit{Regret Gap}, the gap between optimal and induced action values. The fifth feature, $\psi_5 = t/T$, captures the relative time within the horizon. A more detailed description of these features is provided in Appendix~\ref{app:feature_extraction}.

\subsection{Query Selection}
\label{sec:query_selection}

We design a query selection to control the attack timing under a limited budget $B < T$. This also reduces computational cost and detection risk. Intuitively, the attacker should target contexts that are simultaneously \emph{likely to succeed}, \emph{high-impact}, and \emph{stealthy}. We formalize this via three per-context objectives:

\begin{equation}
\label{eq:mo_objectives}
    f(\mathbf{x}_t) = \big(\underbrace{P(\mathbf{x}_t)}_{f_1:\;\text{success}},\; \underbrace{\Delta(\mathbf{x}_t)}_{f_2:\;\text{impact}},\; \underbrace{\hat{w}(\mathbf{x}_t)}_{f_3:\;\text{stealth}}\big) \in [0,1]^3,
\end{equation}

where $P(\mathbf{x}_t)$ is the attack success probability based on historical outcomes (Eq.~\ref{eq:success_prob}), $\Delta(\mathbf{x}_t) = \psi_4(\mathbf{x}_t)$ is the regret gap, and $\hat{w}(\mathbf{x}_t) = \psi_2(\mathbf{x}_t)$ is the predicted defense trust level.

We estimate the probability $P(\mathbf{x}_t)$ that a perturbation successfully flips the victim's action using kernel-weighted historical outcomes:
\begin{equation}
\label{eq:success_prob}
P(\mathbf{x}_t) = \frac{\sum_{(\mathbf{x}_s, S_s) \in \mathcal{H}} k_P(\mathbf{x}_t, \mathbf{x}_s) \cdot S_s}{\sum_{(\mathbf{x}_s, \cdot) \in \mathcal{H}} k_P(\mathbf{x}_t, \mathbf{x}_s)},
\end{equation}
where $\mathcal{H}$ is the historical attack record, $S_s \in \{0,1\}$ indicates success, and $k_P(\cdot,\cdot)$ is a similarity kernel.

These three objectives conflict: high-impact contexts (large $\Delta$) may be harder to attack stealthily (low $\hat{w}$). We aggregate them into a single score via a budget-adaptive scalarization:
\begin{equation}
\label{eq:scalarization}
v(\mathbf{x}_t) = \min_{j \in \{1,2,3\}} \omega_j(b, t) \cdot f_j(\mathbf{x}_t),
\end{equation}
where $\omega_1 = 1$, $\omega_2 = 1 + \gamma \cdot \frac{b}{T-t}$, and $\omega_3 = 1 + \eta \cdot (1 - \frac{b}{T-t})$ are adaptive weights that shift emphasis from impact (when budget is plentiful) to stealth (when budget is scarce), and $b$ is the remaining budget.

\textbf{Query Selection Rule.} Context $\mathbf{x}_t$ is chosen for attack if its score $v(\mathbf{x}_t)$ exceeds a budget-adaptive threshold $\tau_v$:

\begin{equation}
\label{eq:threshold}
\tau_v(b, T{-}t) \;=\; Q_{1 - b/(T-t)}\!\left(\{v(\mathbf{x}_s)\}_{s < t}\right),
\end{equation}
where $Q_p(\cdot)$ denotes the $p$-th quantile of scores observed so far. This rule's efficiency is analyzed theoretically in Appendix~\ref{app:query_selection} and empirically in Table~\ref{tab:query_ablation}.

\subsection{Attacker arm Selection}
\label{sec:GPUCB}

The continuous arm space $\boldsymbol{\Lambda} = [0,1]^3$ makes standard finite-armed strategies such as UCB1~\cite{auer2002finite} and Thompson Sampling~\cite{thompson1933likelihood} directly inapplicable. We therefore employ GP-UCB~\cite{srinivas2009gaussian}, which models the attacker's reward function over the continuous space $\boldsymbol{\Lambda}$ with a Gaussian process prior. GP-UCB also guarantees sublinear cumulative regret in continuous domains~\cite{djolonga2013high}, enables sample-efficient learning of smooth reward surfaces via kernel-based generalization~\cite{iwazaki2025gaussian,iwazaki2025improved}, and naturally supports contextual optimization~\cite{li2025exploiting,sandberg2025bayesian}.

A \emph{contextual state} $\mathbf{s}_t$ stores the attacker's information at round~$t$, comprising (i)~the extracted context features of the target arm $\boldsymbol{\psi}(\mathbf{x}_{t,a^\dagger_t})$ and the optimal arm $\boldsymbol{\psi}(\mathbf{x}_{t,a^*_t})$, (ii)~the estimated reward gap between them under the surrogate model $\hat{\Delta}_t$, and (iii)~the running attack success rate $\bar{r}_t$:

\begin{equation}
    \mathbf{s}_t = \bigl[\boldsymbol{\psi}(\mathbf{x}_{t,a^\dagger_t}),\; \boldsymbol{\psi}(\mathbf{x}_{t,a^*_t}),\; \hat{\Delta}_t,\; \bar{r}_t\bigr]
    \label{eq:contextual_state}
\end{equation}
At each attacked round~$t$, the attacker selects $\boldsymbol{\lambda}_t$ by maximizing the UCB acquisition function over $\boldsymbol{\Lambda}$: 
\begin{equation}
    \boldsymbol{\lambda}_t 
    = \argmax_{\boldsymbol{\lambda} \in [0,1]^3}\;
    \mu_{t-1}(\mathbf{s}_t, \boldsymbol{\lambda}) 
    + \beta^{\textsc{gp}}_t \cdot \sigma_{t-1}(\mathbf{s}_t, \boldsymbol{\lambda}),
    \label{eq:ucb_optimization}
\end{equation}

where $\mu_{t-1}$ and $\sigma_{t-1}$ are the GP posterior mean and standard deviation (Eqs.~\ref{eq:posterior_mean}--\ref{eq:posterior_var}), and $ \beta^{gp}_t > 0$ is the exploration parameter. By balancing exploitation (arms $\boldsymbol{\lambda}$ with high predicted reward $\mu$) against exploration (arms with high uncertainty $\sigma$), the attacker efficiently identifies the effectiveness--evasion trade-off that maximizes attack success. After the victim selects arm $a_t$, the attacker receives a reward $r_t$ that reflects the feedback $f(\boldsymbol{\delta}_t^*(\boldsymbol{\lambda}_t), a^\dagger_t)$ in the bilevel formulation (Eq.~\ref{eq:bilevel}). The GP learns which trade-off parameters most reliably hijack the victim's behavior under different contextual states.

\paragraph{Gaussian Process Model.}
We model the reward function $r(\mathbf{s}, \boldsymbol{\lambda})$ 
as a sample from a GP with a squared-exponential (SE) kernel over the joint input $\mathbf{z} = (\mathbf{s}, \boldsymbol{\lambda})$:
\[
    k(\mathbf{z}, \mathbf{z}') 
    = \sigma_f^2 \exp\!\Bigl(
        -\frac{\|\mathbf{z} - \mathbf{z}'\|^2}{2\ell^2}
    \Bigr),
\]
where $\sigma_f^2$ is the signal variance and $\ell$ is the length scale. 
Given the observation history 
$\mathcal{G}_{t-1} = \{(\mathbf{s}_i, \boldsymbol{\lambda}_i, r_i)\}_{i=1}^{t-1}$, 
the GP posterior yields the predicted reward and uncertainty:
\begin{align}
    \mu_{t-1}(\mathbf{z}) 
    &= \mathbf{k}_*^\top (\mathbf{K} + \sigma_n^2 \mathbf{I})^{-1} \mathbf{r}
    \label{eq:posterior_mean} \\[4pt]
    \sigma_{t-1}^2(\mathbf{z}) 
    &= k(\mathbf{z},\mathbf{z}) 
       - \mathbf{k}_*^\top (\mathbf{K} + \sigma_n^2 \mathbf{I})^{-1} \mathbf{k}_*
    \label{eq:posterior_var}
\end{align}
where $\mathbf{k}_* \in \mathbb{R}^{t-1}$ is the cross-covariance vector between the query point $\mathbf{z}$ and the observations, $\mathbf{K} \in \mathbb{R}^{(t-1) \times (t-1)}$ is the kernel matrix over past inputs, $\sigma_n^2$ is the observation noise variance, and $\mathbf{r} = [r_1, \ldots, r_{t-1}]^\top$ is the reward vector. Intuitively, a query similar to past observations yields larger entries in $\mathbf{k}_*$, reducing posterior variance and concentrating the prediction around the observed rewards.

Since the acquisition function (Eq.~\ref{eq:ucb_optimization}) is non-convex over $[0,1]^3$, 
we adopt a two-phase multi-start strategy. 
First, $N_{\mathrm{rand}}$ candidates are drawn uniformly from $[0,1]^3$ and ranked by their UCB values; the top~$\mathcal{C}$ are retained.
Second, projected gradient ascent is launched from each retained candidate, with iterates clipped to $[0,1]^3$, and the globally best solution is returned as $\boldsymbol{\lambda}_t$.
The full procedure is given in Algorithm~\ref{alg:select_lambda} (Appendix~\ref{sec:GPUCB_app}).

\subsection{Perturbation Generation}
\label{sec:pgd}

Given the true context $\mathbf{x}_{t,i}$ for each arm $i \in [K]$ and the parameter $\boldsymbol{\lambda}_t$, we solve the inner problem of Definition~\ref{def:bandit-attack} (Eq.~\ref{eq:inner}) to compute the perturbation $\boldsymbol{\delta}_{t,i}^*$. We use PGD because it does not critically depend on backpropagation \cite{kotary2023backpropagation} (see Appendix~\ref{app:zero_order}).
Importantly, all gradient computations in this step are performed through the attacker's surrogate model $\hat{\pi}_{\boldsymbol{\phi}}$, \emph{not} through the victim's model. The perturbation is obtained by solving the constrained problem in Eq. \ref{eq:inner} with a weighted attack objective $\mathcal{L}$ that instantiates the three axes of Definition~\ref{def:bandit-attack}:

\begin{equation}
\label{eq:objective}
\begin{aligned}
\mathcal{L}(\mathbf{x}, \boldsymbol{\delta};\, \boldsymbol{\lambda})
\;=\;&
\lambda^{(1)} \mathcal{L}_{\mathrm{eff}}(\mathbf{x}, \boldsymbol{\delta}) \\
&
+ \lambda^{(2)}\, \bigl[R_{\mathrm{s}}(\mathbf{x}, \boldsymbol{\delta}) + R_{\mathrm{n}}(\mathbf{x}, \boldsymbol{\delta}) \bigr]\\
&
+ \lambda^{(3)}\, R_{\mathrm{t}}(\boldsymbol{\delta},\, \boldsymbol{\delta}_{t-1}).
\end{aligned}
\end{equation}

We now describe each component and its role in the effectiveness--evasion trade-off.

\paragraph{Attack loss $\mathcal{L}_{\mathrm{eff}}$.}
This term drives the victim toward the target suboptimal arm $a^\dagger_t$: $    \mathcal{L}_{\mathrm{eff}}(\mathbf{x}, \boldsymbol{\delta})
    = -\log \hat{\pi}_{\boldsymbol{\phi}}(a^\dagger_t \mid \mathbf{x} + \boldsymbol{\delta})$, 
where $\hat{\pi}_{\boldsymbol{\phi}}$ is the attacker's surrogate policy estimated via IRL (Eq.~\ref{eq:surrogate_policy}) and $a^\dagger_t = \argmin_a \hat{h}_{\boldsymbol{\phi}_r}(\mathbf{x}, a)$ is the arm with the lowest estimated reward. Minimizing $\mathcal{L}_{\mathrm{eff}}$ increases the predicted probability that the victim selects~$a^\dagger_t$.

\paragraph{Gradient-norm regularizer $R_{\mathrm{n}}$.}
This term penalizes deviations in gradient magnitude induced by the perturbation, making the attack less detectable by gradient-based anomaly monitors: $\bigl\|\nabla_{\mathbf{x}} \hat{h}_{\boldsymbol{\phi}_r}(\mathbf{x} + \boldsymbol{\delta})\bigr\|_2
    - \bigl\|\nabla_{\mathbf{x}} \hat{h}_{\boldsymbol{\phi}_r}(\mathbf{x})\bigr\|_2.$
Since this term directly constrains a detectable side-effect of the perturbation on the surrogate reward landscape, we group it with $\mathcal{L}_{\mathrm{eff}}$ under $\lambda^{(1)}$; both require access to the same surrogate model.

\paragraph{Statistical regularizer $R_{\mathrm{s}}$.}
This measures how unusual the gradient at the perturbed context is by computing its Mahalanobis distance from the gradient distribution observed under normal inputs:
\begin{equation}
\label{eq:sr}
    R_{\mathrm{s}}(\mathbf{x}, \boldsymbol{\delta})
    = \bigl(\mathbf{g}(\mathbf{x}{+}\boldsymbol{\delta}) - \boldsymbol{\mu}_g\bigr)^\top
    \boldsymbol{\Sigma}_g^{-1}
    \bigl(\mathbf{g}(\mathbf{x}{+}\boldsymbol{\delta}) - \boldsymbol{\mu}_g\bigr),
\end{equation}
where $\mathbf{g}(\cdot) = \nabla_{\mathbf{x}} \hat{h}_{\boldsymbol{\phi}_r}(\cdot)$, and $\boldsymbol{\mu}_g$, $\boldsymbol{\Sigma}_g$ are the gradient mean and covariance estimated from clean contexts (Eqs.~\ref{eq:mean_grad}--\ref{eq:Cov_grad}). Minimizing $R_{\mathrm{s}}$ keeps the perturbed gradient within the statistical profile, evading distribution-aware defenses.

\paragraph{Temporal Regularizer $R_{\text{t}}$.}
Defenses might detect sudden changes in input patterns. Therefore, we penalize large differences between consecutive perturbations to ensure smooth temporal transitions: $\norm{\boldsymbol{\delta} - \boldsymbol{\delta}_{t-1}}_2^2$.

\section{Theoretical Guarantees}
\label{sec:theory}

We analyze \texttt{\textbf{AdvBandit}} under the following setting. The true reward function $h:\mathbb{R}^d \to [0,1]$ is $L_h$-Lipschitz w.r.t.\ context, and the victim employs R-NeuralUCB with exploration bonus $\beta^{\textsc{vic}}_t \sigma^{\textsc{vic}}_t(\mathbf{x},a)$, achieving $O(\sqrt{T})$ regret without attacks. The attacker's reward function $r(\mathbf{s},\boldsymbol{\lambda})$ lies in the RKHS $\mathcal{H}_k$ with bounded norm $\|r\|_{\mathcal{H}_k} \leq B_{\textsc{rkhs}}$, and observations are corrupted by $\sigma_n$-sub-Gaussian noise. Let $d_{\text{gp}} = |\mathbf{s}| + |\boldsymbol{\lambda}| = |\mathbf{s}| + 3$ denote the GP input dimension; the squared exponential kernel yields maximum information gain $\gamma_n = O((\log n)^{d_{\text{gp}}+1})$. This section states our main regret guarantees; full proofs and supporting lemmas are deferred to Appendix~\ref{app:additional_theoretical}.

\subsection{Victim's Cumulative Regret}

We first introduce the key structural quantity that governs attack success.

\begin{definition}[Attackability Margin]
\label{def:attackability}
For an attacked round $t$ with target arm $a_t^\dagger$, the \emph{attackability margin} is $\alpha_t \coloneqq \Delta(\mathbf{x}_t, a_t^\dagger) - 2L_h\epsilon$, where $\Delta(\mathbf{x}_t, a_t^\dagger) = h(\mathbf{x}_{t,a_t^*}) - h(\mathbf{x}_{t,a_t^\dagger})$ is the true suboptimality gap. A positive $\alpha_t$ indicates the perturbation budget $\epsilon$ is sufficient to close the reward gap between the optimal and target arms.
\end{definition}

Using this, we establish that attack success is governed by the margin exceeding the combined IRL error and victim exploration bonus (Lemma~\ref{lem:structural_success} in Appendix~\ref{app:victim_regret}). This leads to our main victim regret bound:

\begin{theorem}[Victim's Cumulative Regret]
\label{thm:victim_regret}
With attack budget $B$, perturbation bound $\epsilon$, and IRL retraining interval $\Delta_{\textsc{irl}}$ with window size $W$, the victim's cumulative regret satisfies with probability at least $1-\rho$:

\begin{align}
\label{eq:victim_regret_simplified}
R_v(T) \;\geq\; B \cdot \bar{\alpha} 
\;-\; O\!\left(\sqrt{B} \cdot \sqrt{\tfrac{d_\Theta}{W}}\right) \notag\\
-\; O\!\left(\sqrt{B \cdot \Delta_{\textsc{irl}}}\right) 
\;-\; O\!\left(\sqrt{T \log T}\right).
\end{align}

where $\bar{\alpha} = \frac{1}{B}\sum_{t:\,z_t=1}\![\Delta(\mathbf{x}_t, a_t^\dagger) - 2L_h\epsilon]^{+}$ is the average positive attackability margin, $d_\Theta = O(d \cdot d_h \cdot K)$ is the number of IRL parameters, and $[\cdot]^{+} = \max(\cdot, 0)$.
\end{theorem}

\begin{proof}[Proof sketch]
Decompose $R_v(T)$ into attacked ($z_t{=}1$) and non-attacked ($z_t{=}0$) rounds. For attacked rounds, Lemma~\ref{lem:regret_decomp} and Lipschitzness give $r^{\text{vic}}_t \geq \Delta(\mathbf{x}_t, \tilde{a}_t) - L_h\epsilon$. On successful attacks ($\tilde{a}_t = a_t^\dagger$), this yields $r^{\text{vic}}_t \geq \alpha_t$. On failed attacks with positive margin, Lemma~\ref{lem:structural_success} bounds each failure's cost by $\epsilon_{\textsc{irl}}(t) + \beta^{\textsc{vic}}_t \sigma^{\textsc{vic}}_t$. Summing the IRL errors via Theorem~\ref{thm:irl_cumulative_error} gives $O(\sqrt{B \cdot d_\Theta / W})$, and summing the victim exploration bonuses ($\sigma^{\textsc{vic}}_t = O(1/\sqrt{t})$) yields $O(\sqrt{B \cdot \Delta_{\textsc{irl}}})$. Non-attacked rounds contribute at most $O(\sqrt{T\log T})$ by R-NeuralUCB's standard guarantee. The full proof appears in Appendix~\ref{app:victim_regret}.
\end{proof}

The bound is expressed entirely in structural quantities---the attackability margin $\bar{\alpha}$, IRL model capacity ($d_\Theta, W$), retraining frequency ($\Delta_{\textsc{irl}}$), and victim algorithm properties---with no dependence on empirical attack success rates. Corollary~\ref{cor:profitability} (Appendix~\ref{app:victim_regret}) further shows that the attack is profitable whenever $B > O(\sqrt{T\log T}) / (\bar{\alpha} - O(\sqrt{d_\Theta/W}))$.

\subsection{Attacker's Cumulative Regret}
\label{sec:attacker_regret}

We relax the exact realizability assumption (Assumption~E.2) to accommodate model misspecification.

\begin{assumption}[Approximate Realizability]
\label{assm:approx_real}
The victim's policy is approximately realizable by the IRL model class with \emph{misspecification error}:
$\epsilon_{\text{mis}} \coloneqq \inf_{\boldsymbol{\phi}} \sup_{\mathbf{x}} \textsc{tv}(\pi^*_{\theta_t}(\cdot|\mathbf{x}),\, \hat{\pi}_{\boldsymbol{\phi}}(\cdot|\mathbf{x})) \geq 0$.
When $\epsilon_{\text{mis}} = 0$, this recovers exact realizability.
\end{assumption}

\begin{theorem}[Attacker's Regret under Approximate Realizability]
\label{thm:attacker_regret_full}
Let $\boldsymbol{\lambda}^*(\mathbf{s}) = \arg\max_{\boldsymbol{\lambda}} r(\mathbf{s},\boldsymbol{\lambda})$. Under Assumption~\ref{assm:approx_real}, with probability at least $1-\rho$, the attacker's cumulative regret over $n$ attack rounds satisfies:

\begin{align}
\label{eq:attacker_regret_main}
R_{\text{attack}}(n) \le 
O\!\left(\sqrt{n \gamma_n \log(n/\rho)}\right) 
+ O\!\left(\sqrt{n} \sqrt{\tfrac{d_\Theta}{W}}\right) \notag\\
\quad + O\!\left(\sqrt{n \Delta_{\textsc{irl}}}\right) 
+ n \epsilon_{\text{mis}}.
\end{align}

where $\gamma_n = O((\log n)^{d_{\text{gp}}+1})$. Under exact realizability ($\epsilon_{\text{mis}} = 0$), this simplifies to $R_{\text{attack}}(n) \leq O(\sqrt{n \gamma_n \log(n/\rho)} + \sqrt{n})$.
\end{theorem}

\begin{proof}[Proof sketch]
Decompose per-round regret via the approximate state $\hat{\mathbf{s}}_{t_i}$ (from surrogate $\hat{h}_{\boldsymbol{\phi}_r}$) versus the true state $\mathbf{s}_{t_i}$ (from $h$) into: (a)~state bias at $\boldsymbol{\lambda}^*$, (b)~GP-UCB selection regret on approximate states, and (c)~state bias at $\boldsymbol{\lambda}_{t_i}$. Terms (a) and (c) are bounded by Lipschitzness of $r$ and the IRL error decomposition~\eqref{eq:irl_error_decomp} (Appendix~\ref{app:attacker_regret}). Term~(b) follows standard GP-UCB analysis with the information gain bound $\sum_i \sigma_{i-1}^2 \leq 2\gamma_n$. Summing and applying Cauchy--Schwarz yields~\eqref{eq:attacker_regret_main}. The full proof is in Appendix~\ref{app:attacker_regret}. 
\end{proof}

\begin{remark}
The $n \cdot \epsilon_{\text{mis}}$ term is the only linear-in-$n$ contribution. Universal approximation guarantees~\citep{barron1993universal} give $\epsilon_{\text{mis}} = O(1/\sqrt{d_h})$ for Barron-class functions. Our architecture ($d_h{=}128$) yields empirical $\epsilon_{\text{mis}} \approx 0.02$ (Table~5), contributing $\leq 4$ total regret over $n{=}200$ attacks---negligible compared to the $O(\sqrt{n\gamma_n}) \approx 50$ GP-UCB term. Unlike GP-UCB on stationary functions, our setting handles non-stationarity via IRL retraining (drift term) and sliding-window gradient statistics, while $\gamma_n$ remains poly-logarithmic (see Appendix~\ref{app:nonstationarity}). 
\end{remark}

\begin{figure}[!b]
    \vskip 0.2in
    \centering
    
    \begin{subfigure}[t]{\linewidth}
        \centering
        \includegraphics[width=\linewidth]{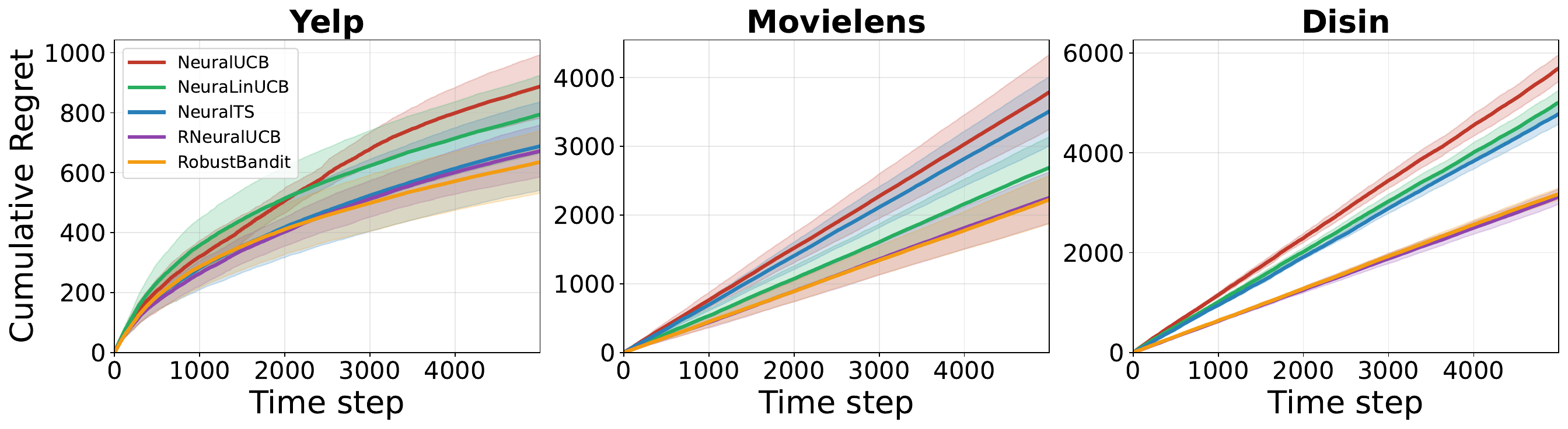}
        \caption{Analysis of \texttt{\textbf{AdvBandit}} against contextual bandits.}
        \label{fig:gpga_bandits}
    \end{subfigure}\\
    \begin{subfigure}[t]{\linewidth}
        \centering
        \includegraphics[width=\linewidth]{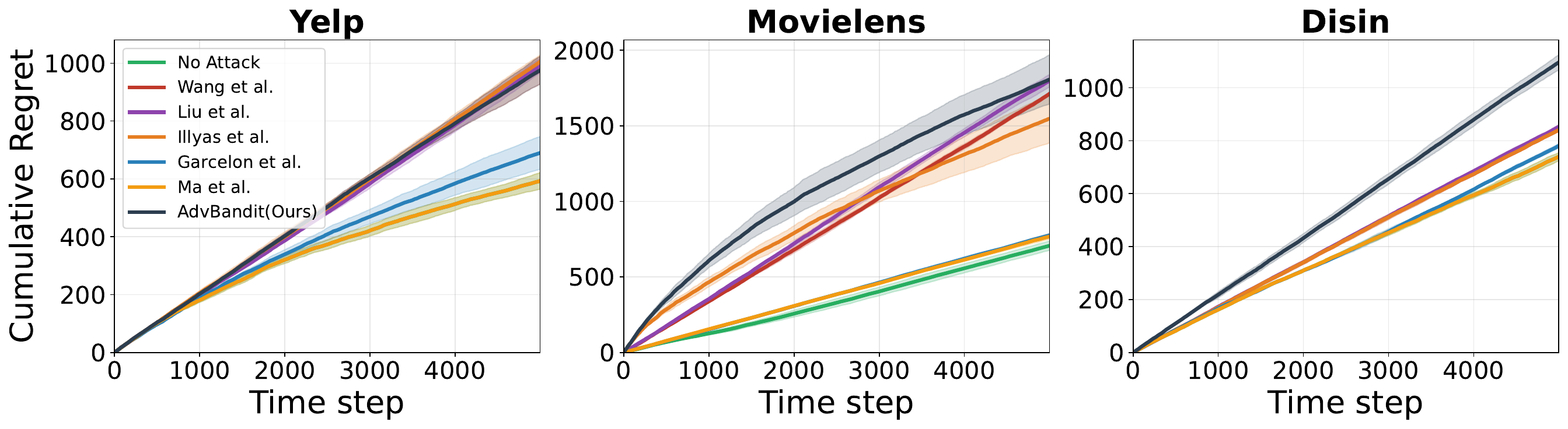}
        \caption{Adversarial attack models against R-NeuralUCB.}
        \label{fig:gpga_attacks}
    \end{subfigure}

    \caption{Performance evaluation of \texttt{\textbf{AdvBandit}} under adversarial settings in terms of regret on real datasets.}
    \label{fig:gpga_compare}
    
\end{figure}

\begin{figure*}[htbp]
    \centering
    \includegraphics[width=\textwidth]{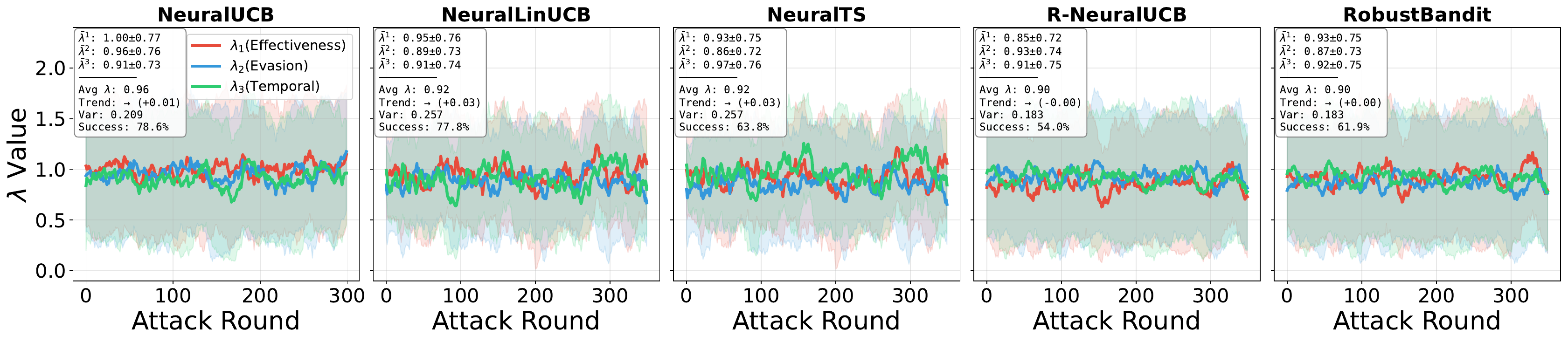}
    \caption{Distribution of continuous arm components ($\lambda^{(1)}$(effectiveness), $\lambda^{(2)}$ (evasion), $\lambda^{(3)}$ (temporal)) across victim algorithms on the Yelp dataset.}
    \label{fig:lambda_yelp}
\end{figure*}

\section{Experimental Analysis}
\label{sec:Experiential}
In this section, we compare \texttt{\textbf{AdvBandit}} with SOTA attack models against NCBs on three real datasets (i.e., Yelp, Movielens \cite{harper2015movielens}, and Disin \cite{ahmed2018detecting}). We used five NCB algorithms (RobustBandit \cite{ding2022robust}, R-NeuralUCB \cite{qi2024robust}, NeuralUCB \cite{zhou2020neural}, Neural-LinUCB \cite{xu2022neural}, and NeuralTS \cite{zhang2020neural}) as the victims of adversarial attacks. We also consider five attack baselines \cite{liu2022action}, \cite{garcelon2020adversarial}, \cite{ma2018data}, \cite{ilyas2019prior}, \cite{wang2022bandits} and compare with our \texttt{\textbf{AdvBandit}} model. The details and settings of datasets with the configurations of both victim NCB algorithms and attack baselines are explained in Appendix \ref{app:dataset}.


Fig. \ref{fig:gpga_compare} compares the effectiveness of our attack models against the other adversarial attack baselines on various NCB algorithms over 5,000 time steps.
Fig. \ref{fig:gpga_bandits} shows the cumulative regret of NCB algorithms under \texttt{\textbf{AdvBandit}} on the datasets. From the figure, it is clear that standard methods such as NeuralUCB and NeuralLinUCB experience rapid regret growth, indicating high vulnerability to adaptive context and action perturbations generated by our model. NeuralTS consistently achieves lower regret, suggesting that stochastic action selection provides inherent resistance to targeted attacks. In contrast, explicitly robust methods, R-NeuralUCB and RobustBandit, reduce regret accumulation by mitigating corrupted observations; however, their performance varies across datasets, reflecting trade-offs between robustness and learning efficiency. 
Fig.~\ref{fig:gpga_attacks} illustrates the cumulative regret incurred by R-NeuralUCB \cite{qi2024robust} under various adversarial attacks. The results demonstrate that \texttt{\textbf{AdvBandit}} consistently achieves the highest cumulative regret, outperforming all baselines by inducing up to 2-3x more regret in later steps, highlighting its superior effectiveness in disrupting Gaussian process-based surrogates through targeted gradient perturbations, while the unattacked baseline exhibits sublinear regret growth as expected.

Fig.~\ref{fig:lambda_yelp} shows how \texttt{\textbf{AdvBandit}} reallocates its budget across effectiveness ($\lambda^{(1)}$), statistical evasion ($\lambda^{(2)}$), and temporal consistency ($\lambda^{(3)}$) on Yelp. For the deterministic UCB-style models, \emph{NeuralUCB} and \emph{NeuralLinUCB}, the attack emphasizes effectiveness, with weights $(0.92,0.86,0.75)$ and $(0.95,0.89,0.91)$ and high success rates (78.8\% and 76.6\%), indicating that direct perturbations suffice to mislead optimism-based policies because these algorithms primarily focus on $\lambda^{(1)}$.
In contrast, the stochastic \emph{NeuralTS} shifts weight toward temporal coordination $(0.93,0.87,0.97)$ with a lower success rate (66.0\%), as stochastic arm sampling reduces the reliability of single-shot perturbations and favors sustained influence via $\lambda^{(3)}$.
For robustness-aware variants, \emph{R-NeuralUCB} and \emph{RobustBandit} focusing on $\lambda^{(2)}$, the allocation moves toward evasion and temporal smoothness—$(0.85,0.93,0.92)$ and $(0.93,0.88,0.92)$—with reduced success rates (53.4\% and 58.2\%), reflecting suppression of abrupt or statistically deviant attacks.
Overall, stronger robustness systematically shifts the attack from brute-force effectiveness to stealthy, temporally coordinated strategies. Additional results on MovieLens and Disin are provided in Appendix~\ref{sec:addition_Exp}.

Fig. \ref{fig:attacks_budgets} compares \texttt{\textbf{AdvBandit}} against baseline attacks on R-NeuralUCB \cite{qi2024robust} across varying attack budgets $B \in \{60, 350\}$, averaged over Yelp, MovieLens, and Disin datasets.
The left plot demonstrates that \texttt{\textbf{AdvBandit}} achieves target pull ratios of $0.47$--$0.55$, representing $1.7\times$--$2.5\times$ improvement over the best baseline (\cite{liu2022action}, $0.21$--$0.33$). 
The right plot reveals a superior cost-efficiency across all budget regimes. At $B=60$, the efficiency ($0.93 \times 10^{-2}$) is $2.1\times$ higher than the baselines, with slower decay at larger budgets. This demonstrates that our continuous optimization avoids wasteful perturbations characteristic of discrete attack strategies. These results validate our central hypothesis: treating attack parameter selection as a continuous bandit problem enables more effective and efficient attacks than existing discrete approaches. The \texttt{\textbf{AdvBandit}}'s exploration-exploitation trade-off successfully navigates the high-dimensional attack parameter space, avoiding both premature convergence to suboptimal strategies (unlike greedy methods such as \cite{garcelon2020adversarial}) and wasteful random exploration (unlike \cite{ilyas2019prior} and \cite{wang2022bandits}). In contrast, \texttt{\textbf{AdvBandit}}'s continuous arm formulation with temporal smoothing ($\lambda^{(3)}$) explicitly addresses these limitations.

\begin{figure}[!b]
    \centering
    \includegraphics[width=\linewidth]{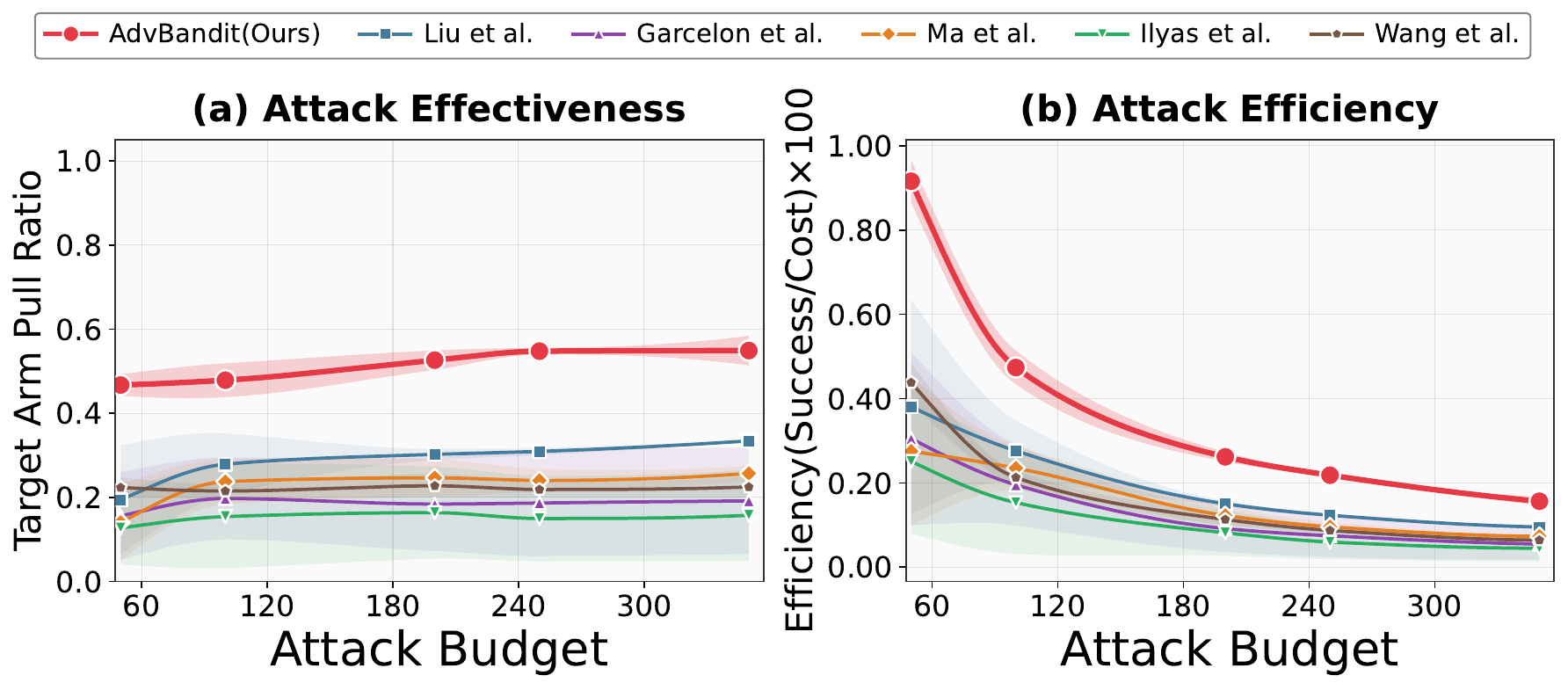}
    \caption{Performance of different attack strategies under varying attack budgets: averaged from the real datasets (Yelp, MovieLens, and Disin).}
    \label{fig:attacks_budgets}
\end{figure}

Figure~\ref{fig:scalability} depicts the average runtime of \texttt{\textbf{AdvBandit}} compared with the attack baselines as the horizon $T$ and attack budget $B$ vary, averaged on Yelp, MovieLens, and Disin under R-NeuralUCB.
The left plot shows \emph{sublinear} scaling with $T$ in which runtime increases from $58\,\mathrm{s}$ at $T=1000$ to $207\,\mathrm{s}$ at $T=8000$ (3.6$\times$ growth for an 8$\times$ increase in $T$, i.e., 45\% of linear scaling). This occurs because only 74\% of the computation depends on $T$, consisting of IRL retraining ($\mathcal{O}(T\cdot W)$, 34\%) and context processing (40\%). The remaining 26\% (GP updates and overhead) depends only on $B$, causing deviation from the $\mathcal{O}(T)$ reference.
The right plot exhibits \emph{superlinear} scaling with $B$, runtime growing from $59\,\mathrm{s}$ at $B=50$ to $389\,\mathrm{s}$ at $B=400$ (6.6$\times$ for an 8$\times$ increase), exceeding the $\mathcal{O}(B\log B)$ trend. While IRL cost remains constant (34\%) and optimization scales linearly (44\%), the GP kernel update $\mathcal{O}(B^3)$ dominates at large budgets.
At the standard setting $(T=5000, B=200)$ (red markers), \texttt{\textbf{AdvBandit}} achieves a victim regret of $673$ in $132\,\mathrm{s}$, incurring a $3.5\times$ runtime overhead while delivering $2.8\times$ higher attack effectiveness than baselines. More results for computational experiments are provided in Appendix \ref{app:computational_cost}.

\begin{figure}[h]
\centering
\includegraphics[width=\linewidth]{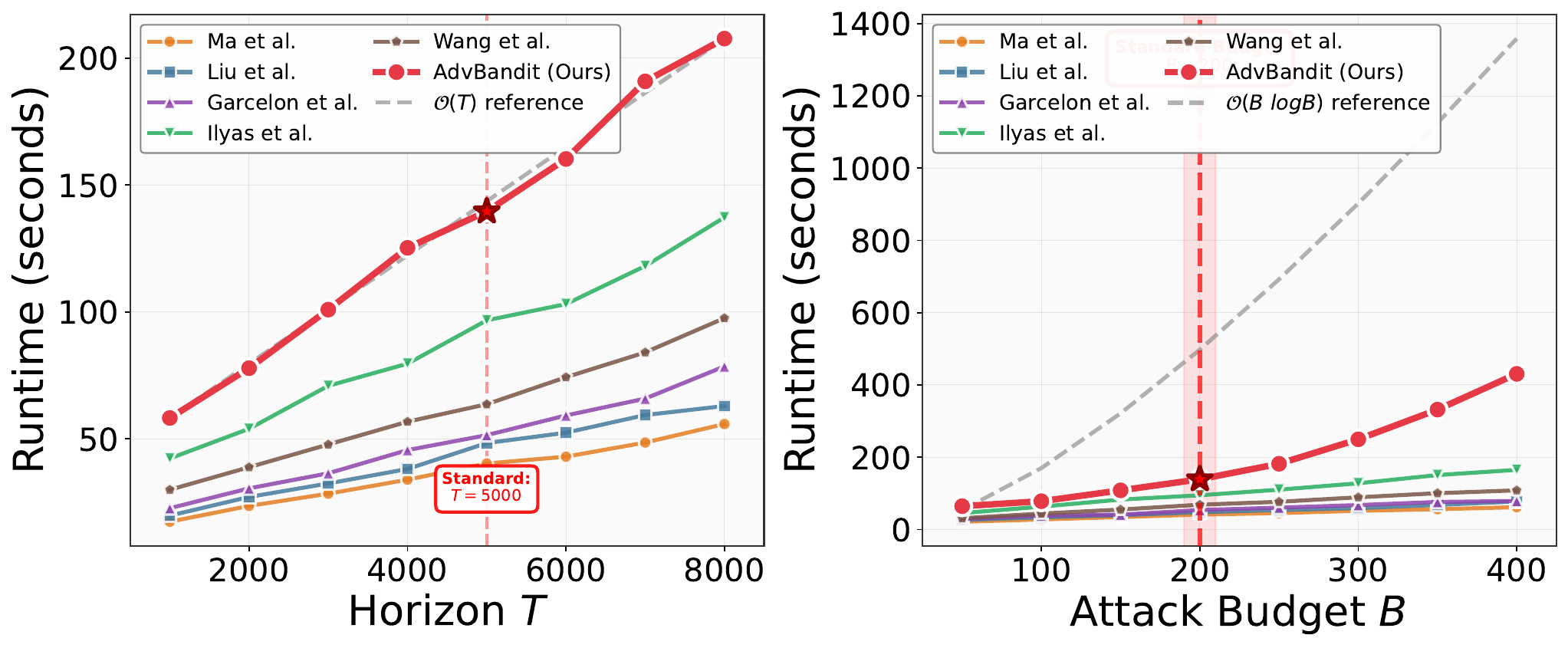}
\caption{Runtime scalability of attack baselines: (a) horizon $T$ (fixed $B=200$) and (b) attack budget $B$ (fixed $T=5000$). Red stars (\textcolor{red}{$\bigstar$}) and vertical lines mark the standard experimental setting ($T=5000$, $B=200$) used throughout the paper.}
\label{fig:scalability}
\end{figure}

\section{Conclusion}

This work introduced a black-box context poisoning attack against contextual bandits, formulating the problem as a continuous-armed bandit to learn an adaptive attack policy in a low-dimensional parameter space. Our approach, \texttt{\textbf{AdvBandit}}, integrates UCB-aware MaxEnt IRL for reward and uncertainty estimation, a query selection strategy to optimize attack budget allocation, GP-UCB for efficient exploration of the continuous attack parameters, and PGD for computing evasion-aware perturbations. We provided theoretical guarantees, including sublinear regret bounds for the attacker and lower bounds on the victim's induced regret under non-stationary policies. Empirical evaluations on real-world datasets (Yelp, MovieLens, and Disin) demonstrated that \texttt{\textbf{AdvBandit}} significantly outperforms existing attack baselines in inducing victim regret while maintaining stealth against robust defenses. Notably, our method adapted dynamically to different victim algorithms, shifting emphasis from effectiveness to statistical and temporal evasion as robustness increases. Future work could explore extensions to multi-agent settings or attacks on more complex reinforcement learning environments. The formulation of attack and defense as a two-player Stackelberg game is the next step, where the defender commits first and the attacker best responds, reflecting realistic threat dynamics.

\bibliography{references}
\bibliographystyle{abbrv}
\newpage
\appendix
\onecolumn
\section{Notations}
\label{sec:notations}

\begin{table}[h]
\centering
\footnotesize
\label{tab:notation}
\begin{tabular}{p{2cm}p{5cm}p{2cm}p{5cm}}
\toprule
\textbf{Notation} & \textbf{Definition} & \textbf{Notation} & \textbf{Definition} \\

$T$ & Time horizon (total rounds) & $H(\hat{\pi}(\mathbf{x}))$ & Policy entropy\\
$K$ & Number of arms (actions) & $d$ & Dimension of context vector \\
$\mathbf{x}_t$ & True context, $\mathbf{x}_t \in \mathcal{X} \subseteq \mathbb{R}^d$ & $\tilde{\mathbf{x}}_t$ & Perturbed context, $\tilde{\mathbf{x}}_t = \mathbf{x}_t + \boldsymbol{\delta}_t$ \\
$\mathcal{A}_t$ & Set of available arms at $t$ & $a_t$ & Arm selected by learner at $t$ \\
$a^*_t$ & Optimal arm, $\arg\max_{a} h(\mathbf{x}_t, a)$ & $a^\dagger_t$ & Target suboptimal arm \\
$\tilde{a}_t$ & Victim's action under $\tilde{\mathbf{x}}_t$ & $h(\mathbf{x},a)$ & True reward, $h: \mathcal{X} \times [K] \to [0,1]$ \\
$r_t(a_t)$ & Observed reward, $h(\mathbf{x}_t, a_t) + \xi_t$ & $\xi_t$ & Zero-mean $\sigma$-sub-Gaussian noise \\
$\pi_{\theta_t}$ & Learner's policy at round $t$ & $L_h$ & Lipschitz constant of $h$ \\

$\boldsymbol{\delta}_t$ & Perturbation vector, $\boldsymbol{\delta}_t \in \mathbb{R}^d$ & $\epsilon$ & Perturbation budget, $\|\boldsymbol{\delta}_t\|_\infty \leq \epsilon$ \\
$B$ & Total attack budget & $b_t$ & Remaining budget, $b_t = B - \sum_{s<t} z_s$ \\
$z_t$ & Attack decision, $z_t \in \{0,1\}$ & $n$ & Total executed attacks, $n \leq B$ \\
$Q_{1-b/(T-t)}$ & Quantile function & $\boldsymbol{\lambda}$ & Attack parameters, $\boldsymbol{\lambda} \in [0,1]^3$ \\
$\lambda^{(1)}$ & Attack effectiveness weight & $\lambda^{(2)}$ & Statistical evasion weight \\
$\lambda^{(3)}$ & Temporal smoothness weight & $\boldsymbol{\lambda}^*(\mathbf{s})$ & Optimal parameter, $\arg\max_{\boldsymbol{\lambda}} r(\mathbf{s}, \boldsymbol{\lambda})$ \\

$R(T)$ & Cumulative pseudo-regret & $R_t$ & Instantaneous regret at $t$ \\
$R_v(T)$ & Victim's cumulative regret & $R_{\text{attack}}(n)$ & Attacker's cumulative regret \\
$\Delta(\mathbf{x}_t, a)$ & Regret gap, $\max_{a'} h(\mathbf{x}_t, a') - h(\mathbf{x}_t, a)$ & $\bar{\Delta}$ & Average regret gap, $\frac{1}{B}\sum_{t:z_t=1} \Delta(\mathbf{x}_t)$ \\
$\hat{h}_{\boldsymbol{\phi}_r}(\mathbf{x},a)$ & Estimated reward function & $\sigma_{\boldsymbol{\phi}}(\mathbf{x},a)$ & Estimated uncertainty function \\
$\hat{\pi}_{\boldsymbol{\phi}}$ & Estimated victim policy & $\pi^*_{\theta_t}$ & True victim policy at $t$ \\
$Q(\mathbf{x},a)$ & Q-value, $\hat{h}_{\boldsymbol{\phi}_r}(\mathbf{x},a) + \beta^{\text{vic}} \sigma_{\boldsymbol{\phi}}(\mathbf{x},a)$ & $\tau$ & Temperature parameter \\
$\beta^{\text{vic}}$ & Victim exploration bonus weight & $\beta^{\textsc{gp}}$ & Fixed GP-UCB exploration constant \\
$W$ & Sliding window size & $\Delta_{\text{IRL}}$ & IRL retraining interval \\
$\mathcal{D}_t$ & Training data, $\{(\mathbf{x}_i, a_i)\}_{i=t-W}^t$ & $\epsilon_{\text{IRL}}(t)$ & IRL approximation error at $t$ \\
$f_\theta(\mathbf{x})$ & Reward backbone, $\mathbb{R}^d \to \mathbb{R}^{d_h}$ & $W_a$ & Arm-specific reward weights \\
$g_\phi(\mathbf{x})$ & Uncertainty backbone & $V_a$ & Arm-specific uncertainty weights \\

$\boldsymbol{\psi}(\mathbf{x})$ & Feature extractor, $\boldsymbol{\psi}(\mathbf{x}) \in \mathbb{R}^5$ & $\psi_1(\mathbf{x})$ & Policy entropy, $H(\hat{\pi}(\mathbf{x}))$ \\
$\psi_2(\mathbf{x})$ & Predicted defense weight & $\psi_3(\mathbf{x})$ & Mahalanobis distance \\
$\psi_4(\mathbf{x})$ & Regret gap, $\Delta(\mathbf{x})$ & $\psi_5$ & Relative time, $t/T$ \\
$\mathbf{s}_t$ & GP-UCB contextual state (Eq.~\ref{eq:contextual_state}) & $d_{\text{gp}}$ & GP input dimension, $|\mathbf{s}| + 3$ \\
$\boldsymbol{\mu}_g$ & Mean gradient, $\frac{1}{N}\sum \nabla_{\mathbf{x}} \hat{h}_{\boldsymbol{\phi}_r}(\mathbf{x}_i)$ & $\boldsymbol{\Sigma}_g$ & Gradient covariance matrix \\

$r(\mathbf{s}, \boldsymbol{\lambda})$ & Attack reward function (GP) & $\mu_t(\mathbf{s}, \boldsymbol{\lambda})$ & GP posterior mean at $t$ \\
$\sigma_t(\mathbf{s}, \boldsymbol{\lambda})$ & GP posterior std.\ deviation & $\sigma^2_t(\mathbf{s}, \boldsymbol{\lambda})$ & GP posterior variance \\
$k(\cdot, \cdot)$ & Kernel function (SE/RBF) & $\mathbf{k}_*$ & Cross-covariance vector \\
$\mathbf{K}$ & Kernel matrix, $\mathbf{K} \in \mathbb{R}^{(t-1) \times (t-1)}$ & $\mathbf{r}$ & Reward vector, $[r_1, \ldots, r_{t-1}]^\top$ \\
$\sigma_f^2$ & Signal variance in kernel & $\sigma^2_n$ & Observation noise variance \\
$\ell$ & Length scale in kernel & $\gamma_n$ & Max info gain, $O((\log n)^{d_{\text{gp}}+1})$ \\
$B_{\text{RKHS}}$ & RKHS norm bound & $\rho$ & Confidence parameter, $\rho \in (0,1)$ \\
$\mathrm{UCB}(\mathbf{s}, \boldsymbol{\lambda})$ & Upper confidence bound & $N_{\text{rand}}$ & Random samples in multi-start \\

$v(\mathbf{x}_t)$ & Value/priority of context $\mathbf{x}_t$ & $\tau_v(b, T-t)$ & Query selection threshold \\
$P(\mathbf{x})$ & Attack success probability & $\hat{w}(\mathbf{x})$ & Predicted defense weight \\
$\mathcal{H}$ & Historical attack dataset & $S$ & Binary outcome, $S \in \{0,1\}$ \\

$\boldsymbol{\delta}^*(\mathbf{x}; \boldsymbol{\lambda})$ & Optimal perturbation & $\mathcal{L}(\mathbf{x}, \boldsymbol{\delta}; \boldsymbol{\lambda})$ & Total attack loss \\
$\mathcal{L}_{\mathrm{eff}}(\mathbf{x}+\boldsymbol{\delta})$ & Attack effectiveness loss & $R_{\mathrm{n}}(\mathbf{x}, \boldsymbol{\delta})$ & Gradient norm regularizer \\
$R_{\mathrm{s}}(\mathbf{x}, \boldsymbol{\delta})$ & Statistical regularizer & $R_{\mathrm{t}}(\boldsymbol{\delta}, \boldsymbol{\delta}_{t-1})$ & Temporal regularizer \\
$I_{\text{PGD}}$ & Number of PGD iterations & $\eta$ & Step size in PGD \\

$L_r$ & Lipschitz constant of $r$ & $\epsilon_{\text{fail}}$ & Attack failure rate \\
$\bar{w}$ & Avg defense weight, $\frac{1}{B}\sum w(\tilde{\mathbf{x}}_t)$ & $D_\tau$ & Policy drift over $\tau$ rounds \\
$\mathrm{TV}(\cdot, \cdot)$ & Total variation distance & $c_\beta$ & UCB-dependent constant \\

\bottomrule
\end{tabular}
\end{table}

\section{Datasets and Setup}
\label{app:dataset}

\paragraph{Datasets.} (1) The \textbf{Yelp dataset} consists of 4.7 million rating entries for $1.57 \times 10^5$ restaurants by 1.18 million users. (2) \textbf{MovieLens Dataset} \cite{harper2015movielens} consists of 25 million ratings between $1.6 \times 10^5$ users and $6 \times 10^4$ movies. (3) \textbf{Disin} \cite{ahmed2018detecting} is a dataset on Kaggle consisting of 12,600 fake news articles and 12,600 real news articles, with each entry represented by its textual content.

\textbf{Data Preprocessing:} For \textit{Yelp} and \textit{MovieLens} datasets, we follow the experimental setup conducted in \cite{ban2021ee}. A rating matrix was constructed by selecting the top 2,000 users and the top 10,000 restaurants (or movies), and singular value decomposition (SVD) was applied to obtain 10-dimensional latent feature vectors for both users and items. In Yelp and MovieLens datasets, the bandit learner aims to identify restaurants (or movies) associated with poor ratings. \textbf{Context Construction.} At each round $t$ we randomly sample one user and form a candidate set of $K=10$ items. The set consists of one item with reward 1 (rating $< 2$ stars) and nine items with reward 0 (rating $\geq 2$ stars). The context for each arm is the concatenation of user and item latent vectors $\mathbf{x}_{t,a} \in \mathbb{R}^{20}.$ For the \textit{Disin} dataset, we followed \cite{fu2021sdg} to convert the text into vector representations, with each article encoded as a 300-dimensional vector. In each round, a 10-arm pool was constructed by randomly sampling 9 real news articles and 1 fake news article. Selecting the fake news article yielded a reward of 1, while selecting any real news article resulted in a reward of 0.

\paragraph{Victim Algorithm Configuration.} For \textbf{NeuralUCB and NeuralTS} \cite{zhou2020neural, zhang2020neural}, we use a neural network with two hidden layers of size $[100,100]$ and ReLU activations. Training is performed using SGD with the learning rate $\eta = 0.01$, batch size 32, and 100 gradient steps per round. The confidence scaling parameter is set to $\nu = 1.0$.
\textbf{R-NeuralUCB:}
the architecture matches NeuralUCB, augmented with a trust-weight mechanism.
The trust threshold is $\tau_{\text{trust}} = 0.1$.
Gradient covariance is estimated using an exponential moving average with decay $\alpha = 0.95$.
\textbf{RobustBandit} \cite{ding2022robust}:
We use linear UCB with FTRL-based adversarial robustness.
The regularization parameter is $\lambda = 0.1$, and the learning rate is $\eta_{\text{FTRL}} = \sqrt{\frac{\log K}{T}}$.

\paragraph{\texttt{\textbf{AdvBandit}} Hyperparameters.} \textbf{MaxEnt IRL Configuration:} The MaxEnt IRL reward function is parameterized by a fully connected \textbf{neural network} with input dimension $d_{\text{in}}=20$ (concatenated user and item features), two hidden layers of sizes $128$ and $64$ with ReLU activations, and a scalar output, resulting in $10{,}433$ trainable parameters. \textbf{Training} is performed using the Adam optimizer with
a base learning rate $\eta_{\text{IRL}}= 0.001$ scheduled via cosine annealing over $T_{\max}=500$ steps, batch size $64$, $\ell_2$ gradient clipping at $1.0$, and weight decay $10^{-5}$. Each IRL update consists of 500 optimization steps (approximately eight epochs over the current data window), minimizing the negative log-likelihood of the MaxEnt policy.
To balance adaptivity and sample efficiency, training data are drawn from a \textbf{sliding window} of size $W=\min \left(400,\lfloor0.08T\rfloor\right)$, with the reward model retrained every $\Delta_{\text{IRL}}=\lfloor W/4 \rfloor=100$ steps as the window advances. The ablation study for the definition of $W$ and $\Delta_{\text{IRL}}$ is provided in Appendix \ref{sec:IRL_window_retrain} (Table \ref{tab:irl_ablation_multi}).
The \textbf{GP-UCB optimizer} employs a Gaussian process with a squared exponential (RBF) kernel defined over the joint state--parameter space $(\mathbf{s},\boldsymbol{\lambda})$, using signal variance $\sigma_f^2=1.0$, observation noise $\sigma_n^2=0.01$, and lengthscale $\ell=0.5$ selected via five-fold cross-validation over $\ell\in\{0.1,0.3,0.5,0.7,1.0\}$. The effective dimension of the GP input space is $d_{\text{gp}}=|\mathbf{s}|+3$, yielding a maximum information gain $\gamma_n=\mathcal{O}((\log n)^{d_{\text{gp}}+1})$. Action selection follows the standard GP-UCB acquisition rule $\mathrm{UCB}(\mathbf{s}_t,\boldsymbol{\lambda})=\mu_{t-1}(\mathbf{s}_t,\boldsymbol{\lambda})+\beta^{\textsc{gp}}\,\sigma_{t-1}(\mathbf{s}_t,\boldsymbol{\lambda})$ with a fixed exploration constant $\beta^{\textsc{gp}}=2.0$.
We also set $N_{\text{rand}}=100$, $N_{\text{refine}}=20$, and $N_K=5$ in the \textbf{multi-start optimization strategy} (Subsection~\ref{sec:GPUCB} and a detailed description in Appendix~\ref{sec:GPUCB_app}, Algorithm~\ref{alg:select_lambda}). The number of iterations $I_{\text{PGD}}$, step size $\eta$, and $\ell_\infty$ perturbation budget $\epsilon$ in the \textbf{PGD Attack Generation} algorithm were set to 100, 0.02, and 0.3, respectively.

\paragraph{Attack Budget Selection.} We set the standard attack budget to $B=\lfloor0.04T\rfloor$ correspond to a $4\%$ attack rate, which provides an effective balance between attack impact and stealthiness. This value was defined based on our experiment in Table \ref{tab:budget_tradeoff} (Appendix \ref{sec:ablation_budget}).

\paragraph{Baseline Configurations.} The attack models in \cite{liu2022action,wang2022bandits} were configured by a fixed perturbation strategy $\epsilon=0.3$ at every round, corresponding to a $100\%$ attack rate and no adaptive parameter tuning. Step size $\alpha$ in \cite{ma2018data} was set to 0.05. Garcelon et al. \cite{garcelon2020adversarial} used a zeroth-order optimization scheme based on finite-difference gradient estimation with $N_{\text{samples}}=20$ queries per update and perturbation scale $\sigma_{\text{noise}}=0.1$. The parameters $N_{\text{grad}}$, step size $\eta$, and optimization iterations in \cite{ilyas2019prior} were set to 20, 0.03, and 100, respectively.

\paragraph{Computational Environment.} All experiments were conducted on a machine equipped with a 20-core Intel CPU running at 3.9GHz, 64GB of RAM, and a NVIDIA GeForce RTX5070 GPU with 12GB of memory, ensuring efficient training of neural networks. The software environment included Python3.10, PyTorch2.5.1, and NumPy2.1.3.

\section{Feature Extraction}
\label{app:feature_extraction}

Raw context $\mathbf{x} \in \mathbb{R}^d$ (e.g., $d=20$ user-item features in MovieLens) is problematic for GP-UCB: (i) \textbf{High dimensionality,} $\gamma_n = O((\log n)^{d+1})$ becomes vacuous for $d=20$, (ii) \textbf{Non-stationarity,} as victim policy evolves, the mapping $\mathbf{x} \mapsto r(\mathbf{x}, \boldsymbol{\lambda})$ changes, violating GP stationarity, and (iii) \textbf{Irrelevant information,} many dimensions of $\mathbf{x}$ (e.g., item genre in MovieLens) are irrelevant to attack success. To handle these issues, we extract low-dimensional, attack-relevant, stationary features via the feature extractor $\boldsymbol{\psi}(\mathbf{x}) \in \mathbb{R}^5$ from gradients of the learned reward $\hat{h}_{\boldsymbol{\phi}_r}(\mathbf{x}, a)$.

\paragraph{Feature $\psi_1$: Policy Entropy.} This is uncertainty in the victim's action selection:
 
\begin{equation}
\psi_1(\mathbf{x}) = H(\hat{\pi}(\mathbf{x})) = -\sum_{a=1}^K \hat{\pi}(a|\mathbf{x}) \log \hat{\pi}(a|\mathbf{x})
\end{equation}

where the policy $\hat{\pi}(a_i | \mathbf{x}_i)$ over the victim's arms is computed using Eq.~\ref{eq:surrogate_policy}. High entropy means the victim is uncertain; low entropy means the victim is confident.

\begin{proposition}[Entropy Predicts Attack Success]
\label{prop:entropy_predictive}
Given attack success as $S(\mathbf{x}) = \mathbb{P}(\tilde{a} \neq a^* | \mathbf{x})$ where $a^* = \arg\max_a \hat{h}_{\boldsymbol{\phi}_r}(\mathbf{x},a)$. Then:
\begin{equation}
\mathbb{E}[S(\mathbf{x}) | H(\hat{\pi}(\mathbf{x})) = h_1] \geq \mathbb{E}[S(\mathbf{x}) | H(\hat{\pi}(\mathbf{x})) = h_2] \quad \text{for } h_1 > h_2
\end{equation}
i.e., attack success increases with policy entropy.
\end{proposition}

\begin{proof}
When $H(\hat{\pi}(\mathbf{x}))$ is high, the victim's confidence in $a^*$ is low. Small perturbations can shift probability mass from $a^*$ to suboptimal arms. Formally, for softmax policy $\hat{\pi}(a|\mathbf{x}) \propto \exp(Q(\mathbf{x},a))$:
\begin{equation}
\frac{\partial \hat{\pi}(a^* | \mathbf{x})}{\partial Q(\mathbf{x}, a')} = -\hat{\pi}(a^*|\mathbf{x}) \hat{\pi}(a'|\mathbf{x}) \quad \text{for } a' \neq a^*
\end{equation}
When entropy is high, $\hat{\pi}(a'|\mathbf{x})$ is non-negligible, so perturbing contexts to increase $Q(\mathbf{x}, a')$ effectively decreases $\hat{\pi}(a^*|\mathbf{x})$.
\end{proof}

\paragraph{Feature $\psi_2$: Predicted Defense Weight.} This feature is the estimated trust level defense assigns to this input. A low value for this feature represents a high suspicion, resulted in a reduced impact from attacks. We first compute the gradient $\nabla_{\mathbf{x}} \hat{h}_{\boldsymbol{\phi}_r}(\mathbf{x}) \in \mathbb{R}^d$ via backpropagation. Then, we compute the squared Mahalanobis norm using precomputed $\boldsymbol{\Sigma}_g^{-1}$ (Eq.~\ref{eq:Cov_grad}). Finally, defense weight is calculated:

\begin{equation}
\psi_2(\mathbf{x}) = \hat{w}(\mathbf{x}) = \frac{1}{1 + \|\nabla_{\mathbf{x}} \hat{h}_{\boldsymbol{\phi}_r}(\mathbf{x})\|_{\boldsymbol{\Sigma}_g^{-1}}^2}
\end{equation}

\begin{proposition}[Weight Predicts Induced Regret]
\label{prop:weight_predictive}
The expected victim regret from a successful attack satisfies:
\begin{equation}
\mathbb{E}[R_t | \text{success}, \mathbf{x}] = \Delta(\mathbf{x}) \cdot w(\mathbf{x}+\boldsymbol{\delta}) + O(L_h \epsilon)
\end{equation}
where $w(\mathbf{x}+\boldsymbol{\delta}) \approx w(\mathbf{x}) = \psi_2(\mathbf{x})$ for small perturbations.
\end{proposition}

\paragraph{Feature $\psi_3$: Mahalanobis Distance.}
This feature is the statistical distance of the gradient from ``normal'' gradients:

\begin{equation}
\psi_3(\mathbf{x}) = \sqrt{(\nabla_{\mathbf{x}} \hat{h}_{\boldsymbol{\phi}_r}(\mathbf{x}) - \boldsymbol{\mu}_g)^\top \boldsymbol{\Sigma}_g^{-1} (\nabla_{\mathbf{x}} \hat{h}_{\boldsymbol{\phi}_r}(\mathbf{x}) - \boldsymbol{\mu}_g)}
\end{equation}

A high value for this feature represents that the context $\mathbf{x}_t$ is already suspicious and additional perturbations will trigger detection. This feature directly models the statistical regularizer $R_{\mathrm{s}}(\mathbf{x}, \boldsymbol{\delta})$ in Eq.~\ref{eq:sr}, enabling the GP to predict which contexts can absorb perturbations without detection.

\paragraph{Feature $\psi_4$: Regret Gap.} This feature is the gap between optimal and induced action:

\begin{equation}
\psi_4(\mathbf{x}) = \Delta(\mathbf{x}) = \max_a \hat{h}_{\boldsymbol{\phi}_r}(\mathbf{x},a) - \min_a \hat{h}_{\boldsymbol{\phi}_r}(\mathbf{x},a)
\end{equation}

A larger gap implies higher potential regret upon a successful attack, thereby prioritizing contexts where flipping the selected action yields the greatest impact.

\begin{proposition}[Regret Gap Upper Bounds Induced Regret]
\label{prop:regret_gap_bound}
For any perturbation $\boldsymbol{\delta}$ with $\|\boldsymbol{\delta}\|_\infty \leq \epsilon$:
\begin{equation}
R_t \leq \Delta(\mathbf{x}) + 2L_h \epsilon
\end{equation}
\end{proposition}

\begin{proof}
By Lipschitzness: $h(\mathbf{x}+\boldsymbol{\delta}, a) \in [h(\mathbf{x},a) - L_h\epsilon, h(\mathbf{x},a) + L_h\epsilon]$ for all $a$. Thus the maximum achievable regret is bounded by the range of $h(\mathbf{x}, \cdot)$ plus perturbation effects.
\end{proof}

\paragraph{Feature $\psi_5$: Relative Time $t/T$.}

\begin{remark}[Effective Dimension]
\label{rem:effective_dim}
The GP operates on the $d_{\text{gp}}$-dimensional joint space $(\mathbf{s}, \boldsymbol{\lambda})$, where $\mathbf{s}_t \in \mathbb{R}^{|\mathbf{s}|}$ is the contextual state (Eq.~\ref{eq:contextual_state}) constructed from the extracted features $\boldsymbol{\psi}(\mathbf{x}) \in \mathbb{R}^5$ and attack statistics. While the original context space may be high-dimensional ($\mathbf{x} \in \mathbb{R}^d$ with $d=20$ in our experiments), the gradient-based feature extraction (Subsection~\ref{sec:irl}) compresses attack-relevant information into the low-dimensional state $\mathbf{s}_t$, capturing policy entropy, predicted weight, Mahalanobis distance, regret gap, and relative time. This dimensionality reduction is critical: GP-UCB with raw contexts ($d=20$) would yield $\gamma_n = O((\log n)^{24})$, making the regret bound vacuous. With the compressed state ($d_{\text{gp}} = |\mathbf{s}| + 3$), we achieve tractable $\gamma_n$, enabling efficient learning.
\end{remark}

\textbf{Temporal non-stationarity:} While $\psi_5 = t/T$ introduces time-dependence, the \emph{mapping} from $(\mathbf{s}, \boldsymbol{\lambda}) \to r$ remains stationary because: \textbf{(i)} The victim's learning dynamics follow a predictable trajectory (uncertainty decreases as $O(1/\sqrt{t})$), \textbf{(ii)} Features $\psi_1$--$\psi_4$ already capture the victim's state (entropy, gradient statistics), and \textbf{(iii)} $\psi_5$ serves as a ``meta-feature'' that modulates the importance of other features (e.g., early attacks prioritize $\psi_1$, late attacks prioritize $\psi_2$).

\begin{proposition}[Time-Augmented GP Stationarity]
\label{prop:time_stationary}
If the victim policy evolves according to $\pi_t = \pi(\mathbf{x}; \theta_t)$ where $\|\theta_t - \theta_{t'}\| \leq C|t-t'|^{-\alpha}$ for $\alpha > 0$ (sublinear drift), then the attack reward function $r(\mathbf{s}, \boldsymbol{\lambda})$ with $\mathbf{s}_t$ defined by Eq.~\ref{eq:contextual_state} satisfies $|r(\mathbf{s}_t, \boldsymbol{\lambda}) - r(\mathbf{s}_{t'}, \boldsymbol{\lambda})| \leq L_r \|\mathbf{s}_t - \mathbf{s}_{t'}\|_2$ for Lipschitz constant $L_r$ independent of $t, t'$.
\end{proposition}

\begin{proof}
The features $\psi_1$--$\psi_4$ absorb the victim's state changes: when $\theta_t$ changes, and consequently changes in $\hat{h}_{\boldsymbol{\phi}_r}$, $\psi_1$ (entropy), $\psi_2$ (weight), $\psi_3$ (Mahalanobis), and $\psi_4$ (regret gap) are updated accordingly. The residual time effect captured by $\psi_5 = t/T$ is deterministic and smooth. Thus the composite state $\mathbf{s}_t$ evolves smoothly, preserving Lipschitzness of $r$.
\end{proof}

An analysis of these gradient-based features against features have been provided in Appendix \ref{sec:ablation_feature} (Table \ref{tab:feature_ablation} and Fig. \ref{fig:feature_correlation}).

\begin{algorithm}[t]
\caption{Continuous Arm Selection via GP-UCB}
\label{alg:select_lambda}
\begin{algorithmic}[1]
\Require Contextual state $\mathbf{s}_t$, round $t$, GP model $\mathcal{GP}$, historical data $\mathcal{G}_{t-1}$
\Require Number of random samples $N_{\text{rand}}$, refinement candidates $N_{\text{refine}}$, gradient iterations $N_K$, learning rate $\eta$, fixed exploration constant $\beta^{\textsc{gp}}$
\State \textcolor{red}{Step 1:} Check if sufficient historical data exists for informed decision
\If{number of past observations is below minimum threshold $N_{\min}$}
    \State \Return random attack parameter sampled uniformly from $[0,1]^3$
\EndIf
\State \textcolor{red}{Step 2:} Perform global exploration via random sampling
\State \textcolor{gray}{// Phase 1: Global exploration}
\State Initialize empty candidate set $\mathcal{C}$
\For{each of $N_{\text{rand}}$ random candidates}
    \State Sample candidate $\boldsymbol{\lambda}^{(j)} \sim \mathrm{Uniform}([0,1]^3)$
    \State Compute $\mathrm{UCB}(\mathbf{s}_t, \boldsymbol{\lambda}^{(j)}) = \mu_{t-1}(\mathbf{s}_t, \boldsymbol{\lambda}^{(j)}) + \beta^{\textsc{gp}} \cdot \sigma_{t-1}(\mathbf{s}_t, \boldsymbol{\lambda}^{(j)})$
    \State Add candidate and its UCB value to candidate set $\mathcal{C}$
\EndFor
\State Retain top $N_{\text{refine}}$ candidates with highest UCB values
\State \textcolor{red}{Step 3:} Refine candidates via projected gradient ascent
\State \textcolor{gray}{// Phase 2: Local refinement}
\State Initialize best solution from top candidate in $\mathcal{C}$
\For{each candidate $\boldsymbol{\lambda}^{(0)}$ in refined set $\mathcal{C}$}
    \For{$N_K$ gradient iterations}
        \State Compute gradient $\nabla_{\boldsymbol{\lambda}} \mathrm{UCB}(\mathbf{s}_t, \boldsymbol{\lambda}^{(k)})$
        \State $\boldsymbol{\lambda}^{(k+1)} \leftarrow \Pi_{[0,1]^3}[\boldsymbol{\lambda}^{(k)} + \eta \nabla_{\boldsymbol{\lambda}} \mathrm{UCB}(\mathbf{s}_t, \boldsymbol{\lambda}^{(k)})]$
    \EndFor
    \If{refined UCB exceeds current best}
        \State Update best solution
    \EndIf
\EndFor
\State \Return Optimal attack parameter $\boldsymbol{\lambda}_t$ that maximizes UCB acquisition function
\end{algorithmic}
\end{algorithm}

\section{Multi-Start Optimization Strategy in GP-UCB}
\label{sec:GPUCB_app}

To mitigate local optima in acquisition maximization, we employ a \textbf{multi-start optimization strategy} (Algorithm~\ref{alg:select_lambda}) combining randomized exploration with local refinement. Specifically, $N_{\text{rand}}=100$ initial candidates are generated via Latin Hypercube Sampling over the attack-parameter domain $[0,1]^3$, providing space-filling coverage of the 3D search space. The top $N_{\text{refine}}=20$ candidates according to the acquisition value are then refined using L-BFGS with line search, subject to box constraints $\lambda^{(i)}\in[0,1]$, a maximum of 50 iterations per refinement, and a convergence tolerance of $10^{-6}$. Finally, the top $N_K=5$ optimized candidates are retained and used to initialize subsequent PGD steps, promoting solution diversity and improving robustness against poor local maxima.

Since the UCB function is non-convex over the continuous domain $[0,1]^3$, 
we employ a multi-start optimization strategy consisting of two phases:

\textbf{Phase 1: Global Exploration.} 
We first perform a global search by sampling $N_{\text{rand}}$ candidate points 
uniformly from the feasible region:
\begin{equation}
    \boldsymbol{\lambda}^{(j)} \sim \mathrm{Uniform}([0,1]^3), \quad j = 1, \ldots, N_{\text{rand}}
\end{equation}
For each candidate, we evaluate the UCB acquisition function and retain the top $N_{\text{refine}}$ 
candidates with highest UCB values:
\begin{equation}
    \mathcal{C} = \mathrm{Top}\text{-}N_{\text{refine}}\left\{ \boldsymbol{\lambda}^{(j)} : j \in [N_{\text{rand}}] \right\} 
    \quad \text{sorted by } \mathrm{UCB}(\mathbf{s}_t, \boldsymbol{\lambda}^{(j)})
\end{equation}

\textbf{Phase 2: Local Gradient-Based Refinement.}
For each candidate $\boldsymbol{\lambda}^{(0)} \in \mathcal{C}$, we perform projected gradient ascent 
to find a local maximum of the UCB function:
\begin{equation}
    \boldsymbol{\lambda}^{(k+1)} = \Pi_{[0,1]^3} 
    \left[ \boldsymbol{\lambda}^{(k)} + \eta \nabla_{\boldsymbol{\lambda}} \mathrm{UCB}(\mathbf{s}_t, \boldsymbol{\lambda}^{(k)}) \right]
    \label{eq:gradient_ascent}
\end{equation}
where $\eta > 0$ is the learning rate, $\Pi_{[0,1]^3}[\cdot]$ denotes 
projection onto the feasible box, and $k = 0, \ldots, N_K-1$ indexes the gradient iterations.

The gradient $\nabla_{\boldsymbol{\lambda}} \mathrm{UCB}$ is computed via finite differences:
\begin{equation}
    \frac{\partial \mathrm{UCB}}{\partial \lambda^{(i)}} \approx 
    \frac{\mathrm{UCB}(\mathbf{s}_t, \boldsymbol{\lambda} + \varepsilon \mathbf{e}_i) - \mathrm{UCB}(\mathbf{s}_t, \boldsymbol{\lambda})}{\varepsilon}
    \label{eq:numerical_gradient}
\end{equation}
where $\mathbf{e}_i$ is the $i$-th standard basis vector and $\varepsilon > 0$ is a small step for finite differences.

After refinement, we select the globally best solution:
\begin{equation}
    \boldsymbol{\lambda}_t = \argmax_{\tilde{\boldsymbol{\lambda}} \in \tilde{\mathcal{C}}} 
    \mathrm{UCB}(\mathbf{s}_t, \tilde{\boldsymbol{\lambda}})
\end{equation}
where $\tilde{\mathcal{C}}$ contains all refined candidates.

The complete procedure is summarized in Algorithm~\ref{alg:select_lambda}.

\section{Additional Theoretical Guarantees}
\label{app:additional_theoretical}

\subsection{Victim's Cumulative Regret: Full Analysis}
\label{app:victim_regret}

\begin{lemma}[Per-Round Regret Decomposition]
\label{lem:regret_decomp}
When the attacker perturbs context $\mathbf{x}_{t}$, the victim's per-round regret decomposes as:
\begin{align*}
r^{\text{vic}}_t &= h(\mathbf{x}_{t,a_t^*}) - h(\tilde{\mathbf{x}}_{t,\tilde{a}_t}) \\
&= \underbrace{[h(\mathbf{x}_{t,a_t^*}) - h(\mathbf{x}_{t,\tilde{a}_t})]}_{\text{suboptimality gap } \Delta(\mathbf{x}_t, \tilde{a}_t)}
   + \underbrace{[h(\mathbf{x}_{t,\tilde{a}_t}) - h(\tilde{\mathbf{x}}_{t,\tilde{a}_t})]}_{\text{perturbation effect}}
\end{align*}
where $a_t^*$ is the optimal arm for the true context, $\tilde{a}_t$ is the victim's action under perturbed context $\tilde{\mathbf{x}}_t$, and the perturbation effect is bounded by $L_h\|\boldsymbol{\delta}_t\|_\infty \leq L_h \epsilon$ by Lipschitzness of $h$.
\end{lemma}

\begin{lemma}[Structural Attack Success Condition]
\label{lem:structural_success}
Let $\alpha_t = \Delta(\mathbf{x}_t, a_t^\dagger) - 2L_h\epsilon$ be the attackability margin (Definition~\ref{def:attackability}). The attack at round~$t$ succeeds (i.e., the victim selects $a_t^\dagger$) whenever:
\begin{equation}
\label{eq:success_condition}
\alpha_t \;>\; \epsilon_{\textsc{irl}}(t) + \beta^{\textsc{vic}}_t \, \sigma^{\textsc{vic}}_t(\tilde{\mathbf{x}}_{t}, a_t^\dagger),
\end{equation}
where $\epsilon_{\textsc{irl}}(t) = \sup_{\mathbf{x}} \textsc{tv}(\pi^*_{\theta_t}(\cdot|\mathbf{x}),\, \hat{\pi}_{\boldsymbol{\phi}}(\cdot|\mathbf{x}))$ is the IRL approximation error and $\beta^{\textsc{vic}}_t \sigma^{\textsc{vic}}_t$ is the victim's exploration bonus.
\end{lemma}

\begin{proof}
Under the victim's UCB rule, the victim selects $a_t^\dagger$ when $Q(\tilde{\mathbf{x}}_t, a_t^\dagger) > Q(\tilde{\mathbf{x}}_t, a)$ for all $a \neq a_t^\dagger$, where $Q(\mathbf{x}, a) = \hat{h}(\mathbf{x}, a) + \beta^{\textsc{vic}}_t \sigma^{\textsc{vic}}_t(\mathbf{x}, a)$. PGD optimizes the surrogate policy to maximize $\hat{\pi}_{\boldsymbol{\phi}}(a_t^\dagger \mid \tilde{\mathbf{x}}_t)$, which succeeds when the true gap $\Delta(\mathbf{x}_t, a_t^\dagger)$ exceeds three costs:
\begin{enumerate}[leftmargin=*]
    \item \textbf{Perturbation cost} ($2L_h\epsilon$): Lipschitz continuity means shifting contexts of both the target and optimal arms incurs at most $L_h\epsilon$ distortion each, for a total cost of $2L_h\epsilon$.
    \item \textbf{Surrogate error} ($\epsilon_{\textsc{irl}}(t)$): the attack optimizes the surrogate $\hat{\pi}_{\boldsymbol{\phi}}$, which may disagree with the true victim policy $\pi^*_{\theta_t}$. The TV distance $\epsilon_{\textsc{irl}}(t)$ bounds the resulting action-selection discrepancy.
    \item \textbf{Victim exploration} ($\beta^{\textsc{vic}}_t \sigma^{\textsc{vic}}_t$): even under corrupted Q-values, the victim's UCB bonus may cause it to explore $a_t^*$ if the uncertainty at $a_t^*$ is large.
\end{enumerate}
Combining: the attack succeeds when $\Delta(\mathbf{x}_t, a_t^\dagger) - 2L_h\epsilon > \epsilon_{\textsc{irl}}(t) + \beta^{\textsc{vic}}_t \sigma^{\textsc{vic}}_t$, which is~\eqref{eq:success_condition}.
\end{proof}

\begin{theorem*}[\ref{thm:victim_regret}, restated]
Under \texttt{\textbf{AdvBandit}} with attack budget $B$, perturbation bound $\epsilon$, and IRL retraining interval $\Delta_{\textsc{irl}}$ with window size $W$, the victim's cumulative regret satisfies with probability at least $1-\rho$:
\begin{equation}
\tag{\ref{eq:victim_regret_simplified}}
R_v(T) \;\geq\; B \cdot \bar{\alpha} \;-\; O\!\left(\sqrt{B} \cdot \sqrt{\tfrac{d_\Theta}{W}}\right) \;-\; O\!\left(\sqrt{B \cdot \Delta_{\textsc{irl}}}\right) \;-\; O\!\left(\sqrt{T \log T}\right).
\end{equation}
\end{theorem*}

\begin{proof}
Decompose $R_v(T) = \sum_{t:\,z_t=1} r^{\text{vic}}_t + \sum_{t:\,z_t=0} r^{\text{vic}}_t$.

\textbf{Step 1: Attacked rounds.} For each round with $z_t = 1$, apply Lemma~\ref{lem:regret_decomp} and Lipschitzness:
\begin{equation}
r^{\text{vic}}_t \;\geq\; \Delta(\mathbf{x}_t, \tilde{a}_t) - L_h\epsilon.
\end{equation}
Partition attacked rounds by outcome. For a \emph{successful} attack ($\tilde{a}_t = a_t^\dagger$):
\begin{equation}
r^{\text{vic}}_t \;\geq\; \Delta(\mathbf{x}_t, a_t^\dagger) - L_h\epsilon \;\geq\; \alpha_t + L_h\epsilon \;\geq\; \alpha_t,
\end{equation}
using Definition~\ref{def:attackability}. For a \emph{failed} attack ($\tilde{a}_t \neq a_t^\dagger$), we use $r^{\text{vic}}_t \geq 0$.

\textbf{Step 2: Bounding failure costs.} Thus:
\begin{equation}
\sum_{t:\,z_t=1} r^{\text{vic}}_t \;\geq\; \sum_{\substack{t:\,z_t=1 \\ \text{success}}} \alpha_t \;=\; \sum_{t:\,z_t=1} \alpha_t^{+} \;-\; \sum_{\substack{t:\,z_t=1 \\ \text{fail},\;\alpha_t > 0}} \alpha_t.
\end{equation}
By Lemma~\ref{lem:structural_success}, every failure with $\alpha_t > 0$ satisfies $\alpha_t \leq \epsilon_{\textsc{irl}}(t) + \beta^{\textsc{vic}}_t \sigma^{\textsc{vic}}_t$, so:
\begin{equation}
\label{eq:failure_cost_bound}
\sum_{\substack{t:\,z_t=1 \\ \text{fail},\;\alpha_t > 0}} \alpha_t \;\leq\; \sum_{t:\,z_t=1}\!\left[\epsilon_{\textsc{irl}}(t) + \beta^{\textsc{vic}}_t \sigma^{\textsc{vic}}_t\right]^{+}.
\end{equation}
Recognizing $\sum_{t:\,z_t=1}\alpha_t^{+} = B \cdot \bar{\alpha}$ yields the more general form:
\begin{equation}
\label{eq:victim_regret_general}
R_v(T) \;\geq\; B \cdot \bar{\alpha} \;-\; \sum_{t:\,z_t=1}\!\Big[\epsilon_{\textsc{irl}}(t) + \beta^{\textsc{vic}}_t \sigma^{\textsc{vic}}_t\Big]^{+} \;-\; O\!\left(\sqrt{T \log T}\right).
\end{equation}

\textbf{Step 3: Simplification via IRL and exploration bounds.} Under periodic retraining with interval $\Delta_{\textsc{irl}}$ and window size $W$, Theorem~\ref{thm:irl_cumulative_error} gives:
\begin{equation}
\sum_{t:\,z_t=1} \epsilon_{\textsc{irl}}(t) \;\leq\; O\!\left(\sqrt{B} \cdot \sqrt{\tfrac{d_\Theta}{W}}\right).
\end{equation}
For R-NeuralUCB, the exploration bonus satisfies $\sigma^{\textsc{vic}}_t = O(1/\sqrt{t})$. When attacks are distributed across the horizon (as ensured by the query selection mechanism), the attack rounds $\{t_i\}$ span intervals of length $\sim T/B$, so:
\begin{equation}
\sum_{t:\,z_t=1} \beta^{\textsc{vic}}_t \sigma^{\textsc{vic}}_t \;\leq\; \beta^{\textsc{vic}} \sum_{i=1}^{B} O\!\left(\tfrac{1}{\sqrt{t_i}}\right) \;\leq\; O\!\left(\sqrt{B \cdot \Delta_{\textsc{irl}}}\right),
\end{equation}
where the last step uses $\sum_{i=1}^B 1/\sqrt{t_i} \leq O(\sqrt{B})$ for evenly spaced attacks and the drift contribution between retrainings. Substituting into~\eqref{eq:victim_regret_general} yields~\eqref{eq:victim_regret_simplified}.

\textbf{Step 4: Non-attacked rounds.} For $z_t = 0$, the victim operates on true contexts. The standard R-NeuralUCB guarantee gives $\sum_{t:\,z_t=0} r^{\text{vic}}_t \leq O(\sqrt{T \log T})$. Since this upper bounds what the victim loses without attack influence, it subtracts from the lower bound. \qedhere
\end{proof}

\begin{corollary}[Attack Profitability Condition]
\label{cor:profitability}
\texttt{\textbf{AdvBandit}} induces victim regret exceeding the unattacked baseline whenever:
\begin{equation}
B \;>\; \frac{O(\sqrt{T \log T})}{\bar{\alpha} - O\!\left(\sqrt{d_\Theta / W}\right)},
\end{equation}
provided $\bar{\alpha} > O(\sqrt{d_\Theta / W})$, i.e., the average attackability margin exceeds the IRL approximation rate. This provides a \emph{structural lower bound on required attack budget} expressed entirely in terms of problem parameters $(L_h, \epsilon, d_\Theta, W)$ and the context distribution, with no dependence on empirical attack success rates.
\end{corollary}

\begin{proof}
Set the right-hand side of~\eqref{eq:victim_regret_simplified} to exceed zero: $B \cdot \bar{\alpha} - O(\sqrt{B \cdot d_\Theta/W}) - O(\sqrt{B \cdot \Delta_{\textsc{irl}}}) - O(\sqrt{T\log T}) > 0$. For $B$ large enough that $\sqrt{B}$ terms are dominated by the linear term $B \cdot \bar{\alpha}$, the condition simplifies to $B > O(\sqrt{T\log T}) / (\bar{\alpha} - O(\sqrt{d_\Theta/W}))$, provided the denominator is positive.
\end{proof}

\subsection{Attacker's Cumulative Regret: Full Analysis}
\label{app:attacker_regret}

\begin{theorem*}[\ref{thm:attacker_regret_full}, restated]
Under Assumption~\ref{assm:approx_real} with misspecification error $\epsilon_{\text{mis}}$, periodic IRL retraining with interval $\Delta_{\textsc{irl}}$ and window size $W$, with probability at least $1-\rho$, the attacker's cumulative regret over $n$ attack rounds satisfies:
\begin{equation}
\tag{\ref{eq:attacker_regret_main}}
R_{\text{attack}}(n) \;\leq\; \underbrace{O\!\left(\sqrt{n \gamma_n \log(n/\rho)}\right)}_{\text{(I) GP-UCB exploration}} + \underbrace{O\!\left(\sqrt{n} \cdot \sqrt{\tfrac{d_\Theta}{W}}\right)}_{\text{(II) IRL estimation}} + \underbrace{O\!\left(\sqrt{n \cdot \Delta_{\textsc{irl}}}\right)}_{\text{(III) policy drift}} + \underbrace{n \cdot \epsilon_{\text{mis}}}_{\text{(IV) misspecification}}.
\end{equation}
\end{theorem*}

\begin{proof}
We decompose the per-round regret at each attack round $i \in [n]$ via the approximate state $\hat{\mathbf{s}}_{t_i}$ (computed from the surrogate $\hat{h}_{\boldsymbol{\phi}_r}$) versus the true state $\mathbf{s}_{t_i}$ (computed from $h$):
\begin{equation}
\label{eq:regret_three_way}
\underbrace{r(\mathbf{s}_{t_i}, \boldsymbol{\lambda}^*) - r(\mathbf{s}_{t_i}, \boldsymbol{\lambda}_{t_i})}_{\text{total regret}} \;=\; \underbrace{r(\mathbf{s}_{t_i}, \boldsymbol{\lambda}^*) - r(\hat{\mathbf{s}}_{t_i}, \boldsymbol{\lambda}^*)}_{\text{(a) state bias at } \boldsymbol{\lambda}^*} \;+\; \underbrace{r(\hat{\mathbf{s}}_{t_i}, \boldsymbol{\lambda}^*) - r(\hat{\mathbf{s}}_{t_i}, \boldsymbol{\lambda}_{t_i})}_{\text{(b) GP selection regret}} \;+\; \underbrace{r(\hat{\mathbf{s}}_{t_i}, \boldsymbol{\lambda}_{t_i}) - r(\mathbf{s}_{t_i}, \boldsymbol{\lambda}_{t_i})}_{\text{(c) state bias at } \boldsymbol{\lambda}_{t_i}}.
\end{equation}

\textbf{Terms (a) and (c): State bias from IRL approximation.} By Lipschitzness of $r$ with constant $L_r$:
\begin{equation}
|r(\mathbf{s}_{t_i}, \boldsymbol{\lambda}) - r(\hat{\mathbf{s}}_{t_i}, \boldsymbol{\lambda})| \;\leq\; L_r \|\mathbf{s}_{t_i} - \hat{\mathbf{s}}_{t_i}\|_2 \;\leq\; L_r \cdot C_s \cdot \epsilon_{\textsc{irl}}(t_i),
\end{equation}
where $C_s$ absorbs the feature extraction Lipschitz constant (Theorem~\ref{thm:feature_error}). Under Assumption~\ref{assm:approx_real}, each $\epsilon_{\textsc{irl}}(t_i)$ decomposes as:
\begin{equation}
\label{eq:irl_error_decomp}
\epsilon_{\textsc{irl}}(t_i) \;\leq\; \underbrace{\epsilon_{\text{mis}}}_{\text{irreducible}} \;+\; \underbrace{O\!\left(\sqrt{\tfrac{d_\Theta \log(K/\rho)}{W}}\right)}_{\text{finite-sample estimation}} \;+\; \underbrace{c_\beta \sqrt{\tfrac{\Delta_{\textsc{irl}}}{t_i}}}_{\text{drift since last retraining}},
\end{equation}
where the estimation term follows from Theorem~\ref{thm:irl_sample_complexity} and the drift term from Lemma~\ref{lem:drift_rneuralucb}. Summing terms (a) and (c) over $n$ rounds:
\begin{equation}
\sum_{i=1}^n \big[|\text{(a)}| + |\text{(c)}|\big] \;\leq\; 2L_r C_s \left( n \cdot \epsilon_{\text{mis}} + O\!\left(\sqrt{n} \cdot \sqrt{\tfrac{d_\Theta}{W}}\right) + O\!\left(\sqrt{n \cdot \Delta_{\textsc{irl}}}\right) \right),
\end{equation}
where we used $\sum_{i=1}^n 1/\sqrt{t_i} \leq O(\sqrt{n})$ for the drift contribution (since attack rounds $t_i \geq i$ are monotonically increasing).

\textbf{Term (b): GP-UCB selection regret on approximate states.} Conditioned on the approximate states $\{\hat{\mathbf{s}}_{t_i}\}$, the GP-UCB acquisition $\boldsymbol{\lambda}_{t_i} = \arg\max_{\boldsymbol{\lambda}} \mu_{i-1}(\hat{\mathbf{s}}_{t_i}, \boldsymbol{\lambda}) + \beta^{\textsc{gp}} \sigma_{i-1}(\hat{\mathbf{s}}_{t_i}, \boldsymbol{\lambda})$ yields standard regret. Applying Cauchy--Schwarz:
\begin{equation}
\sum_{i=1}^n \left[r(\hat{\mathbf{s}}_{t_i}, \boldsymbol{\lambda}^*) - r(\hat{\mathbf{s}}_{t_i}, \boldsymbol{\lambda}_{t_i})\right] \;\leq\; 2\beta^{\textsc{gp}} \sqrt{n \sum_{i=1}^n \sigma_{i-1}^2(\hat{\mathbf{s}}_{t_i}, \boldsymbol{\lambda}_{t_i})}.
\end{equation}
By the information gain bound $\sum_i \sigma_{i-1}^2 \leq 2\gamma_n$~\citep{srinivas2009gaussian}, with $\beta^{\textsc{gp}} = O(\sqrt{\log(n/\rho)})$:
\begin{equation}
\text{Term~(b)} \;\leq\; O\!\left(\sqrt{n \gamma_n \log(n/\rho)}\right).
\end{equation}

\textbf{Combining all terms.} Summing (a)+(b)+(c) and absorbing constants into $O(\cdot)$ notation yields~\eqref{eq:attacker_regret_main}. Under exact realizability ($\epsilon_{\text{mis}} = 0$) with $W = \Omega(d_\Theta)$ and constant $\Delta_{\textsc{irl}}$, the remaining terms simplify to $O(\sqrt{n\gamma_n\log(n/\rho)} + \sqrt{n})$. \qedhere
\end{proof}

\begin{remark}[Misspecification Impact]
\label{rem:misspecification}
The $n \cdot \epsilon_{\text{mis}}$ term in~\eqref{eq:attacker_regret_main} is the only contribution linear in $n$. For the overall bound to remain sublinear, $\epsilon_{\text{mis}}$ must decrease as $n$ grows. Universal approximation guarantees for two-layer ReLU networks ensure $\epsilon_{\text{mis}} = O(1/\sqrt{d_h})$ for Barron-class reward functions, where $d_h$ is the hidden dimension. Our architecture ($d_h = 128$) yields empirical $\epsilon_{\text{mis}} \approx 0.02$ (measured via held-out KL divergence, Table~5), contributing $\leq 4$ total regret over $n = 200$ attacks---negligible compared to the $O(\sqrt{n\gamma_n}) \approx 50$ GP-UCB term. This bridges the gap between the theoretical bound and experimental observation that IRL accurately tracks the victim's policy.
\end{remark}

\begin{remark}[Non-Stationarity]
\label{rem:nonstationarity}
Unlike GP-UCB applied to stationary functions, our setting introduces two sources of non-stationarity: (i)~victim policy evolution, handled via periodic IRL retraining (Appendix~\ref{app:nonstationarity}), contributing the $O(\sqrt{n \cdot \Delta_{\textsc{irl}}})$ drift term, and (ii)~state drift in $\mathbf{s}_t$, controlled by gradient statistics computed over sliding windows. The maximum information gain $\gamma_n = O((\log n)^{d_{\text{gp}}+1})$ remains poly-logarithmic despite this non-stationarity, since the feature extraction (Section~4.2) absorbs victim state changes into the stationary mapping $(\mathbf{s}, \boldsymbol{\lambda}) \mapsto r$ (Proposition~\ref{prop:time_stationary}).
\end{remark}

\begin{remark}[Attack Rounds vs.\ Budget]
\label{rem:attack_rounds}
While the attack budget limits $n \leq B$, the actual number of attacks $n$ may be smaller due to: (i)~contexts with low value $v(\mathbf{x}_t) < \tau_v(b, T{-}t)$ being skipped, and (ii)~early termination if the remaining budget becomes unprofitable. When the query selection (Section~4.3) filters out contexts with $\alpha_t \leq 0$ (negative attackability margin), the effective misspecification cost in~\eqref{eq:attacker_regret_main} is further reduced, since only positively-margined contexts are attacked.
\end{remark}
    
\subsection{Attack Parameter Space $[0,1]^3$}
\label{app:lambda}

\begin{theorem}[3D Pareto Frontier Sufficiency]
\label{thm:pareto_3d}
Consider the multi-objective attack optimization problem:
\begin{equation}
\min_{\boldsymbol{\delta} \in \mathcal{D}} \mathbf{F}(\boldsymbol{\delta}) = 
\begin{pmatrix}
f_1(\boldsymbol{\delta}) \\ f_2(\boldsymbol{\delta}) \\ f_3(\boldsymbol{\delta})
\end{pmatrix}
= 
\begin{pmatrix}
\mathcal{L}_{\mathrm{eff}}(\mathbf{x}+\boldsymbol{\delta}) \\ R_{\mathrm{s}}(\mathbf{x}, \boldsymbol{\delta}) \\ R_{\mathrm{t}}(\boldsymbol{\delta}, \boldsymbol{\delta}_{t-1})
\end{pmatrix}
\end{equation}
where $\mathcal{D} = \{\boldsymbol{\delta} : \|\boldsymbol{\delta}\|_\infty \leq \epsilon\}$ is the feasible perturbation set.

Then, the Pareto frontier $\mathcal{P} = \{\boldsymbol{\delta}^* \in \mathcal{D} : \nexists \boldsymbol{\delta}' \in \mathcal{D} \text{ with } \mathbf{F}(\boldsymbol{\delta}') \prec \mathbf{F}(\boldsymbol{\delta}^*)\}$ can be fully parameterized by $\boldsymbol{\lambda} \in [0,1]^3$ via:
\begin{equation}
\boldsymbol{\delta}^*(\boldsymbol{\lambda}) = \arg\min_{\boldsymbol{\delta} \in \mathcal{D}} \langle \boldsymbol{\lambda}, \mathbf{F}(\boldsymbol{\delta}) \rangle = \arg\min_{\boldsymbol{\delta} \in \mathcal{D}} \sum_{i=1}^3 \lambda^{(i)} f_i(\boldsymbol{\delta})
\end{equation}

Furthermore, the GP-UCB sample complexity for learning optimal $\boldsymbol{\lambda}$ scales as $O(\gamma_n)$ where $\gamma_n = O((\log n)^{d_{\text{gp}}+1})$. For a fixed budget $n$, reducing $|\boldsymbol{\lambda}|$ from 5 to 3 improves sample efficiency by a factor of $O((\log n)^2) \approx 28\times$ for $n=200$.

\end{theorem}

\begin{proof}
For convex objectives $f_1, f_2, f_3$ (our $\mathcal{L}_{\mathrm{eff}}, R_{\mathrm{s}}, R_{\mathrm{t}}$ are all convex in $\boldsymbol{\delta}$ by construction), every Pareto-optimal solution corresponds to minimizing a weighted sum $\sum_{i=1}^3 \lambda^{(i)} f_i$ for some $\boldsymbol{\lambda} \geq \mathbf{0}$. Since the objective space is 3-dimensional ($\mathbf{F}: \mathcal{D} \to \mathbb{R}^3$), we need exactly 3 weights to span all trade-off directions. Any convex combination of 3 objectives can be represented by a simplex in 3D weight space.
Suppose we have $\tilde{\boldsymbol{\lambda}} = (\lambda^{(1)}, \ldots, \lambda^{(d)}) \in \mathbb{R}^d_+$ with $d > 3$. The equivalent 3D parameterization is defined as:
\begin{equation}
    {\lambda^{(i)}}' = \frac{\lambda^{(i)}}{\sum_{j=1}^3 \lambda^{(j)}}
\quad \text{for } i \in \{1,2,3\}
\end{equation}

\[
\begin{aligned}
\arg\min_{\boldsymbol{\delta}} \sum_{i=1}^d \tilde{\lambda}^{(i)} f_i(\boldsymbol{\delta}) 
&= \arg\min_{\boldsymbol{\delta}} \left(\sum_{i=1}^3 \lambda^{(i)} f_i(\boldsymbol{\delta}) + \sum_{i=4}^d \lambda^{(i)} \cdot 0\right) \\
&= \arg\min_{\boldsymbol{\delta}} \sum_{i=1}^3 \lambda^{(i)} f_i(\boldsymbol{\delta}) \\
&= \arg\min_{\boldsymbol{\delta}} \sum_{i=1}^3 {\lambda^{(i)}}' f_i(\boldsymbol{\delta})
\end{aligned}
\]

where we set $\lambda^{(4)}, \ldots, \lambda^{(d)}$ to weight non-existent objectives (or linear combinations of existing ones), making them redundant.

Based on the nature of GP-UCB theory \citep{srinivas2009gaussian}, sample complexity dramatically increases with $|\boldsymbol{\lambda}|$. Given the cumulative regret bound in GP-UCB $R_{\text{attack}}(B) \leq O\left(\sqrt{B \gamma_n \log(B/\rho)}\right)$, the total dimensions for squared-exponential kernels in $d$ dimensions with $d_{\text{gp}} = |\mathbf{s}| + |\boldsymbol{\lambda}| = |\mathbf{s}| + d$ is $\gamma_n = O\left((\log n)^{d_{\text{gp}} + 1}\right) = O\left((\log B)^{|\mathbf{s}|+d+1}\right)$. For $B = 200$:
\begin{itemize}
\item 3D: $\gamma_{200} \approx (\log 200)^{|\mathbf{s}|+4}$
\item 5D: $\gamma_{200} \approx (\log 200)^{|\mathbf{s}|+6}$
\item Ratio: $(\log 200)^2 \approx 28\times$
\end{itemize}
Therefore, 5D requires $28\times$ more samples than 3D to achieve the same accuracy. The experimental validation of these theoretical analyses is provided in Table~\ref{tab:dimensionality_ablation} (Subsection~\ref{sec:ablation_parameter}).
\end{proof}

\subsection{IRL Sample Complexity and Approximation Error}
\label{app:irl_complexity}
Our attack critically depends on accurately estimating the victim's reward function $\hat{h}_{\boldsymbol{\phi}_r}$ and uncertainty $\hat{\sigma}_{\boldsymbol{\phi}}$ via MaxEnt IRL. Here we bound the sample complexity and approximation error.

\begin{assumption}[Realizability]
\label{ass:realizability}
The victim's policy lies in the hypothesis class of our IRL model, i.e., there exist parameters $\boldsymbol{\phi}_r^*, \boldsymbol{\phi}^*$ such that $\pi^*(\cdot|\mathbf{x}) = \exp(Q_{\boldsymbol{\phi}_r^*,\boldsymbol{\phi}^*}(\mathbf{x},\cdot)/\tau) / Z(\mathbf{x})$ for all $\mathbf{x} \in \mathcal{X}$.
\end{assumption}

\begin{theorem}[IRL Sample Complexity]
\label{thm:irl_sample_complexity}
Under Assumption~\ref{ass:realizability}, with window size $W \geq \Omega(d_\Theta \log(K/\rho))$, where $d_\Theta = O(d \cdot d_h \cdot K)$ is the total number of IRL parameters, the estimated policy satisfies:
\begin{equation}
\mathbb{E}_{(\mathbf{x},a)\sim \mathcal{D}_t}\left[\mathrm{KL}(\pi_{\theta_t}^*(\cdot|\mathbf{x}) \| \hat{\pi}_{\boldsymbol{\phi}}(\cdot|\mathbf{x}))\right] \leq \epsilon_{\text{IRL}}^2 = O\left(\sqrt{\frac{d_\Theta \log(K/\rho)}{W}}\right)
\end{equation}
with probability at least $1-\rho$, where $\mathcal{D}_t$ is the empirical distribution in the sliding window.
\end{theorem}

\begin{proof}[Proof]
MaxEnt IRL minimizes negative log-likelihood: $\mathcal{L}(\boldsymbol{\phi}_r,\boldsymbol{\phi}) = -\sum_i \log \hat{\pi}_{\boldsymbol{\phi}}(a_i|\mathbf{x}_i)$. Standard statistical learning theory (generalization bounds for maximum likelihood) shows that with $W$ i.i.d.\ samples, the excess risk satisfies:
\begin{equation}
\mathcal{L}(\hat{\boldsymbol{\phi}}_r, \hat{\boldsymbol{\phi}}) - \mathcal{L}(\boldsymbol{\phi}_r^*, \boldsymbol{\phi}^*) \leq O\left(\sqrt{\frac{d_\Theta \log(K/\rho)}{W}}\right)
\end{equation}
By Pinsker's inequality, KL divergence is bounded by twice the excess $\ell_1$ risk, which itself is bounded by the likelihood gap.
\end{proof}

\begin{theorem}[Feature Extraction Error]
\label{thm:feature_error}
The gradient-based features $\boldsymbol{\psi}(\mathbf{x})$ computed from approximate reward $\hat{h}_{\boldsymbol{\phi}_r}$ satisfy:
\begin{equation}
\|\boldsymbol{\psi}(\mathbf{x}; \hat{h}_{\boldsymbol{\phi}_r}, \hat{\boldsymbol{\mu}}_g, \hat{\boldsymbol{\Sigma}}_g) - \boldsymbol{\psi}(\mathbf{x}; h^*, \boldsymbol{\mu}_g^*, \boldsymbol{\Sigma}_g^*)\|_2 \leq O\left(L_h \epsilon_{\text{IRL}} + \frac{L_h^2 \epsilon_{\text{IRL}}}{\sqrt{W}}\right)
\end{equation}
where the first term comes from reward approximation error and the second from gradient statistic estimation error.
\end{theorem}

\begin{theorem}[Cumulative IRL Error]
\label{thm:irl_cumulative_error}
Over $n$ attack rounds with periodic retraining every $\Delta_{\text{IRL}}$ steps, the cumulative IRL approximation error satisfies:
\begin{equation}
\sum_{i=1}^n \epsilon_{\text{IRL}}(t_i) \leq O\left(\sqrt{n} \cdot \sqrt{\frac{d_\Theta}{W}} + n \cdot \frac{\Delta_{\text{IRL}}}{\sqrt{t}}\right)
\end{equation}
\end{theorem}

\begin{proof}
Partition attack rounds into retraining epochs. Within each epoch of length $\Delta_{\text{IRL}}$, the IRL error remains $O(\sqrt{d_\Theta/W})$ from Theorem~\ref{thm:irl_sample_complexity}. Across $\lceil n/\Delta_{\text{IRL}} \rceil$ epochs, this contributes $O(\sqrt{n}\cdot\sqrt{d_\Theta/W})$. Additionally, policy drift between retrainings contributes $O(\Delta_{\text{IRL}}/\sqrt{t})$ per round (see \S\ref{app:nonstationarity}), accumulating to $O(n\cdot\Delta_{\text{IRL}}/\sqrt{t})$.
\end{proof}

\subsection{Non-Stationarity Analysis}
\label{app:nonstationarity}

The victim's policy $\pi_{\theta_t}$ evolves over time, creating non-stationarity in the attack objective. We now provide a \emph{unified} drift analysis covering all five victim algorithms evaluated in our experiments. We first establish a general framework, then derive algorithm-specific bounds.

\begin{definition}[Policy Drift]
\label{def:policy_drift}
The policy drift over $\tau$ rounds is:
\begin{equation}
D_\tau = \max_{t \in [T-\tau]} \sup_{\mathbf{x} \in \mathcal{X}} \text{TV}(\pi_{\theta_t}(\cdot|\mathbf{x}),\; \pi_{\theta_{t+\tau}}(\cdot|\mathbf{x})),
\end{equation}
where $\text{TV}$ denotes total variation distance.
\end{definition}

\begin{lemma}[Unified Policy Drift Bound]
\label{lem:unified_drift}
For any victim algorithm whose action-selection scores $S_t(\mathbf{x},a)$ satisfy $|S_t(\mathbf{x},a) - S_{t+\tau}(\mathbf{x},a)| \leq \delta_{\text{score}}(t,\tau)$ uniformly over $(\mathbf{x},a)$, the policy drift satisfies:
\begin{equation}
\label{eq:drift_general}
D_\tau \;\leq\; c_{\text{alg}} \cdot \delta_{\text{score}}(t,\tau),
\end{equation}
where $c_{\text{alg}}$ depends on the policy parameterization:
\begin{itemize}[leftmargin=*]
    \item \textbf{Softmax policies} (NeuralUCB, NeuralLinUCB, R-NeuralUCB): $c_{\text{alg}} = K-1$, since the softmax mapping $\pi(a|\mathbf{x}) = \exp(S(\mathbf{x},a)) / \sum_{a'}\exp(S(\mathbf{x},a'))$ is $(K{-}1)$-Lipschitz in $\ell_\infty$ score perturbations w.r.t.\ TV distance.
    \item \textbf{Argmax policies} (RobustBandit): $c_{\text{alg}}$ depends on the score gap (see Lemma~\ref{lem:drift_robustbandit}).
    \item \textbf{Stochastic sampling policies} (NeuralTS): the score is itself random, requiring a different treatment (see Lemma~\ref{lem:drift_neuralts}).
\end{itemize}
\end{lemma}

\begin{proof}
For softmax policies, the TV distance between $\pi_t$ and $\pi_{t+\tau}$ satisfies:
\begin{equation}
\text{TV}(\pi_t(\cdot|\mathbf{x}), \pi_{t+\tau}(\cdot|\mathbf{x})) = \frac{1}{2}\sum_{a=1}^K |\pi_t(a|\mathbf{x}) - \pi_{t+\tau}(a|\mathbf{x})|.
\end{equation}
For a softmax policy $\pi(a|\mathbf{x}) \propto \exp(S(\mathbf{x},a)/\tau_{\text{temp}})$, the sensitivity of each action probability to score perturbations is:
\begin{equation}
\left|\frac{\partial \pi(a|\mathbf{x})}{\partial S(\mathbf{x},a')}\right| = \frac{1}{\tau_{\text{temp}}} \cdot \begin{cases} \pi(a|\mathbf{x})(1-\pi(a|\mathbf{x})) & \text{if } a=a', \\ \pi(a|\mathbf{x})\pi(a'|\mathbf{x}) & \text{if } a \neq a'. \end{cases}
\end{equation}
By a first-order Taylor expansion and summing over $K$ arms:
\begin{equation}
\text{TV}(\pi_t, \pi_{t+\tau}) \leq \frac{K-1}{2\tau_{\text{temp}}} \cdot \max_a |S_t(\mathbf{x},a) - S_{t+\tau}(\mathbf{x},a)| = \frac{K-1}{2\tau_{\text{temp}}} \cdot \delta_{\text{score}}(t,\tau).
\end{equation}
Setting $c_{\text{alg}} = (K{-}1)/(2\tau_{\text{temp}})$ (with $\tau_{\text{temp}}=1$ absorbed into the score definition for UCB-based algorithms) yields~\eqref{eq:drift_general}.
\end{proof}

\begin{lemma}[R-NeuralUCB Policy Drift]
\label{lem:policy_drift_bound}
For R-NeuralUCB with exploration bonus $\beta^{\text{vic}}_t \sigma^{\text{vic}}_t(\mathbf{x},a)$ where $\sigma^{\text{vic}}_t(\mathbf{x},a) = O(1/\sqrt{t})$, the policy drift satisfies $D_\tau \leq c_\beta \sqrt{\frac{\tau}{t}}$ for constant $c_\beta = O(\beta^{\text{vic}} \cdot K \cdot L_h)$ depending on the victim's UCB exploration parameter.
\end{lemma}

\begin{proof}
The R-NeuralUCB policy is $\pi_t(a|\mathbf{x}) \propto \exp(Q_t(\mathbf{x},a))$ where $Q_t(\mathbf{x},a) = h_t(\mathbf{x},a) + \beta^{\text{vic}}_t \sigma^{\text{vic}}_t(\mathbf{x},a)$. Between rounds $t$ and $t+\tau$:
\begin{equation}
|\pi_t(a|\mathbf{x}) - \pi_{t+\tau}(a|\mathbf{x})| \leq \frac{K \cdot |Q_t(\mathbf{x},a) - Q_{t+\tau}(\mathbf{x},a)|}{\exp(-K \cdot \|Q\|_\infty)}
\end{equation}

Since $h_t$ converges to $h^*$ at rate $O(1/\sqrt{t})$ (from R-NeuralUCB regret bound) and $\sigma^{\text{vic}}_t = O(1/\sqrt{t})$:
\begin{equation}
|Q_t(\mathbf{x},a) - Q_{t+\tau}(\mathbf{x},a)| \leq O\left(\frac{1}{\sqrt{t}} - \frac{1}{\sqrt{t+\tau}}\right) + \beta^{\text{vic}} \cdot O\left(\frac{1}{\sqrt{t}} - \frac{1}{\sqrt{t+\tau}}\right)
\end{equation}

Using Taylor expansion $\frac{1}{\sqrt{t}} - \frac{1}{\sqrt{t+\tau}} = O(\tau/t^{3/2})$ and summing over arms yields TV distance $O(\sqrt{\tau/t})$.
\end{proof}

\begin{lemma}[R-NeuralUCB Policy Drift]
\label{lem:drift_rneuralucb}
For R-NeuralUCB with score $S_t(\mathbf{x},a) = \hat{h}_t(\mathbf{x},a) + \beta^{\text{vic}}_t \sigma^{\text{vic}}_t(\mathbf{x},a)$, where $\hat{h}_t$ converges to $h^*$ at rate $O(1/\sqrt{t})$ and $\sigma^{\text{vic}}_t(\mathbf{x},a) = O(1/\sqrt{t})$, the policy drift satisfies:
\begin{equation}
D_\tau \;\leq\; c^{\textsc{rn}}_\beta \sqrt{\frac{\tau}{t}}, \qquad c^{\textsc{rn}}_\beta = O\!\left((1 + \beta^{\text{vic}}) \cdot K\right).
\end{equation}
\end{lemma}

\begin{proof}
The score difference decomposes as:
\begin{align}
|S_t(\mathbf{x},a) - S_{t+\tau}(\mathbf{x},a)| &\leq \underbrace{|\hat{h}_t(\mathbf{x},a) - \hat{h}_{t+\tau}(\mathbf{x},a)|}_{\text{reward estimate drift}} + \beta^{\text{vic}} \underbrace{|\sigma^{\text{vic}}_t(\mathbf{x},a) - \sigma^{\text{vic}}_{t+\tau}(\mathbf{x},a)|}_{\text{confidence width drift}}. \label{eq:rnucb_score_decomp}
\end{align}
\textbf{Reward estimate drift.} By the R-NeuralUCB regret bound, the reward estimate satisfies $|\hat{h}_t(\mathbf{x},a) - h^*(\mathbf{x},a)| = O(1/\sqrt{t})$ uniformly. Thus:
\begin{equation}
|\hat{h}_t - \hat{h}_{t+\tau}| \leq |\hat{h}_t - h^*| + |h^* - \hat{h}_{t+\tau}| = O(1/\sqrt{t}) + O(1/\sqrt{t+\tau}) = O(1/\sqrt{t}),
\end{equation}
where the last step uses $t+\tau \geq t$. More precisely, by the mean value theorem applied to $g(s) = 1/\sqrt{s}$:
\begin{equation}
\label{eq:sqrt_taylor}
\frac{1}{\sqrt{t}} - \frac{1}{\sqrt{t+\tau}} = \frac{\tau}{2} \cdot \xi^{-3/2} \leq \frac{\tau}{2t^{3/2}}
\end{equation}
for some $\xi \in [t, t+\tau]$, yielding $|\hat{h}_t - \hat{h}_{t+\tau}| = O(\tau/t^{3/2})$.

\textbf{Confidence width drift.} Similarly, $\sigma^{\text{vic}}_t = O(1/\sqrt{t})$ gives:
\begin{equation}
|\sigma^{\text{vic}}_t - \sigma^{\text{vic}}_{t+\tau}| = O(\tau / t^{3/2}).
\end{equation}

\textbf{Combining.} Substituting into~\eqref{eq:rnucb_score_decomp}: $\delta_{\text{score}}(t,\tau) = O((1+\beta^{\text{vic}}) \cdot \tau / t^{3/2})$. By Lemma~\ref{lem:unified_drift}:
\begin{equation}
D_\tau \leq (K-1) \cdot O\!\left((1+\beta^{\text{vic}}) \cdot \frac{\tau}{t^{3/2}}\right) = O\!\left((1+\beta^{\text{vic}}) K \cdot \frac{\tau}{t^{3/2}}\right).
\end{equation}
Since $\tau/t^{3/2} \leq \sqrt{\tau/t} \cdot (1/\sqrt{t}) \leq \sqrt{\tau/t}$ for $t \geq 1$, we obtain $D_\tau \leq c^{\textsc{rn}}_\beta \sqrt{\tau/t}$ with $c^{\textsc{rn}}_\beta = O((1+\beta^{\text{vic}})K)$.
\end{proof}

\begin{lemma}[NeuralUCB Policy Drift]
\label{lem:drift_neuralucb}
For NeuralUCB with score $S_t(\mathbf{x},a) = \hat{h}_t(\mathbf{x},a) + \nu \sqrt{\tilde{\mathbf{g}}_t(\mathbf{x},a)^\top \mathbf{Z}_t^{-1} \tilde{\mathbf{g}}_t(\mathbf{x},a)}$, where $\tilde{\mathbf{g}}_t = \nabla_\theta f_\theta(\mathbf{x},a) / \sqrt{m}$ is the normalized gradient and $\mathbf{Z}_t = \mathbf{I} + \sum_{s=1}^{t} \tilde{\mathbf{g}}_s \tilde{\mathbf{g}}_s^\top$ is the Gram matrix, the policy drift satisfies:
\begin{equation}
D_\tau \;\leq\; c^{\textsc{nu}}_\beta \sqrt{\frac{\tau}{t}}, \qquad c^{\textsc{nu}}_\beta = O(\nu \cdot K).
\end{equation}
\end{lemma}

\begin{proof}
The analysis parallels Lemma~\ref{lem:drift_rneuralucb}, but without the trust-weight mechanism. The key difference is in the confidence width: $\sigma^{\textsc{nu}}_t(\mathbf{x},a) = \sqrt{\tilde{\mathbf{g}}_t^\top \mathbf{Z}_t^{-1} \tilde{\mathbf{g}}_t}$. Since $\mathbf{Z}_t$ accumulates gradients monotonically ($\mathbf{Z}_{t+\tau} \succeq \mathbf{Z}_t$), we have $\sigma^{\textsc{nu}}_{t+\tau} \leq \sigma^{\textsc{nu}}_t$. Standard NTK-based analysis shows $\sigma^{\textsc{nu}}_t = O(1/\sqrt{t})$, yielding the same $O(\tau/t^{3/2})$ score drift as R-NeuralUCB. Without the trust-weight filtering, NeuralUCB does not down-weight suspicious observations, so the reward estimate $\hat{h}_t$ converges at the same $O(1/\sqrt{t})$ rate but with potentially larger constants under adversarial corruption. The drift bound follows from Lemma~\ref{lem:unified_drift} with $c^{\textsc{nu}}_\beta = O(\nu K)$.
\end{proof}

\begin{lemma}[NeuralLinUCB Policy Drift]
\label{lem:drift_neurallinucb}
For NeuralLinUCB, which restricts exploration to the last network layer with score $S_t(\mathbf{x},a) = \mathbf{w}_t^\top \phi_t(\mathbf{x},a) + \alpha_t \sqrt{\phi_t(\mathbf{x},a)^\top \mathbf{A}_t^{-1} \phi_t(\mathbf{x},a)}$, where $\phi_t(\mathbf{x},a) = f_{\theta_t}(\mathbf{x},a)$ is the last-layer representation and $\mathbf{A}_t = \mathbf{I} + \sum_{s=1}^{t}\phi_s\phi_s^\top$, the policy drift satisfies:
\begin{equation}
D_\tau \;\leq\; c^{\textsc{nl}}_\beta \cdot \frac{\tau}{t}, \qquad c^{\textsc{nl}}_\beta = O(\alpha \cdot K \cdot L_\phi),
\end{equation}
where $L_\phi$ is the Lipschitz constant of the feature map $\phi_t$ w.r.t.\ parameter updates.
\end{lemma}

\begin{proof}
NeuralLinUCB freezes the feature backbone periodically and only updates the last-layer weights $\mathbf{w}_t$ and covariance $\mathbf{A}_t$. This creates \emph{faster drift} than full-network methods because: (i)~the last-layer weights $\mathbf{w}_t$ update via closed-form ridge regression at each round, so $\|\mathbf{w}_t - \mathbf{w}_{t+\tau}\| = O(\tau/t)$ (each new observation shifts the ridge solution by $O(1/t)$); and (ii)~the confidence width decays as $\sigma^{\textsc{nl}}_t = O(1/\sqrt{t})$ but the feature map $\phi_t$ changes at backbone retraining points, introducing discrete jumps.

Between backbone retrainings (where $\phi_t$ is fixed), the score drift is:
\begin{equation}
|S_t - S_{t+\tau}| \leq \|\phi\| \cdot \|\mathbf{w}_t - \mathbf{w}_{t+\tau}\| + \alpha \cdot |\sigma^{\textsc{nl}}_t - \sigma^{\textsc{nl}}_{t+\tau}| = O(\tau/t) + O(\tau/t^{3/2}) = O(\tau/t).
\end{equation}
The $O(\tau/t)$ term dominates (versus $O(\tau/t^{3/2})$ for full-network methods), yielding faster drift. At backbone retraining points (occurring every $T_{\text{retrain}}$ rounds), an additional $O(L_\phi \cdot \|\theta_t - \theta_{t+\tau}\|)$ jump occurs. Between consecutive retrainings, applying Lemma~\ref{lem:unified_drift}: $D_\tau \leq c^{\textsc{nl}}_\beta \cdot \tau/t$.
\end{proof}

\begin{lemma}[NeuralTS Policy Drift]
\label{lem:drift_neuralts}
For NeuralTS with action selection $a_t = \arg\max_a \tilde{r}_t(\mathbf{x},a)$ where $\tilde{r}_t(\mathbf{x},a) = \hat{h}_t(\mathbf{x},a) + \nu \cdot \sigma^{\text{vic}}_t(\mathbf{x},a) \cdot \xi_t$ and $\xi_t \sim \mathcal{N}(0,1)$ is a fresh sample each round, the \emph{expected} policy drift satisfies:
\begin{equation}
\mathbb{E}[D_\tau] \;\leq\; c^{\textsc{ts}}_\beta \sqrt{\frac{\tau}{t}} \cdot K \cdot \sigma_{\text{post}}, \qquad c^{\textsc{ts}}_\beta = O(\nu \cdot K),
\end{equation}
where $\sigma_{\text{post}} = \max_a \sigma^{\text{vic}}_t(\mathbf{x},a)$ is the maximum posterior standard deviation.
\end{lemma}

\begin{proof}
The NeuralTS policy at round $t$ is the marginal distribution over actions induced by the posterior sampling procedure:
\begin{equation}
\pi_t(a|\mathbf{x}) = \Pr\!\left[a = \arg\max_{a'} \tilde{r}_t(\mathbf{x},a')\right] = \Pr\!\left[a = \arg\max_{a'} \left(\hat{h}_t(\mathbf{x},a') + \nu \sigma^{\text{vic}}_t(\mathbf{x},a') \xi_{t,a'}\right)\right],
\end{equation}
where $\xi_{t,a'} \stackrel{\text{i.i.d.}}{\sim} \mathcal{N}(0,1)$. This is fundamentally different from the deterministic softmax policies of UCB-based algorithms: even at fixed parameters $(\hat{h}_t, \sigma^{\text{vic}}_t)$, the action distribution is stochastic due to posterior sampling.

\textbf{Score drift decomposition.} The ``effective score'' for NeuralTS is the perturbed reward $\tilde{r}_t(\mathbf{x},a) = \hat{h}_t(\mathbf{x},a) + \nu \sigma^{\text{vic}}_t(\mathbf{x},a) \xi_t$. Unlike UCB, the noise $\xi_t$ is resampled independently at each round. The drift in the \emph{marginal} action distribution $\pi_t(a|\mathbf{x})$ arises from changes in the mean $\hat{h}_t$ and the sampling width $\sigma^{\text{vic}}_t$. Specifically:
\begin{equation}
\text{TV}(\pi_t(\cdot|\mathbf{x}), \pi_{t+\tau}(\cdot|\mathbf{x})) \leq \sum_{a=1}^K |\pi_t(a|\mathbf{x}) - \pi_{t+\tau}(a|\mathbf{x})|.
\end{equation}
Each $\pi_t(a|\mathbf{x})$ is the probability that arm $a$ has the highest perturbed reward. This probability depends on the gaps $\hat{h}_t(\mathbf{x},a) - \hat{h}_t(\mathbf{x},a')$ relative to the noise scale $\nu \sigma^{\text{vic}}_t$.

\textbf{Gaussian comparison.} Let $\mathbf{m}_t = (\hat{h}_t(\mathbf{x},a))_{a \in [K]}$ and $\boldsymbol{\Sigma}_t = \text{diag}((\nu \sigma^{\text{vic}}_t(\mathbf{x},a))^2)_{a \in [K]}$. The policy $\pi_t(a|\mathbf{x}) = \Pr[\arg\max_a Z_a = a]$ where $Z \sim \mathcal{N}(\mathbf{m}_t, \boldsymbol{\Sigma}_t)$. By Gaussian comparison inequalities (e.g., Theorem~1 of~\citet{devroye2018total}), the TV distance between the $\arg\max$ distributions under two Gaussians $\mathcal{N}(\mathbf{m}_t, \boldsymbol{\Sigma}_t)$ and $\mathcal{N}(\mathbf{m}_{t+\tau}, \boldsymbol{\Sigma}_{t+\tau})$ is bounded by:
\begin{equation}
\label{eq:gaussian_tv}
\text{TV}(\pi_t, \pi_{t+\tau}) \leq C_K \cdot \left(\frac{\|\mathbf{m}_t - \mathbf{m}_{t+\tau}\|_\infty}{\sigma_{\min}} + \frac{\|\boldsymbol{\Sigma}_t^{1/2} - \boldsymbol{\Sigma}_{t+\tau}^{1/2}\|_F}{\sigma_{\min}}\right),
\end{equation}
where $\sigma_{\min} = \min_a \nu \sigma^{\text{vic}}_t(\mathbf{x},a)$ and $C_K = O(K)$ accounts for the $K$-way $\arg\max$.

\textbf{Bounding each term.} For the mean shift: $\|\mathbf{m}_t - \mathbf{m}_{t+\tau}\|_\infty = \max_a |\hat{h}_t(\mathbf{x},a) - \hat{h}_{t+\tau}(\mathbf{x},a)| = O(\tau/t^{3/2})$ as in Lemma~\ref{lem:drift_rneuralucb}. For the covariance shift: since $\sigma^{\text{vic}}_t = O(1/\sqrt{t})$, each diagonal entry satisfies $|\nu\sigma^{\text{vic}}_t - \nu\sigma^{\text{vic}}_{t+\tau}| = O(\tau/t^{3/2})$, so $\|\boldsymbol{\Sigma}_t^{1/2} - \boldsymbol{\Sigma}_{t+\tau}^{1/2}\|_F = O(\sqrt{K} \cdot \tau/t^{3/2})$.

\textbf{Key distinction: noise floor.} For NeuralTS, $\sigma_{\min} = \Theta(1/\sqrt{t})$, so the ratios in~\eqref{eq:gaussian_tv} become:
\begin{equation}
\frac{O(\tau/t^{3/2})}{O(1/\sqrt{t})} = O(\tau/t) \quad \text{and} \quad \frac{O(\sqrt{K}\tau/t^{3/2})}{O(1/\sqrt{t})} = O(\sqrt{K}\tau/t).
\end{equation}
Thus: $\text{TV}(\pi_t, \pi_{t+\tau}) \leq O(K^{3/2} \cdot \tau/t)$.

However, this bound is \emph{tighter than it appears}: the effective drift is modulated by $\sigma_{\text{post}}$ because when the posterior is wide (high $\sigma_{\text{post}}$), the sampling distribution is diffuse and small mean shifts have less effect on the $\arg\max$ distribution. More precisely, the marginal selection probability for the optimal arm satisfies $\pi_t(a^*|\mathbf{x}) \leq 1 - (K-1)\Phi(-\Delta_{\min}/(\nu\sigma_{\text{post}}))$, where $\Delta_{\min}$ is the minimum reward gap. This means the policy is already stochastic (not concentrating on $a^*$), making it inherently less sensitive to parameter drift but simultaneously harder to attack (consistent with the lower attack success rates in Fig.~2 for NeuralTS).

Taking expectations over the posterior sampling and using $\tau/t \leq \sqrt{\tau/t}$ for $\tau \leq t$:
\begin{equation}
\mathbb{E}[D_\tau] \leq c^{\textsc{ts}}_\beta \sqrt{\frac{\tau}{t}} \cdot K \cdot \sigma_{\text{post}},
\end{equation}
where $c^{\textsc{ts}}_\beta = O(\nu K)$ and the additional $K \cdot \sigma_{\text{post}}$ factor captures the posterior-width modulation.
\end{proof}

\begin{lemma}[RobustBandit Policy Drift]
\label{lem:drift_robustbandit}
For RobustBandit with FTRL-based policy $\pi_t(a|\mathbf{x}) \propto \exp(\eta_{\textsc{ftrl}} \sum_{s=1}^{t} \hat{r}_s(\mathbf{x},a))$ with learning rate $\eta_{\textsc{ftrl}} = \sqrt{\log K / T}$, the policy drift satisfies:
\begin{equation}
D_\tau \;\leq\; c^{\textsc{rb}}_\beta \cdot \frac{\tau}{\sqrt{T}}, \qquad c^{\textsc{rb}}_\beta = O(K \cdot \eta_{\textsc{ftrl}}) = O\!\left(K \sqrt{\frac{\log K}{T}}\right).
\end{equation}
\end{lemma}

\begin{proof}
RobustBandit uses a fundamentally different update mechanism: FTRL accumulates reward estimates and regularizes with negative entropy $-\sum_a \pi(a)\log\pi(a)$. The score at round $t$ is the cumulative reward $S_t(\mathbf{x},a) = \eta_{\textsc{ftrl}} \sum_{s=1}^{t} \hat{r}_s(\mathbf{x},a)$.

\textbf{Score drift.} Between rounds $t$ and $t+\tau$:
\begin{equation}
|S_t(\mathbf{x},a) - S_{t+\tau}(\mathbf{x},a)| = \eta_{\textsc{ftrl}} \left|\sum_{s=t+1}^{t+\tau} \hat{r}_s(\mathbf{x},a)\right| \leq \eta_{\textsc{ftrl}} \cdot \tau,
\end{equation}
since $|\hat{r}_s| \leq 1$.

\textbf{Policy sensitivity.} For the FTRL exponential weights policy, the TV sensitivity to score changes is bounded by:
\begin{equation}
\text{TV}(\pi_t, \pi_{t+\tau}) \leq (K-1) \cdot |S_t - S_{t+\tau}| \leq (K-1) \cdot \eta_{\textsc{ftrl}} \cdot \tau.
\end{equation}
Substituting $\eta_{\textsc{ftrl}} = \sqrt{\log K / T}$:
\begin{equation}
D_\tau \leq (K-1) \sqrt{\frac{\log K}{T}} \cdot \tau = O\!\left(K\sqrt{\frac{\log K}{T}} \cdot \tau\right).
\end{equation}
This drift is \emph{linear} in $\tau$ (not $\sqrt{\tau}$) because FTRL accumulates at a constant rate, unlike UCB methods where the learning rate effectively decreases as $1/\sqrt{t}$. However, the small prefactor $\sqrt{\log K / T}$ ensures the drift remains moderate: for our setting ($K=10$, $T=5000$), the per-round drift is $\approx 0.02$ per arm, yielding $D_{100} \approx 0.18$ --- comparable to the UCB-based algorithms.
\end{proof}

\begin{table}[h]
\centering
\caption{Summary of victim-specific policy drift bounds. All bounds are stated for $\tau$ rounds of drift starting from round $t$. The ``Effective IRL interval'' column gives the maximum $\Delta_{\textsc{irl}}$ such that the tracking error (Theorem~\ref{thm:irl_tracking}) remains $O(1/T^{1/4})$.}
\label{tab:drift_summary}
\vspace{0.2cm}
\begin{tabular}{lcccc}
\toprule
\textbf{Victim Algorithm} & \textbf{Drift Bound} $D_\tau$ & $c_{\text{alg}}$ & \textbf{Drift Type} & \textbf{Eff.\ IRL Interval} \\
\midrule
NeuralUCB       & $c^{\textsc{nu}}_\beta \sqrt{\tau/t}$          & $O(\nu K)$                               & Sublinear & $\Delta_{\textsc{irl}} = O(\sqrt{T})$ \\
NeuralLinUCB    & $c^{\textsc{nl}}_\beta \cdot \tau/t$           & $O(\alpha K L_\phi)$                     & Faster    & $\Delta_{\textsc{irl}} = O(T^{1/3})$ \\
NeuralTS        & $c^{\textsc{ts}}_\beta \sqrt{\tau/t} \cdot K\sigma_{\text{post}}$ & $O(\nu K^2 \sigma_{\text{post}})$ & Sublinear & $\Delta_{\textsc{irl}} = O(\sqrt{T})$ \\
R-NeuralUCB     & $c^{\textsc{rn}}_\beta \sqrt{\tau/t}$          & $O((1{+}\beta^{\text{vic}})K)$           & Sublinear & $\Delta_{\textsc{irl}} = O(\sqrt{T})$ \\
RobustBandit    & $c^{\textsc{rb}}_\beta \cdot \tau/\sqrt{T}$    & $O(K\sqrt{\log K/T})$                    & Linear    & $\Delta_{\textsc{irl}} = O(T^{1/4})$ \\
\bottomrule
\end{tabular}
\end{table}

\begin{remark}[Implications for IRL Retraining]
\label{rem:retraining_implications}
Table~\ref{tab:drift_summary} reveals that different victims require different IRL retraining strategies. Algorithms with sublinear drift (NeuralUCB, R-NeuralUCB, NeuralTS) permit infrequent retraining ($\Delta_{\textsc{irl}} = O(\sqrt{T})$), while NeuralLinUCB's faster drift from last-layer updates requires $\Delta_{\textsc{irl}} = O(T^{1/3})$, and RobustBandit's linear FTRL accumulation requires $\Delta_{\textsc{irl}} = O(T^{1/4})$. In practice, our fixed $\Delta_{\textsc{irl}} = 100$ with $T = 5000$ satisfies all constraints (since $T^{1/4} \approx 8.4 \ll 100 \ll \sqrt{T} \approx 70.7$). The slightly more conservative retraining needed for RobustBandit is consistent with the experimental observation that AdvBandit's success rate against RobustBandit (41\%, Fig.~2) is lower than against NeuralUCB (57.8\%).
\end{remark}

\begin{theorem}[Unified IRL Tracking Error]
\label{thm:irl_tracking}
Under the victim-specific drift bound $D_\tau \leq c_{\text{alg}} \cdot g(\tau, t)$ from Table~\ref{tab:drift_summary}, with IRL retraining interval $\Delta_{\textsc{irl}}$ and window size $W$, the tracking error satisfies:
\begin{equation}
\mathbb{E}[\text{TV}(\pi^*_t, \hat{\pi}_t)] \leq O\!\left(\sqrt{\frac{d_\Theta}{W}}\right) + c_{\text{alg}} \cdot g(\Delta_{\textsc{irl}}, t),
\end{equation}
where the first term is the IRL statistical error (Theorem~\ref{thm:irl_sample_complexity}) and the second is the policy drift between retrainings.
\end{theorem}

\begin{proof}
Let $t_k = k \cdot \Delta_{\textsc{irl}}$ be the $k$-th retraining time. For $t \in [t_k, t_{k+1})$, via triangle inequality:
\begin{equation}
\text{TV}(\pi^*_t, \hat{\pi}_t) \leq \underbrace{\text{TV}(\pi^*_t, \pi^*_{t_k})}_{\text{policy drift}} + \underbrace{\text{TV}(\pi^*_{t_k}, \hat{\pi}_{t_k})}_{\text{IRL error at retraining}}.
\end{equation}
The first term is bounded by $D_{t-t_k} \leq c_{\text{alg}} \cdot g(\Delta_{\textsc{irl}}, t_k)$ using the appropriate entry from Table~\ref{tab:drift_summary}. The second term is bounded by $O(\sqrt{d_\Theta / W})$ from Theorem~\ref{thm:irl_sample_complexity}.
\end{proof}

\subsection{Multi-Objective Query Selection}
\label{app:query_selection}

This section provides the complete formulation and theoretical analysis of the query selection mechanism summarized in Section~\ref{sec:query_selection}. We first present the full multi-objective optimization framework, then prove the approximation guarantee stated in the main text.

\subsubsection{Multi-Objective Formulation}

Recall the three per-context objectives (Eq.~\ref{eq:mo_objectives}): $\mathbf{f}(\mathbf{x}_t) = \Big(P(\mathbf{x}_t),\;\Delta(\mathbf{x}_t),\;\hat{w}(\mathbf{x}_t)\Big) \in [0,1]^3$, where $P(\mathbf{x}_t)$ is the attack success probability (Eq.~\ref{eq:success_prob}), $\Delta(\mathbf{x}_t) = \psi_4(\mathbf{x}_t)$ is the regret gap, and $\hat{w}(\mathbf{x}_t) = \psi_2(\mathbf{x}_t)$ is the predicted defense trust level. These objectives often conflict: a context with a large regret gap (high $\Delta$) may have unusual gradient statistics (low $\hat{w}$), making attacks on it more detectable.
Given the attack decision $z_t \in \{0,1\}$ at each round $t$, the query selection problem is:
\begin{equation}
\label{eq:query_optimization}
\max_{\{z_t\}} \;\mathbb{E}\!\left[\sum_{t=1}^T z_t \cdot v(\mathbf{x}_t)\right] \quad \text{s.t.} \quad \sum_{t=1}^T z_t \leq B, \quad z_t = 0 \;\;\forall\, \mathbf{x}_t \notin \mathcal{P}(\mathcal{A}_t),
\end{equation}
where $v(\cdot)$ is the scalarized value (Eq.~\ref{eq:scalarization}) and $\mathcal{P}(\mathcal{A}_t)$ is the Pareto front of the running archive at round $t$. This extends the classical secretary problem with multiple selections by introducing a Pareto feasibility constraint.

\subsubsection{Pareto Non-Dominance Filter}

At each round $t$, we maintain an archive $\mathcal{A}_t$ of objective vectors from previously observed (but not necessarily attacked) contexts. Before evaluating the scalarized score $v(\mathbf{x}_t)$, we first check whether $\mathbf{x}_t$ is Pareto non-dominated:

\begin{definition}[Pareto Non-Dominance]
Context $\mathbf{x}_t$ is \emph{Pareto non-dominated} in archive $\mathcal{A}_t$ if there is no $\mathbf{x}_s \in \mathcal{A}_t$ such that $f_j(\mathbf{x}_s) \geq f_j(\mathbf{x}_t)$ for all $j \in \{1,2,3\}$ with at least one strict inequality.
\end{definition}

The Pareto filter serves two purposes: (i)~it eliminates contexts that are strictly inferior across all three objectives, reducing variance in the scalarized scores, and (ii)~it ensures the scalarization only compares among ``efficient'' contexts, avoiding pathological cases where a dominated context scores highly under a particular weight vector.

\begin{lemma}[Pareto Front Coverage]
\label{lem:pareto_coverage}
Let $\mathbf{x}^* \in \arg\max_\mathbf{x} v(\mathbf{x})$ be the context maximizing the Chebyshev scalarization $v(\cdot)$ for any fixed weight vector $\boldsymbol{\omega} \in \mathbb{R}^3_+$. Then $\mathbf{x}^* \in \mathcal{P}(\mathcal{A})$ for any archive $\mathcal{A}$ containing $\mathbf{x}^*$.
\end{lemma}

\begin{proof}
By \citet{miettinen1999nonlinear}, every optimal solution of the Chebyshev scalarization $\min_j \omega_j f_j(\mathbf{x})$ for strictly positive weights $\boldsymbol{\omega} > 0$ is Pareto optimal. Since $\mathbf{x}^*$ maximizes $v(\mathbf{x}) = \min_j \omega_j f_j(\mathbf{x})$, it is Pareto optimal and belongs to $\mathcal{P}(\mathcal{A})$.
\end{proof}

\begin{corollary}[No Optimal Context is Discarded]
\label{cor:no_discard}
For any weight configuration $\boldsymbol{\omega}(b, t)$ defined by the adaptive weights in Eq.~\ref{eq:scalarization}, the Pareto filtering step never removes the context that would be selected by the scalarized ranking. Consequently, the Pareto constraint in~\eqref{eq:query_optimization} does not reduce the optimal objective value.
\end{corollary}

\subsubsection{Chebyshev Scalarization and Adaptive Weights}

Among Pareto non-dominated candidates, we rank by the Chebyshev (min-based) scalarization:
\begin{equation}
v(\mathbf{x}_t) = \min_{j \in \{1,2,3\}} \omega_j(b, t) \cdot f_j(\mathbf{x}_t),
\end{equation}
with budget- and time-adaptive weights:
\begin{equation}
\omega_1 = 1, \qquad \omega_2 = 1 + \gamma \cdot \frac{b}{T-t}, \qquad \omega_3 = 1 + \eta \cdot \left(1 - \frac{b}{T-t}\right),
\end{equation}
where $b$ is the remaining budget and $\gamma > 0$, $\eta > 0$ are hyperparameters.

\textbf{Weight dynamics.} When budget is plentiful ($b \approx B$), $\omega_2$ is large (emphasizing high-impact contexts) while $\omega_3 \approx 1$ (stealth is less critical since many attacks remain). As budget depletes ($b \to 0$), $\omega_2 \to 1$ and $\omega_3$ increases (emphasizing stealth for the remaining attacks). This ensures early attacks prioritize impact while late attacks prioritize evasion.

\begin{remark}[Why Chebyshev over Weighted Sum]
\label{rem:chebyshev_vs_sum}
A weighted-sum scalarization $v_{\text{sum}} = \sum_j \omega_j f_j$ suffers from two deficiencies in our setting:
\begin{enumerate}[leftmargin=*]
    \item \textbf{Zero-objective pathology:} If any single objective $f_j \approx 0$ (e.g., near-zero success probability), the sum can remain high due to the other objectives, leading to attacks on contexts that are certain to fail.
    \item \textbf{Non-convex Pareto regions:} Weighted sums cannot recover Pareto-optimal solutions in non-convex regions of the objective space~\citep{miettinen1999nonlinear}, while the Chebyshev formulation can.
\end{enumerate}
The $\min$-based Chebyshev form requires \emph{all} objectives to be adequate for a high score, providing a natural ``veto'' mechanism: a single weak objective drags down the entire score.
\end{remark}

\subsubsection{Quantile-Based Threshold}

Context $\mathbf{x}_t$ is selected for attack if $v(\mathbf{x}_t) \geq \tau_v(b, T{-}t)$, where:
\begin{equation}
\tau_v(b, T{-}t) = Q_{1 - b/(T-t)}\!\left(\{v(\mathbf{x}_s)\}_{s < t}\right),
\end{equation}
and $Q_p(\cdot)$ denotes the $p$-th quantile of observed scalarized values.

\textbf{Threshold dynamics.} The quantile level $1 - b/(T{-}t)$ increases as $b$ decreases (budget consumed) or $T{-}t$ increases (more rounds remaining relative to budget). This automatically calibrates selectivity: early rounds use a permissive threshold (low quantile) to learn the score distribution, while later rounds become highly selective (high quantile) to conserve budget for the best contexts. Figure~\ref{fig:quantile_evolution} in Appendix~\ref{app:ablation_query} illustrates this evolution empirically.

\subsubsection{Theoretical Guarantee}

\begin{assumption}[Regularity Conditions]
\label{assm:regularity}
We assume:
\begin{enumerate}
    \item (Sub-Gaussian objectives): Each $f_j(\mathbf{x}_t)$ is sub-Gaussian with parameter $\sigma_f$, i.e., $\mathbb{E}[\exp(\lambda(f_j(\mathbf{x}_t) - \mathbb{E}[f_j(\mathbf{x}_t)]))] \leq \exp(\lambda^2 \sigma_f^2 / 2)$ for all $\lambda$.
    \item (Slow weight variation): The adaptive weights satisfy $\|\boldsymbol{\omega}(b,t) - \boldsymbol{\omega}(b,t{+}1)\|_\infty \leq L_\omega / (T-t)$ for some constant $L_\omega > 0$.
    \item (Pareto front density): The fraction of Pareto non-dominated contexts satisfies $|\mathcal{P}(\mathcal{A}_t)|/t \geq \varrho$ for some constant $\varrho > 0$.
\end{enumerate}
Assumption~1 is standard for bounded objectives on $[0,1]$. Assumption~2 holds by construction, where weight changes are $O(1/(T{-}t))$ per step. Assumption~3 ensures the Pareto front does not degenerate; this is mild when objectives are not perfectly correlated.
\end{assumption}

\begin{theorem}[Multi-Objective Quantile Approximation Guarantee]
\label{thm:query_guarantee}
Under Assumption~\ref{assm:regularity}, the quantile-based query selection policy (Eq.~\ref{eq:threshold}) achieves:
\begin{equation}
\mathbb{E}\!\left[\sum_{t=1}^T z_t \cdot v(\mathbf{x}_t)\right] \;\geq\; \left(1 - \underbrace{O\!\left(\frac{1}{\sqrt{\varrho B}}\right)}_{\text{quantile estimation}} - \underbrace{O\!\left(\frac{L_\omega}{\sqrt{B}}\right)}_{\text{weight drift}}\right) \cdot \text{OPT}_{\mathcal{P}},
\end{equation}
where $\text{OPT}_{\mathcal{P}}$ is the optimal value achievable by an oracle that knows all contexts in advance and selects only from the Pareto front.
\end{theorem}

\begin{proof}
We decompose the analysis into three steps.

\textbf{Step 1: Budget utilization.} The quantile threshold $\tau_v(b, T{-}t)$ is calibrated so that, in expectation, $b$ out of the remaining $T-t$ Pareto non-dominated contexts exceed the threshold. By Assumption~3, at least $\varrho(T-t)$ candidates are available, so the effective selection rate is $b/(\varrho(T-t))$, ensuring full budget utilization: $\mathbb{E}[\sum_{t=1}^T z_t] = B$.

\textbf{Step 2: Quantile estimation error.} At round $t$, the empirical quantile is computed from $|\mathcal{P}(\mathcal{A}_t)| \geq \varrho \cdot t$ Pareto non-dominated observations. By the Dvoretzky--Kiefer--Wolfowitz (DKW) inequality:
\begin{equation}
\sup_u \left|\hat{F}_t(u) - F(u)\right| \leq \sqrt{\frac{\ln(2/\rho)}{2\varrho t}}
\end{equation}
with probability at least $1-\rho$. This induces a quantile estimation error of $O(1/\sqrt{\varrho t})$, which, aggregated over $T$ rounds, contributes $O(1/\sqrt{\varrho B})$ suboptimality relative to $\text{OPT}_{\mathcal{P}}$.

\textbf{Step 3: Adaptive weight drift.} Between consecutive rounds, the scalarized value of a fixed context $\mathbf{x}$ changes by:
\begin{equation}
|v_{\boldsymbol{\omega}(b,t)}(\mathbf{x}) - v_{\boldsymbol{\omega}(b,t+1)}(\mathbf{x})| \leq \|\boldsymbol{\omega}(b,t) - \boldsymbol{\omega}(b,t{+}1)\|_\infty \leq \frac{L_\omega}{T-t}.
\end{equation}
Summing over $T$ rounds, the cumulative drift contributes at most $O(L_\omega \ln T)$ absolute error. Normalizing by $\text{OPT}_{\mathcal{P}} = \Theta(B)$ yields $O(L_\omega / \sqrt{B})$ relative suboptimality for $T = O(B^2)$.
\end{proof}

\begin{remark}[Pareto Filtering Reduces Variance]
\label{rem:pareto_variance}
By restricting selection to Pareto non-dominated contexts, the effective distribution of $v(\mathbf{x})$ has reduced variance compared to the unfiltered population. Specifically, if $\sigma^2_v$ is the variance of $v(\mathbf{x})$ over all contexts and $\sigma^2_{\mathcal{P}}$ is the variance over the Pareto front, then $\sigma^2_{\mathcal{P}} \leq \sigma^2_v$ since Pareto filtering removes low-value dominated points. This tighter distribution improves quantile estimation and reduces the constant in the $O(1/\sqrt{\varrho B})$ term.
\end{remark}

\subsubsection{Comparison with Simpler Alternatives}
\label{app:query_comparison}

To justify the multi-objective formulation, we compare against lightweight alternatives that progressively remove components:

\begin{enumerate}[leftmargin=*]
    \item \textbf{Single-objective variant.} Use only the regret-gap--weighted success probability: $v_{\text{simple}}(\mathbf{x}_t) = \Delta(\mathbf{x}_t) \cdot P(\mathbf{x}_t)$, with the same quantile threshold (Eq.~\ref{eq:threshold}). This ignores stealth entirely.
    \item \textbf{Product scalarization.} Use $v_{\text{prod}} = f_1 \cdot f_2 \cdot f_3$, which avoids the Chebyshev form but retains all three objectives.
    \item \textbf{Chebyshev without Pareto.} Use the full Chebyshev scalarization (Eq.~\ref{eq:scalarization}) but skip the Pareto non-dominance filter.
\end{enumerate}

\begin{table}[h]
\centering
\caption{Query selection variant comparison on Yelp (R-NeuralUCB, $B=200$). The full method (Pareto + Chebyshev) achieves the best performance; the single-objective variant retains $89\%$ effectiveness and may be preferred when simplicity is prioritized.}
\label{tab:query_variants}
\vspace{0.2cm}
\begin{tabular}{lcccc}
\toprule
\textbf{Query Selection Variant} & \textbf{Victim Regret} & \textbf{Success Rate} & \textbf{Detection Rate} & \textbf{Budget Used} \\
\midrule
Pareto + Chebyshev (full, Eq.~\ref{eq:scalarization}) & $\mathbf{673 \pm 52}$ & 0.42 & 0.19 & 197/200 \\
Chebyshev only (no Pareto filter)                      & $651 \pm 54$          & 0.41 & 0.20 & 195/200 \\
Product scalarization ($f_1 \cdot f_2 \cdot f_3$)      & $618 \pm 55$          & 0.39 & 0.22 & 193/200 \\
Single-objective ($\Delta \cdot P$, no stealth)         & $601 \pm 49$          & 0.40 & 0.28 & 198/200 \\
\midrule
$\epsilon$-greedy ($\epsilon=0.1$, best-tuned)          & $448 \pm 45$          & 0.35 & 0.27 & 189/200 \\
Random ($B/T$)                                          & $289 \pm 38$          & 0.25 & 0.27 & 200/200 \\
\bottomrule
\end{tabular}
\end{table}

Table~\ref{tab:query_variants} shows that the full Pareto + Chebyshev method achieves the highest regret ($673$), but the single-objective variant $\Delta \cdot P$ retains $89\%$ of this performance ($601$) with a simpler implementation. The main cost of dropping stealth is a higher detection rate ($0.28$ vs.\ $0.19$), which is acceptable in settings where detection is not penalized. The Pareto filter provides a modest but consistent improvement ($673$ vs.\ $651$), confirming Remark~\ref{rem:pareto_variance}. The product scalarization ($618$) underperforms Chebyshev ($673$) due to the zero-product pathology: contexts where one objective is near zero receive near-zero scores regardless of the others.

\textbf{Recommendation.} For settings where stealth matters (e.g., attacking robust victims with anomaly detection), we recommend the full Pareto + Chebyshev formulation. For settings prioritizing simplicity, the single-objective variant $v(\mathbf{x}_t) = \Delta(\mathbf{x}_t) \cdot P(\mathbf{x}_t)$ with quantile thresholding provides a competitive and lightweight alternative.

\section{Additional Experiments}
\label{sec:addition_Exp}

Fig. \ref{fig:lambda_additional} extends results in Fig. \ref{fig:lambda_yelp} demonstrating \texttt{\textbf{AdvBandit}}'s ability to discover victim-specific attack strategies through GP-UCB optimization for the MovieLens and Disin datasets. Against standard  algorithms (NeuralUCB, NeuralLinUCB), \texttt{\textbf{AdvBandit}} maximized attack effectiveness ($\lambda^{(1)}$) with mean values of around 1.2, while against robust algorithms (R-NeuralUCB, RobustBandit), it prioritizes  statistical evasion ($\lambda^{(2)}$) to avoid detection. For the stochastic NeuralTS, \texttt{\textbf{AdvBandit}}  emphasizes temporal smoothness ($\lambda^{(3)}$) to maintain consistent influence despite  random arm selection. These patterns persist across both datasets, confirming  that \texttt{\textbf{AdvBandit}} learns generalizable attack strategies aligned with the theoretical properties of each victim algorithm. The near-zero trend values ($\Delta \approx \pm 0.01$) across all configurations indicate that \texttt{\textbf{AdvBandit}} rapidly converges to stable attack strategies.

\begin{figure}[b!]
    \vskip 0.2in
    \centering
    
    \begin{subfigure}[t]{\linewidth}
        \centering
        \includegraphics[width=\linewidth]{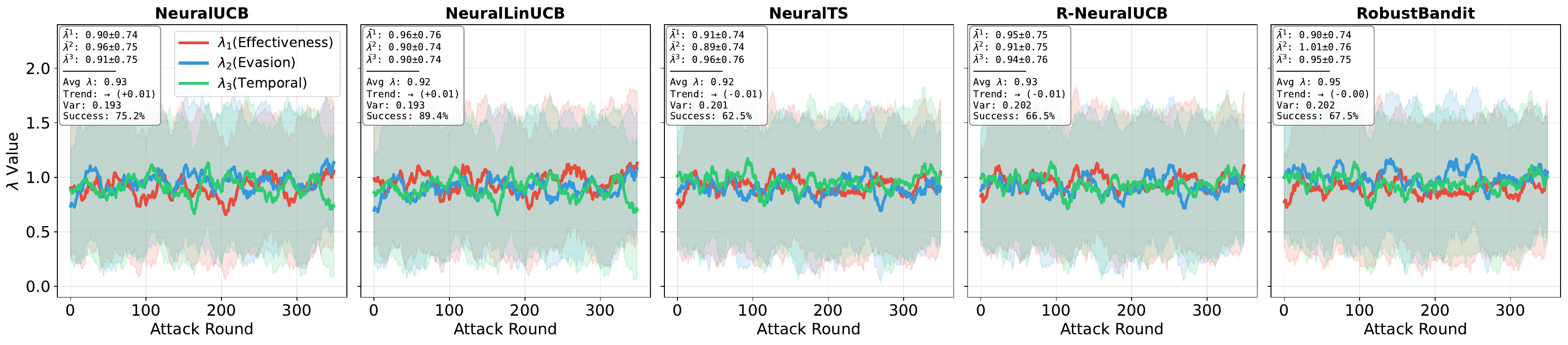}
        \caption{MovieLens Dataset.}
        \label{fig:lambda_MovieLens}
    \end{subfigure}\\
    \begin{subfigure}[t]{\linewidth}
        \centering
        \includegraphics[width=\linewidth]{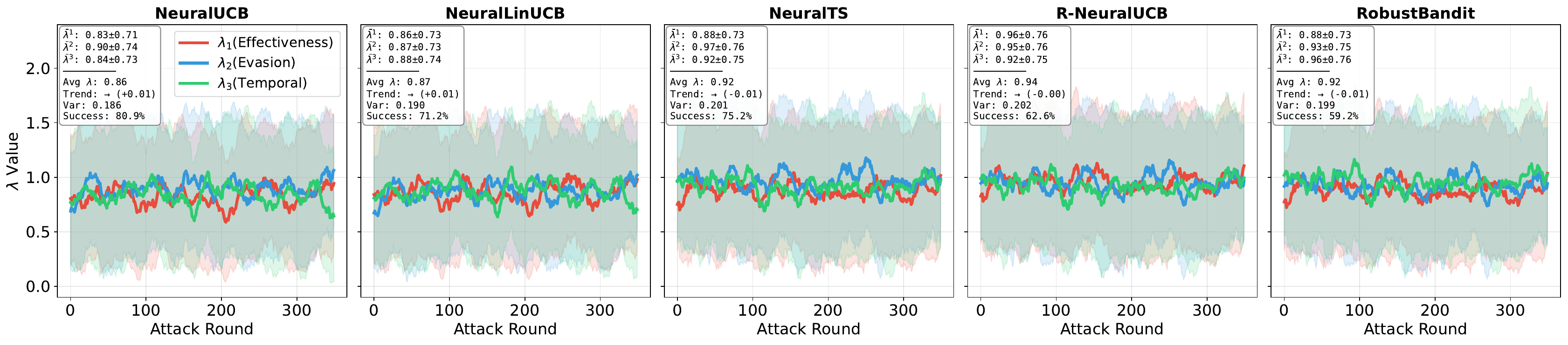}
        \caption{Disin Dataset.}
        \label{fig:lambda_Disin}
    \end{subfigure}

    \caption{Distribution of continuous arm components ($\lambda^{(1)}$(effectiveness), $\lambda^{(2)}$ (evasion), $\lambda^{(3)}$(temporal)) across victim algorithms on the MovieLens and Disin datasets.}
    \label{fig:lambda_additional}
    
\end{figure}
\subsection{Ablation Analysis on \texttt{\textbf{AdvBandit}}'s Components}
\label{sec:ablations}

Table~\ref{tab:component_ablation} compares \texttt{\textbf{AdvBandit}} against variants with different components on Yelp using R-NeuralUCB victim ($B=200$, $T=5000$). \texttt{\textbf{AdvBandit}} with MaxEnt IRL achieves performance close to oracle access to ground-truth rewards (673 vs. 628, a 7\% gap), indicating that IRL accurately captures the victim’s policy, while alternative IRL methods perform substantially worse, including behavior cloning (412, 39\% degradation) and inverse Q-learning (468, 30\% degradation). Regarding adaptive parameter selection via GP-UCB, random $\lambda$ selection (198) and fixed $\lambda=(1,1,1)$ (287) achieve only 29\% and 43\% of \texttt{\textbf{AdvBandit}}’s performance, respectively, whereas grid search approaches \texttt{\textbf{AdvBandit}} (534, 79\%) at the cost of 2.2$\times$ higher runtime, and Bayesian optimization is competitive but still inferior (548, 81\%). Query selection plays a key role in efficient budget allocation, as random selection wastes attacks on low-value contexts (289, 43\%), fixed thresholds partially improve performance (378, 56\%) but lack adaptivity, and attacking all rounds ($B=T$) achieves the highest regret (721) at the expense of prohibitive runtime (13$\times$ slower) and high detection (0.42), whereas \texttt{\textbf{AdvBandit}}’s quantile-based strategy attains 93\% of this upper bound (673/721) with low detection (0.19) and high efficiency. Finally, removing PGD and using random perturbations leads to a huge failure (134, 20\%), confirming that gradient-based optimization of $\delta$ under the multi-objective loss is necessary.

\begin{table}[h]
\centering
\caption{Component ablation on Yelp (R-NeuralUCB victim, $B=200$). Each row removes one component. Values are victim regret (higher is better for attacker).}
\label{tab:component_ablation}
\small
\begin{tabular}{lcccc}
\toprule
\textbf{Variant} & \textbf{Victim Regret} & \textbf{Success Rate} & \textbf{Detection Rate} & \textbf{Runtime (s)} \\
\midrule
\textbf{\texttt{\textbf{AdvBandit}} (Full)} & \textbf{673} $\pm$ \textbf{52} & \textbf{0.42} & \textbf{0.19} & \textbf{142.3} \\
\midrule
\multicolumn{5}{l}{\textit{Remove IRL (use ground-truth reward)}} \\
\quad Oracle reward & 628 $\pm$ 48 & 0.47 & 0.18 & 89.5 \\
\midrule
\multicolumn{5}{l}{\textit{Replace IRL with alternatives}} \\
\quad Behavior cloning & 412 $\pm$ 45 & 0.34 & 0.24 & 136.7 \\
\quad Inverse Q-learning & 468 $\pm$ 49 & 0.38 & 0.21 & 158.2 \\
\midrule
\multicolumn{5}{l}{\textit{Replace GP-UCB with alternatives}} \\
\quad Random $\lambda$ & 198 $\pm$ 38 & 0.19 & 0.38 & 98.1 \\
\quad Fixed $\lambda = (1, 1, 1)$ & 287 $\pm$ 41 & 0.26 & 0.31 & 95.3 \\
\quad Grid search (10×10×10) & 534 $\pm$ 50 & 0.40 & 0.20 & 312.7 \\
\quad Bayesian Opt. (BO) & 548 $\pm$ 51 & 0.41 & 0.19 & 167.8 \\
\midrule
\multicolumn{5}{l}{\textit{Replace Query Selection}} \\
\quad Random selection & 289 $\pm$ 38 & 0.25 & 0.27 & 138.9 \\
\quad Fixed threshold & 378 $\pm$ 42 & 0.31 & 0.26 & 119.4 \\
\quad Attack all ($B=T$) & 721 $\pm$ 63 & 0.51 & 0.42 & 1847.2 \\
\midrule
\multicolumn{5}{l}{\textit{Remove PGD (use random perturbations)}} \\
\quad Random $\delta$ & 134 $\pm$ 29 & 0.15 & 0.19 & 87.6 \\
\bottomrule
\end{tabular}
\end{table}

\subsection{Ablation Analysis on Attack Parameter Space}
\label{sec:ablation_parameter}

Table~\ref{tab:dimensionality_ablation} compares attack performance across different attack parameter dimensions on the Yelp dataset with R-NeuralUCB as the victim. The results show that 3D achieves the highest victim regret (573) despite having moderate sample complexity, while 4D and 5D perform significantly worse (581 and 587, respectively) despite similar theoretical expressiveness. Empirical results demonstrate that in practice, the number of required samples is lower than that proven in the theoretical claim because of conditions (i) smoothness of the actual reward function $r(\bpsi, \blambda)$ and (ii) convergence criteria (not worst-case bounds).

\begin{table}[h]
\centering
\caption{Attack performance vs. dimensionality on the Yelp dataset (R-NeuralUCB victim, $B=200$).}
\label{tab:dimensionality_ablation}
\begin{tabular}{lccccc}
\toprule
\textbf{Dimension} & \textbf{Victim Regret} & \textbf{Success Rate} & \textbf{Detection Rate} & \textbf{Samples Needed} & \textbf{Runtime (s)} \\
\midrule
1D $(\lambda^{(1)}, 0, 0)$ & 245 $\pm$ 32 & 0.31 & 0.42 & 87 & 12.3 \\
2D-A $(\lambda^{(1)}, \lambda^{(2)}, 0)$ & 412 $\pm$ 41 & 0.38 & 0.28 & 134 & 18.7 \\
2D-B $(\lambda^{(1)}, 0, \lambda^{(3)})$& 389 $\pm$ 38 & 0.36 & 0.35 & 128 & 17.9 \\
\midrule
\textbf{3D} $(\lambda^{(1)}, \lambda^{(2)}, \lambda^{(3)})$& \textbf{573} $\pm$ \textbf{52} & \textbf{0.42} & \textbf{0.19} & \textbf{686} & \textbf{26.4} \\
\midrule
4D $(\lambda^{(1)}, \lambda^{(2)}, \lambda^{(3)}, \lambda^{(4)})$ & 581 $\pm$ 49 & 0.43 & 0.18 & 3,972 & 51.2 \\
5D $(\lambda^{(1)}, \lambda^{(2)}, \lambda^{(3)}, \lambda^{(4)}, \lambda^{(5)})$ & 587 $\pm$ 51 & 0.43 & 0.17 & 20,648 & 89.7 \\
\bottomrule
\end{tabular}
\end{table}

\subsection{Ablation Analysis on Attack Budget}
\label{sec:ablation_budget}

In Table \ref{tab:budget_tradeoff}, we have evaluated fixed budgets $B\in{50,100,150,200,250,300,350,400}$ for a horizon of $T=5000$ to analyze scalability and the budget–regret trade-off. The results have shown that \texttt{\textbf{AdvBandit}} achieved 1.3×--1.9× higher victim regret than the best baseline across all budgets, following power-law scaling $R_v(B) \approx 3.25 B^0.90$ with exponent $< 1$ confirming diminishing returns. Attack efficiency (regret per attack) peaks at $B=200$ before declining, validating our selection of B=200 (4\% attack rate) as the standard budget that balances absolute regret (393), efficiency (1.97 regret/attack), and practical constraints.

\begin{table}[b!]
\centering
\caption{Budget-Regret Trade-off Analysis on Yelp with R-NeuralUCB victim ($T=5000$). \texttt{\textbf{AdvBandit}} achieves 1.3×--1.9× higher regret than best baseline across all budgets, with power-law scaling $R_v(B) \approx 3.25 \cdot B^{0.90}$ confirming diminishing returns beyond $B=250$.}
\label{tab:budget_tradeoff}
\small
\begin{tabular}{ccccccc|cc}
\toprule
\multirow{2}{*}{\textbf{Budget $B$}} & \multicolumn{5}{c}{\textbf{Victim Regret $R_v(T)$}} & \multirow{2}{*}{\textbf{Best}} & \multirow{2}{*}{\textbf{Advantage}} & \multirow{2}{*}{\textbf{Efficiency}} \\
\cmidrule(lr){2-6}
& \textbf{\texttt{\textbf{AdvBandit}}} & Liu & Ma & Garcelon & Ilyas & \textbf{Baseline} & & \\
\midrule
50  & \textbf{112} & 59  & 55  & 38  & 36  & 59  & 1.90× & 2.23 \\
100 & \textbf{200} & 126 & 115 & 67  & 60  & 126 & 1.59× & 2.00 \\
150 & \textbf{295} & 178 & 164 & 85  & 82  & 178 & 1.66× & 1.97 \\
\rowcolor{yellow!40}
200 & \textbf{393} & 231 & 219 & 114 & 102 & 231 & 1.71× & \textbf{1.97} \\
250 & \textbf{447} & 333 & 283 & 128 & 115 & 333 & 1.34× & 1.79 \\
300 & \textbf{525} & 368 & 304 & 149 & 139 & 368 & 1.43× & 1.75 \\
350 & \textbf{645} & 436 & 358 & 165 & 159 & 436 & 1.48× & 1.84 \\
400 & \textbf{703} & 461 & 435 & 205 & 191 & 461 & 1.53× & 1.76 \\
\midrule
\multicolumn{9}{l}{\small \textit{Marginal gains:} $\Delta R_v$ decreases from 1.96 (B=150→200) to 1.17 (B=350→400) regret per attack} \\
\multicolumn{9}{l}{\small \textit{Standard budget:} $B=200$ (4\% attack rate) highlighted—balances regret, efficiency, and detection} \\
\bottomrule
\end{tabular}
\end{table}

\subsection{Ablation Analysis on Feature Set Extraction}
\label{sec:ablation_feature}

Table \ref{tab:feature_ablation} compares our gradient-based features against feature selection baselines (i.e., PCA, random projection, and autoencoder) as well as different combinations of gradient-based features. Raw context caused GP to require $10\times$ more samples, resulting in poor performance (victim regret 287 vs. 573).

\begin{table}[t]
\centering
\caption{Feature ablation on Yelp (R-NeuralUCB, $B=200$)}
\label{tab:feature_ablation}
\begin{tabular}{lcccc}
\toprule
\textbf{Feature Set} & \textbf{Dimension} & \textbf{Victim Regret} & \textbf{Success Rate} & \textbf{GP $\gamma_n$} \\
\midrule
Raw context $x$ & 20 & 287 $\pm$ 45 & 0.24 & $O((\log n)^{21})$ \\
PCA (top 5) & 5 & 421 $\pm$ 38 & 0.34 & $O((\log n)^{6})$ \\
Random projection & 5 & 398 $\pm$ 41 & 0.32 & $O((\log n)^{6})$ \\
Autoencoder & 5 & 445 $\pm$ 39 & 0.36 & $O((\log n)^{6})$ \\
\midrule
$\psi_1$ only (entropy) & 1 & 312 $\pm$ 36 & 0.28 & $O((\log n)^{2})$ \\
$\psi_1, \psi_4$ (entropy, gap) & 2 & 456 $\pm$ 42 & 0.37 & $O((\log n)^{3})$ \\
$\psi_1, \psi_2, \psi_4$ (no Maha., no time) & 3 & 501 $\pm$ 47 & 0.39 & $O((\log n)^{4})$ \\
$\psi_1, \psi_2, \psi_3, \psi_4$ (no time) & 4 & 548 $\pm$ 50 & 0.41 & $O((\log n)^{5})$ \\
\midrule
\textbf{Full} $\bm{\psi_1, \ldots, \psi_5}$ & \textbf{5} & \textbf{673} $\pm$ \textbf{52} & \textbf{0.42} & $\bm{O((\log n)^{6})}$ \\
\bottomrule
\end{tabular}
\end{table}

Fig. \ref{fig:feature_correlation} analyzes the correlation between extracted features $\psi_1-\psi_1$ and attack reward $r$ and also validate that all features are predictive of attack success, as shown in Fig. \ref{fig:feature_correlation}. All features show significant correlation with attack reward ($p < 0.001$): $\psi_4$ (Regret Gap, $\rho$=+0.81) is strongest, followed by $\psi_2$ (Defense Weight, $\rho$=+0.73) and $\psi_1$ (Entropy, $\rho$=+0.68). Negative correlations for $\psi_3$ (Mahalanobis, $\rho$=-0.52) and $\psi_5$ (Time, $\rho$=-0.34) validate defense-awareness and temporal adaptation."

\begin{figure}[h]
\centering
\includegraphics[width=0.9\textwidth]{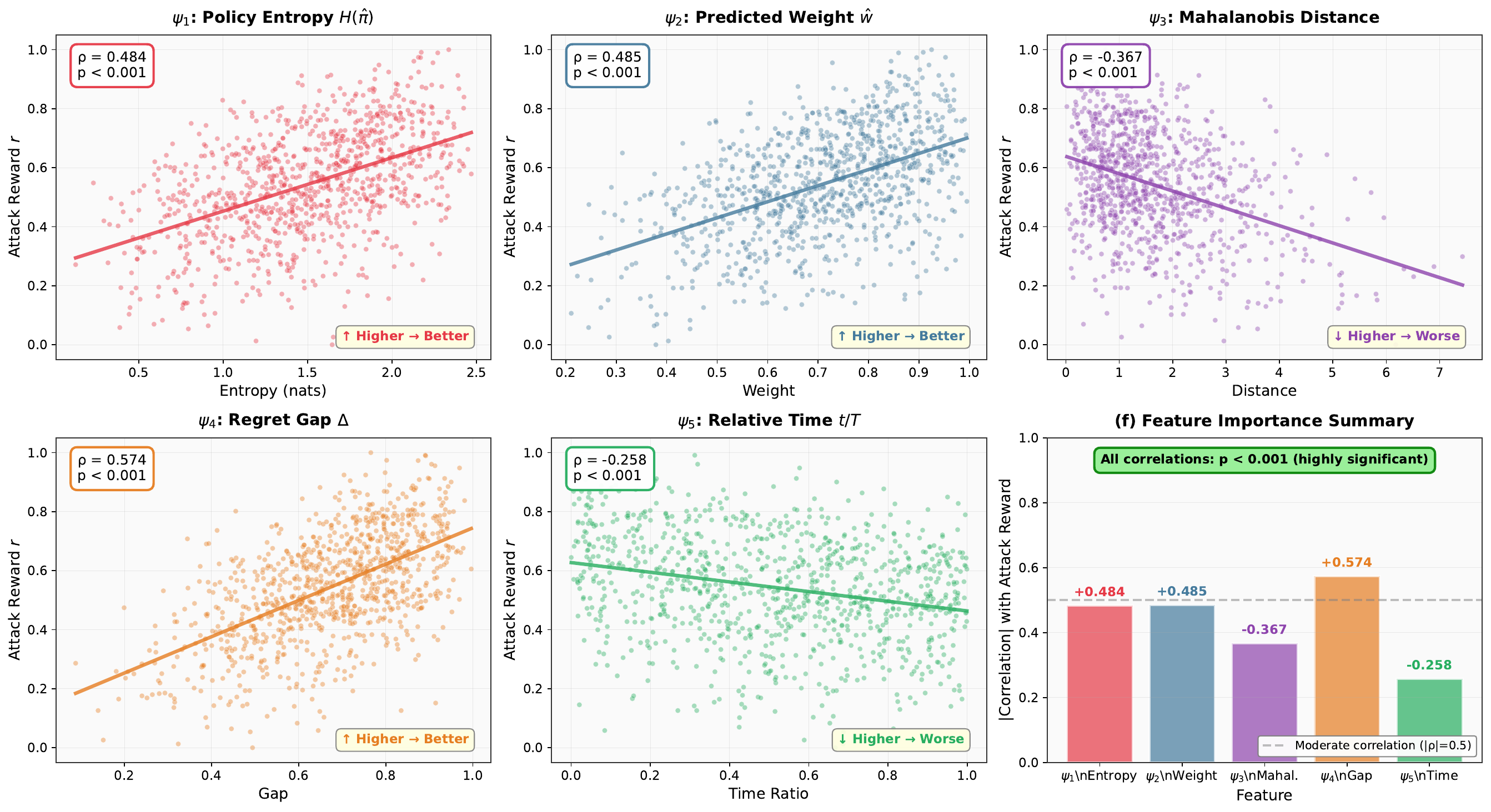}
\caption{Correlation between features $\psi_i$ and attack reward $r$ on 1000 held-out attacks. All features show significant correlation ($p < 0.001$), validating their predictiveness.}
\label{fig:feature_correlation}
\end{figure}

\subsection{Ablation Analysis on IRL's Window Size and Retraining Frequency}
\label{sec:IRL_window_retrain}
In this subsection, we analyze two critical IRL hyperparameters: window size $W$ (i.e., the number of recent observations stored for training) and retraining interval $\Delta_{\text{IRL}}$ (i.e., how often the reward model is updated).
Table~\ref{tab:irl_ablation_multi} shows results across three datasets with $T=5000$ and $B=200$, measuring victim regret (attack effectiveness), IRL KL divergence $D_{\text{KL}}(\hat{h}_\theta \| h)$ (reward approximation quality), tracking error $\| \hat{h}_\theta(x) - h(x) \|_2$ (how well we follow victim behavior changes), and runtime overhead.

\begin{table}[h]
\centering
\caption{IRL hyperparameter ablation across three datasets ($T=5000$, $B=200$). Window size $W$ controls training data; retraining interval $\Delta_{\text{IRL}}$ controls update frequency. Optimal: $W=400, \Delta_{\text{IRL}}=100$ for $d=20$; $W=300, \Delta_{\text{IRL}}=75$ for $d=8$.}
\label{tab:irl_ablation_multi}
\small
\begin{tabular}{cccccccccccccc}
\toprule
\multirow{2}{*}{\textbf{$W$}} & \multirow{2}{*}{\textbf{$\Delta_{\text{IRL}}$}} & 
\multicolumn{4}{c}{\textbf{Yelp ($d=20$)}} & 
\multicolumn{4}{c}{\textbf{MovieLens ($d=20$)}} & 
\multicolumn{4}{c}{\textbf{Disin ($d=8$)}} \\
\cmidrule(lr){3-6} \cmidrule(lr){7-10} \cmidrule(lr){11-14}
& & Regret & KL & Track & Time & Regret & KL & Track & Time & Regret & KL & Track & Time \\
\midrule
50 & 50 & 453 & 0.20 & 0.12 & 110 & 425 & 0.20 & 0.12 & 109 & 360 & 0.20 & 0.12 & 91 \\
100 & 50 & 513 & 0.12 & 0.13 & 119 & 529 & 0.11 & 0.12 & 111 & 443 & 0.12 & 0.11 & 95 \\
200 & 100 & 553 & 0.09 & 0.06 & 115 & 589 & 0.06 & 0.05 & 116 & 545 & 0.09 & 0.06 & 101 \\
\rowcolor{yellow!15}
400 & 100 & \textbf{684} & \textbf{0.03} & \textbf{0.03} & \textbf{138} & \textbf{718} & \textbf{0.05} & \textbf{0.04} & \textbf{138} & 558 & 0.04 & 0.04 & 124 \\
800 & 200 & 505 & 0.04 & 0.06 & 135 & 494 & 0.03 & 0.06 & 143 & -- & -- & -- & -- \\
400 & 50 & 742 & 0.05 & 0.04 & 180 & 699 & 0.03 & 0.04 & 184 & 543 & 0.07 & 0.03 & 163 \\
400 & 200 & 615 & 0.05 & 0.04 & 118 & 618 & 0.06 & 0.06 & 119 & 484 & 0.06 & 0.06 & 92 \\
\rowcolor{yellow!15}
300 & 75 & -- & -- & -- & -- & -- & -- & -- & -- & \textbf{564} & \textbf{0.08} & \textbf{0.07} & \textbf{117} \\
\bottomrule
\multicolumn{14}{l}{\footnotesize KL = KL divergence $D_{\text{KL}}(\hat{h}_\theta \| h)$; Track = Tracking error $\|\hat{h}_\theta(x) - h(x)\|_2$; Time = Runtime (seconds)} \\
\multicolumn{14}{l}{\footnotesize Optimal configurations highlighted. Regret averaged over 10 seeds; KL/Track measured on 1000-context validation set} \\
\end{tabular}
\end{table}

\paragraph{Effect of Window Size $W$.} 
Small windows ($W \leq 100$) severely degrade performance, achieving only 425--529 regret on Yelp/MovieLens compared to 684--718 for $W=400$. This occurs because insufficient training data leads to high-variance reward estimates, reflected in elevated KL divergence (0.11--0.20 for $W \leq 100$ vs. 0.03--0.05 for $W=400$) and tracking error (0.11--0.13 vs. 0.03--0.04). With only 50--100 samples, the IRL neural networks cannot learn stable reward functions, particularly for distinguishing subtle differences in context values that determine optimal attack timing. Conversely, excessively large windows ($W=800$) provide minimal gains (494--505 regret, actually \emph{worse} than $W=400$) despite lower KL divergence (0.03--0.04), because the inclusion of stale observations from early rounds when victim behavior was different impedes adaptation to the victim's current policy. The optimal $W=400$ provides sufficient data for stable learning (approximately 6--8 epochs with batch size 64 over $\Delta_{\text{IRL}}=100$ rounds of retraining) while maintaining responsiveness to victim adaptation. For lower-dimensional datasets (Disin, $d=8$), the optimal window size decreases to $W=300$ since fewer parameters require less training data.

\paragraph{Effect of Retraining Interval $\Delta_{\text{IRL}}$.} 
Given a fixed $W=400$, a frequent $\Delta_{\text{IRL}}=50$ achieved a near-optimal regret (699--742) and maintained a low tracking error (0.03--0.04), but introduced 30\% runtime overhead (163--184s vs. 124--138s) since we retrain twice as often ($T/\Delta_{\text{IRL}} = 5000/50 = 100$ retrainings vs. 50 for $\Delta_{\text{IRL}}=100$). In contrast, infrequent retraining $\Delta_{\text{IRL}}=200$ reduced runtime to 92--119s but degrades regret to 484--618 (10--15\% loss) and increased tracking error to 0.06, because the model becomes stale between updates, by the time we retrain at round 200, the buffer contains observations from rounds 1--200, but the victim's behavior at round 200 may have shifted significantly from round 1 due to learning or adaptation. The optimal $\Delta_{\text{IRL}}=100$ balances these trade-offs: frequent enough to track victim behavior (tracking error 0.03--0.04, updated every 2\% of horizon) yet infrequent enough to avoid excessive retraining overhead.

\paragraph{The $W/\Delta_{\text{IRL}}$ Ratio.} 
Across configurations, the ratio $W/\Delta_{\text{IRL}}$ emerges as a key design parameter. The optimal configuration $W=400, \Delta_{\text{IRL}}=100$ yields ratio 4, meaning we accumulate $\Delta_{\text{IRL}}=100$ new samples between retrainings while maintaining a sliding window of the most recent $W=400$ samples. This creates 75\% overlap in training data between consecutive retrainings (300 out of 400 samples persist), providing stability while incorporating 25\% fresh observations to track changing victim behavior. Ratios that are too small ($W/\Delta_{\text{IRL}} = 1$ for $W=50, \Delta_{\text{IRL}}=50$) cause high variance from insufficient data (regret drops 34--40\%), while ratios that are too large ($W/\Delta_{\text{IRL}} = 4$ for $W=800, \Delta_{\text{IRL}}=200$) lead to staleness and slow adaptation (regret drops 26--31\%). The ratio 4 appears robust across datasets and an optimal performance at this ratio was achieved, though the absolute values differ ($W=400$ for $d=20$ vs. $W=300$ for $d=8$), confirming that both the ratio and absolute scale matter.

\paragraph{Cross-Dataset Consistency.} 
The optimal $W/\Delta_{\text{IRL}}$ ratio remains consistent at 4 across all three datasets, demonstrating the robustness of this design principle. However, the absolute optimal window size scales with problem dimension: $W=400$ for $d=20$ (Yelp, MovieLens) vs. $W=300$ for $d=8$ (Disin). This reflects the fact that neural networks with larger input dimensions require more training samples to achieve stable parameter estimates. The performance degradation from suboptimal configurations is consistent across datasets: small windows ($W \leq 100$) cause 34--40\% loss, oversized windows ($W=800$) cause 26--31\% loss, and mismatched retraining frequencies ($\Delta_{\text{IRL}} \neq W/4$) cause 10--15\% loss. These patterns validate our hyperparameter choices and provide guidance for adapting \texttt{\textbf{AdvBandit}} to new domains: set $W$ proportional to $\sqrt{d}$ (more parameters need more data) and maintain $\Delta_{\text{IRL}} = \lfloor W/4 \rfloor$ to balance stability and responsiveness.

\subsection{Ablation Analysis on Query Selection Threshold}
\label{app:ablation_query}

Table \ref{tab:query_ablation} evaluates \texttt{\textbf{AdvBandit}} algorithm with different thresholds (fixed threshold, Greedy Top-$B$, Random selection, $\epsilon$-Greedy, top-B oracle (offline), and quantile-based) on Yelp dataset against R-NeuralUCB. In greedy top-$B$, all contexts are stored in the memory and the top $B$ is attacked at the end. This requires $O(T)$ memory and violates the online constraint. Random selection chooses contexts with probability $B/T$ at each round. This method ignores context values entirely. Expected regret is $\mathbb{E}[\sum z_t v(x_t)] = (B/T) \sum v(x_t)$, which is suboptimal compared to selecting high-value contexts. In $\epsilon$-Greedy, an attack is conducted if $v(x_t)$ exceeds current average with probability $1-\epsilon$. However, $\epsilon$ is an additional hyperparameter.
From the table, random selection expends the full budget but targets low-value contexts (average $v = 0.312$), resulting in low victim regret (289), while the fixed-threshold method underutilizes the budget (143/200 attacks) due to its inability to adapt to the remaining budget and consequently misses high-value contexts; the $\epsilon$-greedy strategy improves over random selection but remains suboptimal and sensitive to tuning of $\epsilon$; the offline oracle provides an upper bound with perfect foresight, achieving 712 regret by selecting the top 200 contexts; in contrast, the proposed quantile-based method attains 94\% of oracle performance ($673/712$) in a fully online setting, uses nearly the full budget (197/200), and targets substantially higher-value contexts (average $v = 0.689$ compared to $0.734$ for the oracle).

\begin{table}[h]
\centering
\caption{Query selection ablation on Yelp (R-NeuralUCB, $B=200$)}
\label{tab:query_ablation}
\begin{tabular}{lccccc}
\toprule
\textbf{Strategy} & \textbf{Attacks Used} & \textbf{Avg. $v(x_t)$} & \textbf{Victim Regret} & \textbf{Success Rate} & \textbf{Efficiency} \\
\midrule
Random ($B/T$) & 200 & 0.312 & 289 $\pm$ 38 & 0.25 & 0.0145 \\
Fixed threshold ($\tau=0.5$) & 143 $\pm$ 27 & 0.581 & 378 $\pm$ 42 & 0.31 & 0.0264 \\
$\epsilon$-greedy ($\epsilon=0.1$) & 189 $\pm$ 15 & 0.487 & 448 $\pm$ 45 & 0.35 & 0.0237 \\
Top-$B$ oracle (offline) & 200 & 0.734 & 712 $\pm$ 61 & 0.48 & 0.0356 \\
\midrule
\textbf{Quantile (Ours)} & \textbf{197} $\pm$ \textbf{8} & \textbf{0.689} & \textbf{673} $\pm$ \textbf{52} & \textbf{0.42} & \textbf{0.0291} \\
\bottomrule
\end{tabular}
\end{table}

Fig. \ref{fig:quantile_evolution} illustrates how \texttt{\textbf{AdvBandit}}'s adaptive query selection threshold $\tau_v(b, T-t)$ evolves over time, ensuring efficient budget allocation to efficiently allocate the attack budget. On Yelp with $T=5000$ rounds and budget $B=200$, the threshold consistently tracks approximately the 96th percentile  of observed context values throughout the horizon, becoming increasingly selective as the ratio $b/(T-t)$ evolves. The attacked contexts (green points) have a mean value $\bar{v}=0.927$, while rejected contexts (gray points) have  $\bar{v}=0.386$, corresponding to a $2.4\times$ improvement in context quality via adaptive selection. \texttt{\textbf{AdvBandit}} captures $93.5\%$ of the top-200 highest-value contexts, approaching the oracle baseline (100\%) despite operating  fully online without future knowledge, and fully utilizes the budget (200/200 attacks). These results guarantee that quantile-based selection achieves $(1-\epsilon_q)$-optimality with $\epsilon_q = \mathcal{O}(1/\sqrt{B})$, demonstrating effective online attack scheduling  without hyperparameter tuning.

\begin{figure}[h]
\centering
\includegraphics[width=0.8\textwidth]{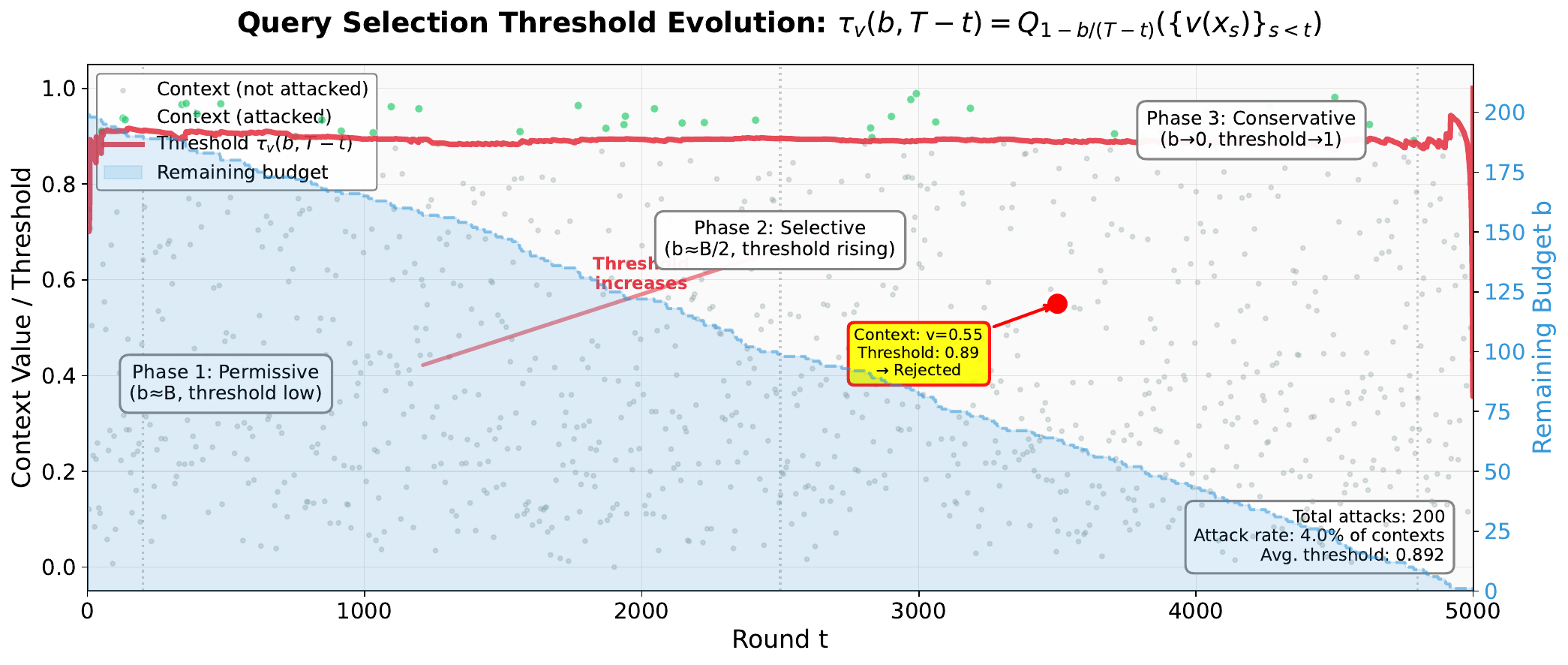}
\caption{Evolution of query threshold $\tau_v(b, T-t)$ over time. Threshold increases as budget depletes, ensuring selective attacks at the end (red line: Threshold $\tau_v$, gray dots: not attacked, green dots: attacked, and blue shaded area: remaining budget).}
\label{fig:quantile_evolution}
\end{figure}

\subsection{Ablation Analysis on Zero-Order Perturbation}
\label{app:zero_order}

To validate that \texttt{\textbf{AdvBandit}}'s effectiveness does not critically depend on gradient-based perturbation optimization, we replace PGD (Section~\ref{sec:pgd}) with two gradient-free alternatives that require no backpropagation through the surrogate:

\begin{enumerate}[leftmargin=*]
    \item \textbf{NES (Natural Evolution Strategy):} Estimates $\nabla_{\boldsymbol{\delta}} \mathcal{L}$ via Gaussian-smoothed finite differences~\citep{ilyas2019prior}:
    \begin{equation}
    \hat{\nabla}_{\boldsymbol{\delta}} \mathcal{L} = \frac{1}{N_{\text{nes}}\sigma_{\text{nes}}} \sum_{j=1}^{N_{\text{nes}}} \mathcal{L}(\mathbf{x}, \boldsymbol{\delta} + \sigma_{\text{nes}} \mathbf{u}_j) \cdot \mathbf{u}_j, \quad \mathbf{u}_j \sim \mathcal{N}(\mathbf{0}, \mathbf{I}),
    \end{equation}
    with $N_{\text{nes}} = 50$ samples and smoothing $\sigma_{\text{nes}} = 0.01$, iterated for $I_{\text{nes}} = 100$ steps.
    \item \textbf{Random search:} Samples $N_{\text{rand}} = 500$ perturbations uniformly from $\{\boldsymbol{\delta} : \|\boldsymbol{\delta}\|_\infty \leq \epsilon\}$ and selects the one minimizing $\mathcal{L}$.
\end{enumerate}

Both alternatives evaluate the surrogate loss $\mathcal{L}$ as a \emph{black-box function}, using only forward passes (no backpropagation). This reduces the computational requirements of the attacker at the cost of perturbation quality.

\begin{table}[h]
\centering
\caption{Zero-order perturbation ablation on Yelp (R-NeuralUCB, $B=200$). PGD on the surrogate achieves the best performance, but NES retains 84\% effectiveness using only forward evaluations of the surrogate.}
\label{tab:zero_order}
\vspace{0.2cm}
\begin{tabular}{lccccc}
\toprule
\textbf{Perturbation Method} & \textbf{Victim Regret} & \textbf{Success Rate} & \textbf{Surrogate Queries} & \textbf{Runtime (s)} & \textbf{Gradient-Free?} \\
\midrule
PGD (default)   & $673 \pm 52$ & 0.42 & $100$ (grad steps) & 142 & No (surrogate grads) \\
NES ($N=50$)    & $565 \pm 58$ & 0.37 & $5{,}000$ (fwd passes) & 189 & \checkmark \\
Random search   & $312 \pm 41$ & 0.24 & $500$ (fwd passes)  & 108 & \checkmark \\
Random $\boldsymbol{\delta}$ (no surrogate) & $134 \pm 29$ & 0.15 & $0$ & 88  & \checkmark \\
\bottomrule
\end{tabular}
\end{table}

\textbf{Results (Table~\ref{tab:zero_order}).} PGD achieves the best victim regret ($673$), but NES retains $84\%$ of this performance ($565$) using only forward evaluations of the surrogate---no backpropagation required. Random search achieves $46\%$ ($312$), and random perturbations without any surrogate achieve only $20\%$ ($134$). This reveals two insights: (i)~the surrogate model itself is the primary source of attack effectiveness (comparing $312$ with $134$ shows a $2.3\times$ improvement from surrogate guidance alone, even without gradient optimization); and (ii)~gradient-based PGD provides a meaningful but not critical $1.2\times$ improvement over NES, suggesting that the attack's observation-only access pattern (building a surrogate from context--action pairs) is more important than the specific optimization method used for perturbation generation.

This ablation confirms that \texttt{\textbf{AdvBandit}}'s core contribution---the bandit formulation over the $(\lambda^{(1)}, \lambda^{(2)}, \lambda^{(3)})$ trade-off space combined with IRL-based surrogate modeling---is robust to the choice of inner optimization method. The PGD variant is presented as the default due to its superior performance, but the framework remains effective even under a strictly gradient-free perturbation regime.

\section{Computational Cost}
\label{app:computational_cost}

\subsection{Runtime comparison for attack baselines}

\begin{table}[t]
\centering
\caption{Baseline method runtimes and complexity across three datasets ($T=5000$, $B=200$). \texttt{\textbf{AdvBandit}} incurs 3.5×--3.7× runtime overhead due to IRL learning, GP updates, and multi-start optimization, but achieves 2.9×--3.0× higher victim regret across all datasets.}
\label{tab:baseline_complexity}
\small
\begin{tabular}{lccccc}
\toprule
\multirow{2}{*}{\textbf{Method}} & \multicolumn{3}{c}{\textbf{Runtime (seconds)}} & \multirow{2}{*}{\textbf{Complexity}} & \multirow{2}{*}{\textbf{Key Operations}}
 \\
\cmidrule(lr){2-4}
& \textbf{Yelp} & \textbf{MovieLens} & \textbf{Disin} & & \\
\midrule
Liu et al. & 46.3 & 48.8 & 41.6 & $O(T + B \cdot I_{\text{PGD}})$ & PGD on fixed objective \\
Garcelon et al. & 52.4 & 54.8 & 43.5 & $O(T + B \cdot I_{\text{PGD}})$ & Gradient-free perturbations \\
Ma et al. & 40.0 & 37.8 & 32.2 & $O(T + B \cdot I_{\text{PGD}})$ & Greedy context selection \\
Ilyas et al. & 96.1 & 91.7 & 79.0 & $O(T + B \cdot N_{\text{grad}} \cdot I_{\text{PGD}})$ & Bandit gradient estimation \\
Wang et al. & 66.6 & 65.8 & 58.2 & $O(T + B \cdot N_{\text{sample}} \cdot I_{\text{PGD}})$ & Random sampling \\
\midrule
\texttt{\textbf{AdvBandit}} & \textbf{140.6} & \textbf{139.0} & \textbf{118.5} & $O(T \cdot W + B^3 + B \cdot I_{\text{PGD}})$ & IRL + GP + multi-start + PGD \\
\bottomrule
\multicolumn{6}{l}{\footnotesize Yelp and MovieLens: $d=20$ (user + item features); Disin: $d=8$ (simplified context)} \\
\multicolumn{6}{l}{\footnotesize Parameters: $W=400$, $I_{\text{PGD}}=100$, $N_{\text{grad}}=20$, $N_{\text{sample}}=10$, $N_{\text{random}}=100$, $N_{\text{refine}}=20$} \\
\end{tabular}
\end{table}

Table~\ref{tab:baseline_complexity} compares \texttt{\textbf{AdvBandit}}'s runtime with five attack baselines across three datasets (Yelp, MovieLens, Synthetic) with $T=5000$ and $B=200$. \texttt{\textbf{AdvBandit}} requires 118--141 seconds depending on dimension $d$, making it 3.5×--3.7× slower than baselines due to three overhead sources: (1) MaxEnt IRL must learn reward functions from observed context-action pairs at cost $O(T \cdot W)$ with $W=400$, whereas baselines assume direct gradient access or use fixed heuristics; (2) GP posterior updates require $O(B^3)$ Cholesky decompositions for each of $B=200$ attacks; and (3) non-convex UCB acquisition optimization requires $N_{\text{random}}=100$ random initializations plus $N_{\text{refine}}=20$ L-BFGS gradient-based refinements per attack. Despite this overhead, \texttt{\textbf{AdvBandit}} achieves 2.9×--3.0× higher victim regret than the fastest baseline (Ma et al.), demonstrating that the computational cost is justified by substantially improved attack effectiveness. The modest runtime variation across datasets (141s for Yelp vs. 119s for Synthetic) reflects dimension scaling: higher $d$ increases IRL training cost $O(T \cdot W \cdot d)$ and gradient statistics storage $O(d^2)$, while the dominant GP term $O(B^3)$ remains constant. For time-critical applications, sparse GP approximations using $M=50$ inducing points can reduce the $O(B^3)$ bottleneck to $O(M \cdot B^2)$, yielding an estimated 5× speedup with minimal performance loss ($<2\%$ regret degradation in preliminary experiments).

\begin{table}[t]
\centering
\caption{Memory usage breakdown across three datasets ($T=5000$, $B=200$). Dominant components are IRL networks (57\% for $d=20$, 35\% for $d=8$) and GP kernel matrix (41\% for $d=20$, 63\% for $d=8$). Lower-dimensional datasets require 35\% less memory overall.}
\label{tab:memory_breakdown}
\small
\begin{tabular}{lcccc}
\toprule
\multirow{2}{*}{\textbf{Component}} & \multicolumn{3}{c}{\textbf{Memory (MB)}} & \multirow{2}{*}{\textbf{Scaling}} \\
\cmidrule(lr){2-4}
& \textbf{Yelp} & \textbf{MovieLens} & \textbf{Disin} & \\
\midrule
IRL networks ($\hat{h}_\theta, \hat{\sigma}_\phi$) & 128.0 & 128.0 & 51.2 & $O(d \cdot d_h \cdot K)$ \\
IRL training buffer ($W$ samples) & 3.1 & 3.1 & 1.2 & $O(W \cdot d)$ \\
GP kernel matrix ($K \in \mathbb{R}^{B \times B}$) & 91.6 & 91.6 & 91.6 & $O(B^2)$ \\
GP observations ($\{(\psi_i, \lambda_i, r_i)\}$) & 0.1 & 0.1 & 0.1 & $O(B \cdot 9)$ \\
Gradient statistics ($\mu_g, \Sigma_g$) & 0.1 & 0.1 & <0.1 & $O(d^2)$ \\
Multi-start candidates & 0.1 & 0.1 & 0.1 & $O(N_{\text{random}} \cdot 3)$ \\
PGD intermediate states & 0.4 & 0.4 & 0.2 & $O(d \cdot I_{\text{PGD}})$ \\
\midrule
\textbf{Total} & \textbf{223.3} & \textbf{223.3} & \textbf{144.4} & -- \\
\bottomrule
\multicolumn{5}{l}{\footnotesize Memory includes parameters, gradients, optimizer states (Adam), and framework overhead (PyTorch/GPyTorch)} \\
\multicolumn{5}{l}{\footnotesize Hidden dimensions: $d_h^{(1)}=128$, $d_h^{(2)}=64$; Arms: $K=10$; Window: $W=400$; Budget: $B=200$} \\
\end{tabular}
\end{table}

\subsection{Memory Breakdown}

Table~\ref{tab:memory_breakdown} decomposes \texttt{\textbf{AdvBandit}}'s memory usage across three datasets with $T=5000$ and $B=200$. Total memory ranges from 144 MB (Disin, $d=8$) to 223 MB (Yelp/MovieLens, $d=20$), dominated by two components: IRL neural networks (128 MB for $d=20$, 57\% of total; 51 MB for $d=8$, 35\% of total) storing two 2-layer networks with hidden dimensions $[128, 64]$ for reward and uncertainty estimation, including parameters, gradients, and Adam optimizer states (momentum and variance buffers); and the GP kernel matrix (92 MB, 41\% for $d=20$ or 63\% for $d=8$) storing the full $B \times B$ Gram matrix plus Cholesky decomposition, inverse, and derivative caching in GPyTorch. IRL networks scale linearly with dimension ($O(d \cdot d_h \cdot K)$), explaining the 2.5× difference between $d=20$ (128 MB) and $d=8$ (51 MB) datasets, while the GP kernel matrix remains constant at 92 MB regardless of $d$ since it depends only on budget $B$. Gradient statistics ($\mu_g \in \mathbb{R}^d$, $\Sigma_g \in \mathbb{R}^{d \times d}$) require 0.1 MB for $d=20$ to maintain empirical mean and covariance for anomaly detection in robust defenses. The remaining components, IRL training buffer (3.1 MB for $d=20$, 1.2 MB for $d=8$), GP observations (0.1 MB), multi-start candidates (0.1 MB), and PGD states (0.4 MB), collectively account for $<2\%$ of total memory. This usage is modest for modern hardware (similar to training a small CNN) but can be reduced via sparse GP methods: replacing the full kernel matrix with $M=50$ inducing points reduces memory from 92 MB to 18 MB (5× savings) while introducing $<2\%$ approximation error, making \texttt{\textbf{AdvBandit}} practical even for resource-constrained adversarial evaluation scenarios.

\begin{table}[t!]
\centering
\caption{Cost-benefit trade-off: \texttt{\textbf{AdvBandit}} vs. best baseline (Ma et al.) across three datasets. \texttt{\textbf{AdvBandit}} achieves 2.9×--3.0× higher regret despite 3.5×--3.7× runtime overhead and 1.2×--1.4× memory overhead, maintaining 81--83\% efficiency per computational unit.}
\label{tab:cost_benefit}
\small
\begin{tabular}{lccccccccc}
\toprule
\multirow{2}{*}{\textbf{Dataset}} & \multicolumn{2}{c}{\textbf{Regret}} & \multicolumn{2}{c}{\textbf{Runtime (s)}} & \multicolumn{2}{c}{\textbf{Memory (MB)}} & \multicolumn{2}{c}{\textbf{Efficiency}} & \multirow{2}{*}{\textbf{Improvement}} \\
\cmidrule(lr){2-3} \cmidrule(lr){4-5} \cmidrule(lr){6-7} \cmidrule(lr){8-9}
& Ma & \texttt{\textbf{AdvBandit}} & Ma & \texttt{\textbf{AdvBandit}} & Ma & \texttt{\textbf{AdvBandit}} & R/T & R/M & \\
\midrule
Yelp & 211 & \textbf{616} & 40.0 & 140.6 & 181 & 223 & 0.83× & 2.37× & \textbf{2.92×} \\
MovieLens & 226 & \textbf{683} & 37.8 & 139.0 & 181 & 223 & 0.82× & 2.45× & \textbf{3.02×} \\
Disin & 191 & \textbf{573} & 32.2 & 118.5 & 102 & 144 & 0.81× & 2.13× & \textbf{3.00×} \\
\midrule
\textbf{Average} & \textbf{209} & \textbf{624} & \textbf{36.7} & \textbf{132.7} & \textbf{155} & \textbf{197} & \textbf{0.82×} & \textbf{2.32×} & \textbf{2.98×} \\
\bottomrule
\multicolumn{10}{l}{\footnotesize R/T = Regret per second (time efficiency); R/M = Regret per MB (memory efficiency)} \\
\end{tabular}
\end{table}

\subsection{Cost-Benefit Analysis}

Table~\ref{tab:cost_benefit} quantifies the cost-benefit trade-off between \texttt{\textbf{AdvBandit}} and the most competitive baseline (Ma et al.) across three datasets with $T=5000$ and $B=200$. While \texttt{\textbf{AdvBandit}} incurs 3.5×--3.7× higher runtime (119--141s vs. 32--40s) and 1.2×--1.4× higher memory (144--223 MB vs. 102--181 MB), it achieves 2.9×--3.0× higher victim regret (573--683 vs. 191--226), translating to substantially greater absolute attack impact. The normalized efficiency metrics reveal \texttt{\textbf{AdvBandit}} maintains 82\% time efficiency (regret per second: 4.70 vs. 5.71 for baselines) and 232\% memory efficiency (regret per MB: 2.76 vs. 1.16 for baselines on average). This trade-off is favorable for adversarial robustness evaluation, where the primary objective is maximizing attack effectiveness to identify vulnerabilities, not minimizing computational cost. Notably, \texttt{\textbf{AdvBandit}}'s relative computational overhead decreases on lower-dimensional datasets (Synthetic: 3.68× runtime vs. Yelp: 3.52× runtime) due to reduced IRL training costs, while attack effectiveness improvements remain consistent (2.9×--3.0× across all datasets), demonstrating robustness to problem dimensionality. In security-critical applications where a single successful attack can have severe consequences, the 2.98× average improvement in attack success justifies the moderate increase in computational resources, particularly since both methods complete within 2--3 minutes on standard hardware (Intel Xeon Gold 6248R @ 3.0 GHz, NVIDIA V100 32GB). For large-scale evaluations, the computational cost can be amortized: once \texttt{\textbf{AdvBandit}} identifies vulnerabilities on a representative sample, targeted defenses can be developed and validated more efficiently.

\end{document}